\documentclass{article} 
\usepackage{iclr2027_conference,times}

\usepackage{url}            
\usepackage{booktabs}       
\usepackage{multirow}
\usepackage{amsfonts}       
\usepackage{nicefrac}       
\usepackage{microtype}      
\usepackage{xcolor}         
\usepackage{graphicx}
\usepackage{amsmath}
\usepackage{amssymb}
\usepackage{amsthm}
\usepackage{mathtools}
\usepackage{etoc}
\usepackage{tabularx}
\usepackage{algorithm}
\usepackage{subcaption}
\usepackage[noend]{algpseudocode}

\usepackage[pagebackref=true,breaklinks=true,colorlinks=true,bookmarks=false]{hyperref}
\definecolor{deepred}{HTML}{940000}
\hypersetup{linkcolor=deepred}
\hypersetup{urlcolor  = [rgb]{0.4,0.15,0.95}}
\hypersetup{citecolor=[rgb]{0.4,0.15,0.95}}
\usepackage[capitalize,noabbrev]{cleveref}

\usepackage{pifont}
\usepackage{makecell}
\usepackage{titletoc}

\usepackage[toc,page,header]{appendix}
\usepackage{minitoc}

\renewcommand \thepart{}
\renewcommand \partname{}

\usepackage{caption}

\definecolor{Gray}{gray}{0.92}
\newlength\savewidth

\theoremstyle{plain}
\newtheorem{theorem}{Theorem}[section]

\theoremstyle{definition}

\theoremstyle{remark}

\title{AMO: Operator-level Adaptive Muon Orthogonalization}

\author{
    Xinlin Zhuang\textsuperscript{1,2,3,4} \quad
    Panyi Ouyang\textsuperscript{2} \quad
    Yichen Li\textsuperscript{3,5} \quad
    Jiangming Shi\textsuperscript{4} \quad
    Yizhang Chen\textsuperscript{2} \\
    \textbf{Shuman Liu}\textsuperscript{2} \quad
    \textbf{Ying Qian}\textsuperscript{4} \quad
    \textbf{Weiyang Liu}\textsuperscript{1}\thanks{Corresponding authors.} \quad
    \textbf{Haibo Zhang}\textsuperscript{2}\footnotemark[1] \quad
    \textbf{Imran Razzak}\textsuperscript{3}\footnotemark[1] \\
    \normalfont
    \textsuperscript{1}The Chinese University of Hong Kong \quad
    \textsuperscript{2}Shopee \quad
    \textsuperscript{3}MBZUAI \\
    \textsuperscript{4}East China Normal University \quad
    \textsuperscript{5}Huazhong University of Science and Technology \\
    \texttt{\footnotesize xinlinzhuang@link.cuhk.edu.hk}
}

\iclrfinalcopy 

\begin{document}

\maketitle
\doparttoc
\faketableofcontents
\lhead{Preprint, under-review.}    

\maketitle

\begin{abstract}
Muon has recently emerged as a competitive alternative to AdamW for large-scale pre-training, with orthogonalization via Newton–Schulz (NS) iteration as its core operation. 
Standard Muon applies a uniform NS schedule to all parameter matrices, overlooking possible differences in orthogonalization difficulty and its impact on performance.
Through a systematic empirical study, we show that this per-matrix heterogeneity is pervasive and strongly associated with matrix geometry, which evolves dynamically across operator types, training stages, and network depths.
Therefore, uniform NS schedules can lead to uneven orthogonalization quality across the model.
Motivated by these findings, we propose Operator-level Adaptive Muon Orthogonalization (AMO), an observe-then-commit method that measures weight geometry by operator type early in training and then uses these signals to allocate the NS budget for the remainder of training. 
AMO delivers consistent improvements over uniform-schedule Muon across \textbf{standard}, \textbf{prolonged}, and \textbf{continual} pre-training, surpassing the strongest baseline by +0.76 on Llama3.1-1.4B and +0.51 on Qwen3-1.7B in average downstream performance of 12 evaluation tasks, with gains persisting at Llama3.1-4B scale.
\end{abstract}

\section{Introduction}

Muon~\citep{muon} has emerged as a competitive alternative to AdamW for large-scale model pre-training~\citep{muon_tail,what_really_matters}, with its core operation being the orthogonalization via Newton-Schulz (NS) iterations~\citep{dsv4,kimi_k2,kimi_k3,practical_efficiency,dmuon,GramMuon,adamuon,normuon}.
Two key design choices governing NS orthogonalization are the count of NS iteration steps $T$ and the coefficient sequence $\{(a_k, b_k, c_k)\}_{k=1}^T$. 
Prior works have improved NS through globally shared coefficient sequences, including optimization-fitted coefficients and Chebyshev-based designs~\citep{pe,cans}, while ROOT~\citep{root} takes an important step towards adaptive Muon by specializing NS coefficients to \textbf{matrix shape}.
Yet ROOT's shape-level adaptation still treats matrices with the same shape as interchangeable, leaving an important question unresolved: \textit{do matrices truly differ in how difficult they are to orthogonalize, and if so, along which dimension?}

We conduct a systematic pilot study to investigate this question (Sec.~\ref{sec:pilot}). 
\textbf{First}, a factorial study over two model families, two optimizers, four learning-rate (LR) schedulers, and three data orderings shows that monotone quality-based ordering does not reliably improve over random mixing (Fig.~\ref{fig:data_curriculum_main}), suggesting that such data-side curricula are not a reliable remedy under the evaluated settings and motivating optimizer-side adaptation.
\textbf{Second}, a model geometry analysis tracking momentum matrices updated by Muon reveals that heterogeneity is real and structured: a geometry-derived input statistic predicts NS quality (signal validity, Fig.~\ref{fig:geometry_main}(\textbf{a})); operator types exhibit persistent, well-separated difficulty trajectories (operator-type structure, Fig.~\ref{fig:geometry_main}(\textbf{b})); per-layer profiles evolve similarly and differ only marginally in magnitude (weak layer magnitude, Fig.~\ref{fig:geometry_main}(\textbf{c})); and even those marginal layer differences reverse during training (non-stationary layer ordering, Fig.~\ref{fig:geometry_main}(\textbf{d})). 
Together, these findings motivate \textbf{operator type} as an effective and compact unit of adaptation and support measuring the control signal from training dynamics rather than relying on a fixed matrix-level proxy.

\begin{figure}[t]
    \centering
    \setlength{\abovecaptionskip}{4pt}
    \setlength{\belowcaptionskip}{-4pt}
    \vspace{-1mm}
    \includegraphics[width=1.0\linewidth]{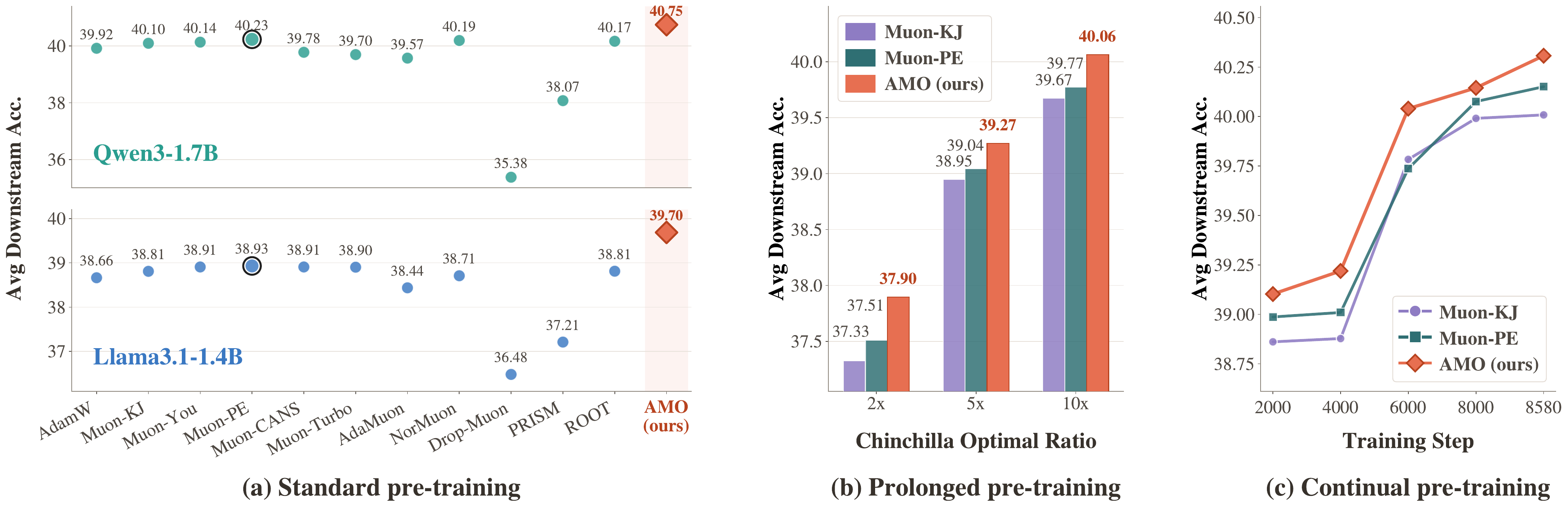}
    \caption{\footnotesize Results of AMO and baseline optimizers on \textbf{(a)} standard pre-training across model families and scales, \textbf{(b)} prolonged pre-training (2$\times$/5$\times$/10$\times$ Chinchilla optimal ratios) on Qwen3-0.6B, and \textbf{(c)} continual pre-training of Qwen3-0.6B on Slimpajama. AMO consistently leads in all three settings.}
    \label{fig:main}
\end{figure}

Taken together, these observations motivate \textbf{operator-type adaptive NS allocation}: \textit{how to allocate NS effort non-uniformly across parameter matrices based on their measured optimization difficulty?}
To address this problem, we propose \textbf{Operator-level Adaptive Muon Orthogonalization (AMO)}, a simple observe-then-commit method that observes per-operator-type geometry signals under the common uniform schedule for a fixed 
window, computes a single allocation plan via greedy budget reallocation over projected NS error curves, then linearly transitions to that plan and locks it for the remainder of training. 
Extensive experiments on Llama3.1 and Qwen3 models ranging from 0.6B to 1.7B parameters show that AMO consistently outperforms uniform-schedule Muon and other baselines across model scales in standard pre-training, prolonged pre-training, and continual pre-training, demonstrating the effectiveness of AMO (Fig.~\ref{fig:main}).

Our contributions are threefold. 
First, we motivate \textbf{optimizer-side adaptation} by showing that monotone quality-based data curricula do not reliably improve over random mixing across the evaluated datasets and optimization settings, indicating that such data-side interventions are not a reliable remedy for the observed heterogeneity.
Second, we \textbf{characterize its structure} through four geometry findings: signal validity, operator-type structure, weak layer magnitude, and non-stationary layer ordering, which together identify where and how heterogeneity manifests and provide the structural basis for operator-type adaptive NS allocation. 
Third, we \textbf{propose AMO}, a simple observe-then-commit method that achieves consistent gains over uniform-schedule Muon methods across standard, prolonged, and continual pre-training, demonstrating the effectiveness of adaptive Muon.
\vspace{-1.5mm}
\section{Pilot Analysis}
\vspace{-.75mm}
\label{sec:pilot}

A common design in Muon-family optimizers is to apply a shared NS schedule across updated parameter matrices in the network.
Most of them~\citep{pe,cans,turbo_muon} share one schedule across all matrices and all training steps, while ROOT~\citep{root} specializes NS coefficients to \textbf{matrix shape} but still treats matrices of the same shape (e.g.\ \texttt{attn\_k} and \texttt{attn\_v} in Llama / Qwen models) as interchangeable.
Whether this uniformity is adequate depends on whether per-matrix optimization difficulty is truly uniform, and if not, on \textit{where} the heterogeneity comes from. 
We therefore ask two questions. 
\textbf{Q1}: Can curriculum learning reduce per-matrix heterogeneity enough to remove the need for optimizer-side adaptation? \textbf{Q2}: Is the structure of this heterogeneity \textbf{stable enough} across operator type, model depth, or training step to admit an adaptive scheme?
Our experiments suggest no to \textbf{Q1} under the evaluated settings (Sec.~\ref{subsec:data_curriculum}) and yes to \textbf{Q2} (Sec.~\ref{subsec:model_geometry}).
The rest part provides concrete explorations and guides method designs.

\subsection{Data Curriculum}
\label{subsec:data_curriculum}

Classical curriculum learning predicts that ordering examples by difficulty improves optimization~\citep{lr_waste}. 
Here, we use quality scores~\citep{meta_rater} as practical difficulty-related signals for monotone curricula, whose effectiveness in mitigating the observed heterogeneity should manifest as reliable gains in final performance.
We test this hypothesis with a factorial study over two base models (Llama3.1-760M, Qwen3-0.6B), two optimizers (both Muon and AdamW are considered), four LR schedulers (Warmup-Stable (WS), Warmup-Stable-Decay (WSD), Cosine, Constant), and three data orderings (\texttt{Uniform}, \texttt{Ascend}, \texttt{Descend}) on the \texttt{FineWeb-Edu} dataset~\citep{fineweb}, where data are sorted by the \textit{Educational Value} score field.
The same data pool is used across all runs where only the \textbf{data presentation order} differs. 
Full hyperparameters, configurations, and additional results are provided in App.~\ref{app:data_curriculum}.

Fig.~\ref{fig:data_curriculum_main} reports the average downstream accuracy across a fixed evaluation suite for all 16 configurations. 
The performance ordering $\texttt{Uniform} > \texttt{Descend} > \texttt{Ascend}$ holds consistently in \textbf{every} configuration, with no inversion across model, optimizer, or LR scheduler.
Notably, \texttt{Ascend} and \texttt{Descend} traverse the \textit{Educational-Value} score in opposite directions yet both underperform \texttt{Uniform}, indicating that neither direction provides a reliable advantage over random mixing in these experiments.
This pattern suggests that monotone quality-based ordering is generally not a reliable remedy across the evaluated datasets and optimization settings, while also indicating a non-negligible interaction between curriculum ordering and the LR schedule.
We therefore turn to optimizer-side geometry to directly characterize the observed per-matrix heterogeneity.

\begin{figure}[t]
    \centering
    \includegraphics[width=1.0\linewidth]{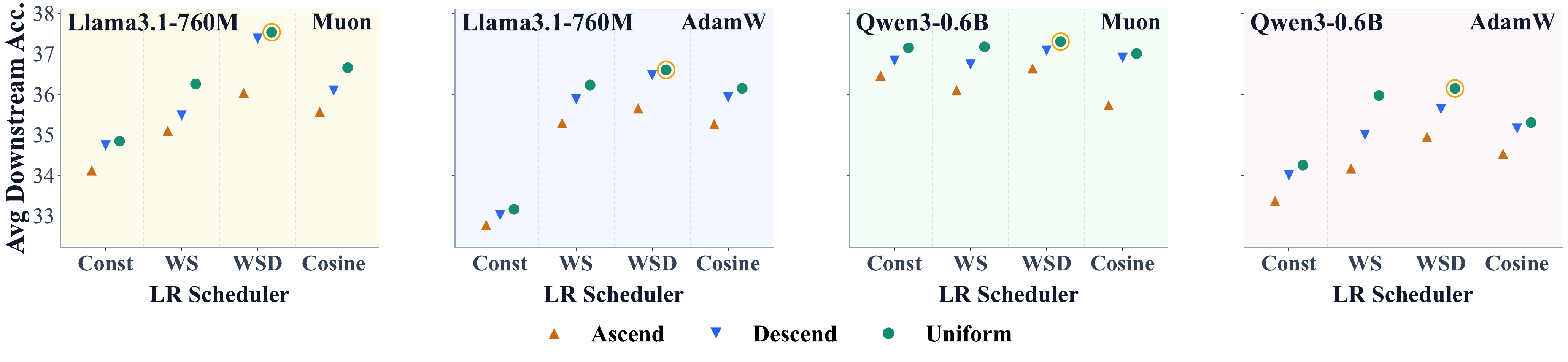}
    \caption{\footnotesize Average downstream accuracy under three data orderings, across 16 configurations. 
    The best result in each subplot is marked with a yellow circle.
    \texttt{Ascend} and \texttt{Descend} denote curriculum learning in increasing/decreasing sample score order, while \texttt{Uniform} uses a random order.
    }
    \label{fig:data_curriculum_main}
\end{figure}

\subsection{Model Geometry}

\label{subsec:model_geometry}

Let $M_t \in \mathbb{R}^{m\times n}$ denote one Muon momentum matrix at step $t$. 
With Nesterov momentum, $M_t = G_t + \mu B_t, \, B_t = \mu B_{t-1} + G_t,$
where $G_t$ is the stochastic gradient and $B_t$ is the momentum buffer. 
Muon orthogonalizes $M_t$ via $T$-step NS iterations, producing $M_t'$. 
We compute, for each Muon-updated 2D matrix in every transformer block (seven operator types: 
\texttt{attn\_\{q,k,v,o\}}, \texttt{mlp\_\{gate,up,down\}}), the condition number of the NS input
$\kappa(M_t)$ and the shape-normalized NS residual $\tilde{\epsilon}_{\mathrm{NS}} \;=\; 
\frac{\lVert M_t' {M_t'}^{\!\top} - I \rVert_F}{\sqrt{\min(m,n)}}$,
which is comparable across matrices of different shapes. 
Fig.~\ref{fig:geometry_main} summarizes four findings on a representative model Qwen3-0.6B, and analogous patterns hold for all seven remaining Llama3.1 and Qwen3 models with parameters ranging from 70M to 1.7B (App.~\ref{app:model_geometry}).

\textbf{F1 and F2: A geometry signal predicts NS quality and concentrates at operator type.}
As shown in Fig.~\ref{fig:geometry_main}(\textbf{a}), 
the NS-input condition number $\kappa(M_t)$ is positively correlated with the shape-normalized residual $\tilde{\epsilon}_{\mathrm{NS}}$ with Pearson $r=0.729$, confirming that a geometry-derived input-side statistic provides an informative proxy for NS quality (\textbf{F1}).
App.~\ref{app:model_geometry} shows that Pearson $r$ lies in $[0.722, 0.813]$ for Llama3.1 and $[0.671, 0.912]$ for Qwen3.
Operator-type trajectories of median $\kappa$ are persistently separated throughout training (Fig.~\ref{fig:geometry_main}(\textbf{b})), with between-type variation visibly dominating within-type variation across layers (\textbf{F2}). 
Moreover, the dominant hard operators differ across families: \texttt{attn\_v} on Qwen3 and \texttt{attn\_q} and \texttt{attn\_o} on Llama3.1, with no single operator dominating across families.
This observation suggests that matrix shape alone may miss operator-specific heterogeneity even when different operators share the same shape.
These results therefore motivate \textbf{operator type} as a compact unit of adaptation.

\begin{figure}[t]
    \centering
    \includegraphics[width=1.0\linewidth]{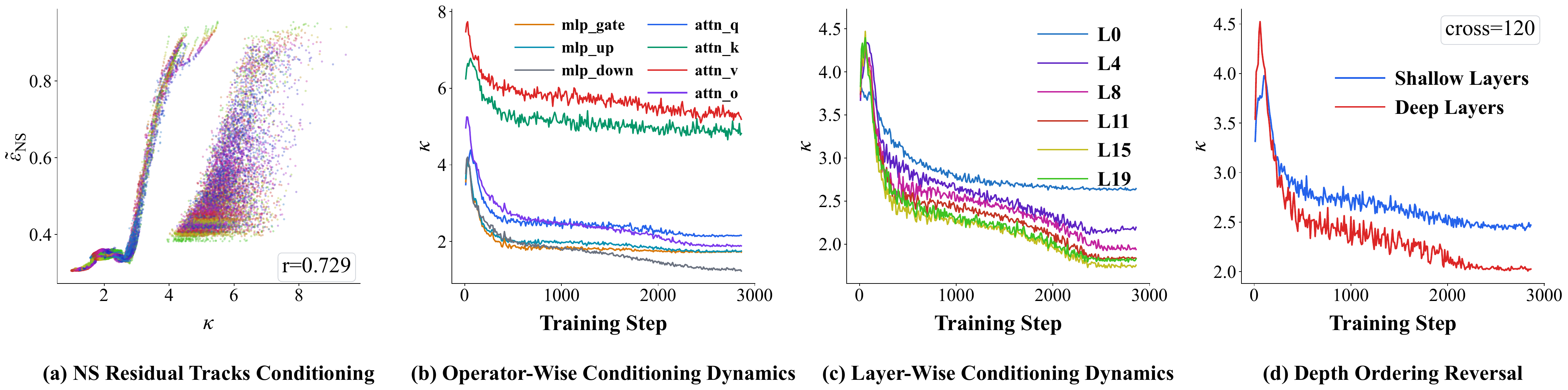}
   \caption{\footnotesize Geometry findings on Qwen3-0.6B. $\kappa$ is shown at $\log_{10}$ scale.
    \textbf{(a)} NS-input conditioning $\kappa(M_t)$ vs. 
    shape-normalized residual $\tilde{\epsilon}_{\mathrm{NS}}$ 
    (Pearson's $r=0.729$). 
    Each point represents one parameter matrix at one training step, colored by layer depth.
    \textbf{(b)} Median $\kappa$ trajectories per operator 
    type: \texttt{attn\_v} dominates the hardest regime, 
    \texttt{mlp\_down} the easiest. 
    \textbf{(c)} Layer-wise median $\kappa$ trajectories for 
    six layers show similar evolution patterns.
    \textbf{(d)} Median $\kappa$ trajectories of shallow (first three) vs. deep (last three) layers for \texttt{attn\_q}; their relative ordering reverses at step 120, indicating that difficulty cannot be predicted from depth alone.}
    \label{fig:geometry_main}
\end{figure}

\textbf{F3 and F4: Per-layer differences are modest and non-stationary.}
Fig.~\ref{fig:geometry_main}(\textbf{c}) tracks the layer-wise median $\kappa$, aggregated across operator types, for six layers spanning the network depth (L0, L4, L8, L11, L15, L19).
These trajectories evolve similarly throughout training and differ only modestly in magnitude, in contrast to the pronounced gaps across operator types in Fig.~\ref{fig:geometry_main}(\textbf{b}) (\textbf{F3}).
Moreover, Fig.~\ref{fig:geometry_main}(\textbf{d}) compares shallow (first three layers) and deep (last three layers) groups of \texttt{attn\_q}: their relative ordering \textit{reverses} at step 120, so the group that is harder to orthogonalize early in training becomes the easier one later (more discussion in App.~\ref{app:post_warmup_stability}) (\textbf{F4}).
Depth-based difficulty ordering is therefore not fixed, and layer depth alone provides an insufficiently stable basis for prescribing the allocation schedule a priori.
Instead, the allocation signal should be measured from the running model, and the allocation unit should preserve its relative difficulty over training, a property that operator type exhibits but depth does not.
Together, these four findings motivate the design requirements for adaptive Muon, which we instantiate as a simple observe-then-commit method in Sec.~\ref{sec:method}.

\vspace{-1.5mm}
\section{Method}
\vspace{-1.5mm}
\label{sec:method}

\subsection{Preliminaries}

\textbf{Muon and NS Iteration.}
Muon orthogonalizes momentum $M$ via $T$-step NS iterations of the form
\begin{equation}
\footnotesize
M \;\leftarrow\; a_k M + b_k M M^{\!\top} M + c_k (M M^{\!\top})^2 M,
\qquad k = 1, \dots, T,    
\end{equation}
parametrized by the coefficient sequence $\{(a_k, b_k, c_k)\}_{k=1}^T$. 
The \textit{only} degrees of freedom in NS design are therefore the iteration count $T$ and the coefficient sequence. 
The input-side geometry signal $\ell(M) = \sigma_{\min}^{(\tau)}(M) / \lVert M \rVert_F$, where \(\tau\) is a truncation threshold and singular values below \(\tau\) are treated as \(\tau\) and $\lVert \cdot \rVert_F$ denotes the matrix Frobenius norm, is monotone in \(1/\kappa(M)\) in the well-conditioned regime when computed from exact Singular Value Decomposition (SVD).

\textbf{Existing NS designs are global across matrices updated by Muon.}
Muon~\citep{muon} fits a single coefficient sequence, whereas You fits $T$ coefficient sequences through optimization and applies them to every matrix at each training step.
More recently, PE and CANS~\citep{pe,cans} derive $T$ coefficient sequences from Chebyshev-based conservative bounds assuming a universal worst-case input ($\ell = 10^{-3}$, a conservative lower bound on the geometry signal across matrices and steps). 
ROOT~\citep{root} specializes coefficient sequences to \textit{matrix shape} but treats matrices of the same shape identically, ignoring the operator-type heterogeneity documented in \textbf{F2}. 
None of these methods uses a runtime-measured signal, and none adapts at operator-type granularity, the two gaps Sec.~\ref{sec:pilot} identifies.

\textbf{From findings to design requirements.}
The four geometry findings (\textbf{F1-F4}) inspire the following three design requirements.
\textbf{D1} (from \textbf{F1}): \textit{use $\ell$ as the control signal for optimization difficulty} because it is a valid predictor of NS quality, integrates directly with PE-style polynomial composition~\citep{pe}, and can be computed directly from SVD.
\textbf{D2} (from \textbf{F2}+\textbf{F3}+\textbf{F4}): \textit{adapt at operator-type granularity}. 
\textbf{F2} confirms operator type as a meaningful unit and \textbf{F3+F4} jointly rule out layer-wise adaptation.
\textbf{D3} (from the dynamic nature of \textbf{F3}, \textbf{F4}): \textit{acquire the signal through online observation}, since the geometry cannot be reliably prescribed from architectural priors (we discuss static NS coefficients in Sec.~\ref{subsec:static_coeff} and static NS step in App.~\ref{app:ns_step}).
Moreover, a fourth choice: committing to a single allocation rather than continuously re-tuning, is methodological: \textbf{F2}'s persistent operator-type separation (Fig.~\ref{fig:geometry_main}(\textbf{b})) shows the cross-type ordering does not undergo destabilizing changes during training, while continuous re-tuning would cause optimizer dynamics to become unstable.

\subsection{Adaptive Muon Orthogonalization}

\begin{algorithm}[t]
\small
\caption{Adaptive Muon Optimizer (AMO)}
\label{alg:amo}
\begin{algorithmic}[1]
\Require budget ratio $r$, baseline step count $T_{\mathrm{base}}$, 
step range $[k_{\min}, k_{\max}]$, shrinkage $\alpha$, baseline 
$\ell_{\mathrm{base}}$, observation interval $\Delta_{\mathrm{obs}}$, 
observation end step $T_{\mathrm{obs}}$, observation sample size $n_\mathrm{smp}$, transition duration $\Delta$
\State compute total budget $B \gets \mathrm{round}(r \cdot |\mathcal{O}| \cdot T_{\mathrm{base}})$
\State \textbf{Phase 1 (Observation):} run training under uniform 
$T_{\mathrm{base}}$ schedule; every $\Delta_{\mathrm{obs}}$ steps, 
record per-type median $\ell_{\mathrm{eff}}^{(i)}$ from $n_\mathrm{smp}$ matrices within each operator type $i \in \mathcal{O}$. Exit at step $T_{\mathrm{obs}}$.
\State \textbf{Plan:}\; 
$\ell_{\mathrm{target}}^{(i)} \gets 
\alpha \cdot \mathrm{median\text{-}of\text{-}medians}(\ell_{\mathrm{eff}}^{(i)}) 
+ (1-\alpha) \cdot \ell_{\mathrm{base}}$ for each $i \in \mathcal{O}$
\State For each $i \in \mathcal{O}$ and $k \in [k_{\min}, k_{\max}]$: 
$E_i(k) \gets 1 - \textsc{PE-simulate}(\ell_{\mathrm{target}}^{(i)}, k)$
\State $\{T_i^\star\}_{i \in \mathcal{O}} \gets \textsc{BudgetAllocate}(\{E_i\}, B)$ 
\Comment{greedy \textit{add} / \textit{remove} / \textit{transfer}}
\State \textbf{Phase 2 (Transition):} \textbf{for} $u = 1, \ldots, \Delta$ \textbf{do}
\State \quad $p \gets u/\Delta$
\State \quad \textbf{for each} $i \in \mathcal{O}$: \;
$T_i(u) \gets T_{\mathrm{base}} + p\,(T_i^\star - T_{\mathrm{base}})$, \;\;
$\ell_{\mathrm{eff}}^{(i)}(u) \gets \ell_{\mathrm{base}} + p\,(\ell_{\mathrm{target}}^{(i)} - \ell_{\mathrm{base}})$
\State \quad $B(u) \gets \mathrm{round}\bigl(\sum_{i \in \mathcal{O}} T_i(u)\bigr)$;\; round $\{T_i(u)\}$ s.t. $\sum_{i \in \mathcal{O}} T_i(u) = B(u)$
\State \quad regenerate PE coefficients at $(T_i(u), \ell_{\mathrm{eff}}^{(i)}(u))$ for each $i \in \mathcal{O}$
\State \textbf{end for}
\State \textbf{Lock:} freeze the schedule and coefficients; stop geometry capture.
\end{algorithmic}
\end{algorithm}

Following \textbf{D1-D3}, the core operation of \textbf{AMO} is to allocate a fixed NS step budget across the seven operator types $\mathcal{O}=\{\texttt{attn\_\{q,k,v,o\}}, \texttt{mlp\_\{gate,up,down\}}\}$ and recompute the coefficient sequence for each operator type, following the observe-then-commit methodological design above. 
The complete workflow of AMO is provided in Alg.~\ref{alg:amo}.

\textbf{Phase 1: Observation (D3).}
Under the uniform baseline schedule, AMO captures the optimizer-internal momentum matrix $M_t$ every $\Delta_{\mathrm{obs}}$ steps, until reaching the observation end step $T_{\mathrm{obs}}$. 
For each operator type, a fixed-size subset of $n_\mathrm{smp}$ matrices is uniformly sampled for observation, rather than using all matrices of that type, to improve efficiency.
Per observation we record the per-type median of $\ell_{\mathrm{eff}}^{(i)} = \ell_{\mathrm{raw}}^{(i)}/s$, where $s=1.01$ is a scale factor used for numerical safety. 
After a total of $T_{\mathrm{obs}} / \Delta_{\mathrm{obs}}$ observations, we have collected sufficient geometry observations to invoke planning.

\textbf{Planning (D1, D2).}
Aggregation uses a median-of-medians, the median, over observations, of each observation's per-type median across sampled matrices, yielding a robust per-type estimate $\ell_{\mathrm{robust}}^{(i)}$ that is shrunk toward a conservative baseline $\ell_\mathrm{base}$,
\begin{equation}
\footnotesize
\ell_{\mathrm{target}}^{(i)} 
\;=\; 
\alpha \cdot \ell_{\mathrm{robust}}^{(i)}
\;+\; (1-\alpha) \cdot \ell_{\mathrm{base}},
\end{equation}
where the shrinkage coefficient $\alpha$ stabilizes allocation against single-type overconfidence. 
For each type $i$ and each candidate step count $k \in [k_{\min}, k_{\max}]$ (we empirically set $k_{\min}=3$ and $k_{\max}=7$), we invoke the PE composer at $(\ell_{\mathrm{target}}^{(i)}, k)$ and simulate the resulting NS trajectory, giving a \textit{projected} error curve $E_i(k) = 1 - \ell_{\mathrm{after\text{-}final}}$, where $\ell_{\mathrm{after\text{-}final}}$ is $\ell$ on the matrix produced by the final NS iteration. 
Searching over projected PE error curves rather than mapping $\ell$ directly to a step count correctly accounts for the joint effect of $T$ and the PE coefficients fitted at that $T$. 
Under the budget constraint $\sum_i T_i = B$, a discrete greedy routine uses only the ranking of per-type marginal gains and losses: \textit{add}, \textit{remove}, \textit{transfer}, with details and an optimality statement provided in App.~\ref{app:proofs}.

\textbf{Phase 2: Transition and Lock.}
A discontinuous switch from the baseline schedule $(T_{\mathrm{base}}, \ell_{\mathrm{base}})$ to the planned schedule $(T_i^\star, \ell_{\mathrm{target}}^{(i)})$ would cause the orthogonalized update $M_t'$ to jump, destabilizing the momentum buffer $B_t$ whose accumulated gradients were produced under the old schedule. 
AMO therefore distributes the change over a transition duration of $\Delta$ steps via joint linear interpolation. With $p(u) = u/\Delta$ for each $u = 1, 2, \ldots, \Delta$,
\begin{equation}
\footnotesize
    T_i(u) 
= T_{\mathrm{base}} + p(u)\bigl(T_i^\star - T_{\mathrm{base}}\bigr),
\qquad
\ell_{\mathrm{eff}}^{(i)}(u) 
= \ell_{\mathrm{base}} + p(u)\bigl(\ell_{\mathrm{target}}^{(i)} - \ell_{\mathrm{base}}\bigr).
\end{equation}
Two consistency rules apply during transition. 
(i) \textit{Two-stage rounding.} Rounding each $T_i(u)$ independently would let the running total drift off the planned trajectory, so we first round the interpolated total $B(u) \!=\! \mathrm{round}(\sum_{i \in \mathcal{O}} T_i(u))$, then round per-type counts under the constraint $\sum_i T_i(u) = B(u)$. The total therefore advances monotonically from $|\mathcal{O}| \cdot T_{\mathrm{base}}$ at $u=0$ to $B$ at $u=\Delta$. 
(ii) \textit{Coefficient regeneration.} At every NS-applied step, PE coefficients are regenerated at the current interpolated state $(T_i(u), \ell_{\mathrm{eff}}^{(i)}(u))$, so AMO adapts the full coefficient table, not only the step count. 
At step $u = \Delta$, the schedule and coefficients freeze and geometry capture stops; the persistent cross-type ordering established in \textbf{F2} ensures that locking does not discard a useful signal.
Overall, observation horizon $T_{\mathrm{obs}}$, observation interval $\Delta_{\mathrm{obs}}$, observation sample size $n_\mathrm{smp}$, shrinkage coefficient $\alpha$, and transition duration $\Delta$ are discussed in Sec.~\ref{subsec:ablation}.
\section{Experiments and Results}

\subsection{Experimental Setup}
\label{sec:exp_setup}
\textbf{Optimizers.}
We compare AMO against AdamW and several Muon-style optimizers.
AdamW uses $\beta_1=0.9$, $\beta_2=0.95$, $\epsilon=10^{-8}$, weight decay $0.1$, and a WSD LR schedule.
For all Muon variants (Muon~\citep{muon}, AdaMuon~\citep{adamuon}, You, PE~\citep{pe}, CANS~\citep{cans}, ROOT~\citep{root}, and Turbo~\citep{turbo_muon}), we apply Muon to 2D matrices in Transformer blocks and AdamW to embeddings, output heads, biases, and normalization parameters.
Muon uses Nesterov momentum $0.95$, weight decay $0.1$, and five NS iterations with the official KJ coefficients~\citep{muon}, followed by a $0.2\sqrt{\max(m,n)}$ rescaling for an $m\times n$ matrix~\citep{dmuon}.
The base learning rate follows the empirical scaling rule of~\citet{lr_law}, with a cosine-decay schedule for all Muon-style runs.
Implementation details of optimizers utilized are provided in App.~\ref{app:subsec_exp_setting}.

\textbf{Pre-training.}
We pre-train eight dense decoder-only Transformers from the Llama-3.1 and Qwen3 families~\citep{llama3,qwen3}, spanning four scales per family from 70M to 1.7B parameters.
Except for continual pre-training, all models are trained from scratch with the official Gaussian initialization.
All models use grouped-query attention~\citep{gqa}, RMSNorm~\citep{rmsnorm}, and SwiGLU feed-forward blocks~\citep{swiglu}.
Training data are sampled from \texttt{FineWeb-Edu}~\citep{fineweb} and packed into 2048-token sequences with each family's native tokenizer.
Standard pre-training uses the Chinchilla-optimal budget of 20 tokens per parameter~\citep{scaling_law} (roughly 1B--34B tokens across scales), and prolonged pre-training uses 2$\times$, 5$\times$, and 10$\times$ this budget.
All runs use bfloat16 mixed precision, Flash-Attention-2~\citep{flashattention2}, gradient clipping at norm $1.0$, and random seed 42, on 6--32 NVIDIA H100 (80GB) GPUs.

\textbf{Evaluation.}
We evaluate all model checkpoints using the \texttt{lm-evaluation-harness}~\citep{eval-harness} framework on 12 standard downstream benchmarks spanning three capability axes: 
\textit{commonsense reasoning} (HellaSwag~\citep{hellaswag}, PIQA~\citep{piqa}, WinoGrande~\citep{winogrande}, OpenBookQA~\citep{openbookqa}, CommonsenseQA~\citep{commonsenseqa}), 
\textit{scientific knowledge} (ARC-Easy, ARC-Challenge~\citep{arc}, SciQ~\citep{sciq}), 
and \textit{general language understanding} (MMLU~\citep{mmlu}, MMLU-Pro~\citep{mmlu_pro}, RACE~\citep{race}, AGIEval-en~\citep{agieval}). 
All tasks are evaluated in zero-shot mode to isolate pre-training quality from in-context learning, except MMLU-Pro, which uses its official 5-shot setting. 
We report accuracy using the framework's default generation likelihood-based scoring, except MMLU-Pro, which uses greedy generation with exact-match per its official protocol. 
The unweighted average across the 12 benchmarks (\textit{Avg.}) serves as the main metric.

\subsection{Main Results}

Across the evaluated model families and scales, AMO consistently improves over uniform-schedule Muon and other baseline optimizers. 
Beyond the reported main comparisons, we further verify the robustness of these gains with paired \textbf{multi-seed} experiments on Llama3.1-1.4B (App.~\ref{app:multi_seed}) and extend the evaluation to \textbf{Llama3.1-4B} pre-trained for \textbf{100B} tokens (App.~\ref{app:large_scale}). 
These additional results show that the observed improvements persist across random seeds and at a larger model scale.

\textbf{AMO outperforms baseline optimizers across model scales and families.}
As shown in Fig.~\ref{fig:main}(\textbf{a}), AMO achieves the highest average downstream accuracy across two model families, improving over the strongest baseline by +0.76 on Llama3.1-1.4B (vs.\ Muon-PE, 39.69 vs.\ 38.93) and +0.51 on Qwen3-1.7B (vs.\ Muon-PE, 40.75 vs.\ 40.24) and the same ranking holds at the smaller scale, with +0.20 on Llama3.1-760M and +0.44 on Qwen3-0.6B (full details and results in App.~\ref{app:main_full}). 
The gain is also broadly distributed across 12 tasks rather than concentrated on a few benchmarks. 
Notably, the strongest static baseline varies across models, so no single hand-designed schedule wins everywhere, whereas AMO does so without per-model tuning. 
We interpret this as evidence that operator-type-aware budget allocation transfers across model scales (760M to 1.7B) and two architectural families, while online adaptation mitigates the brittleness of any fixed coefficient choice.

\textbf{AMO is effective under prolonged pre-training.}
To test whether AMO's advantage persists beyond Chinchilla-optimal 
compute, we extend pre-training on Qwen3-0.6B to \textbf{2$\times$}, \textbf{5$\times$}, and \textbf{10$\times$} Chinchilla optimal token ratios, consuming \textbf{24B}, \textbf{60B}, and \textbf{120B} tokens. 
Full results are provided in App.~\ref{app:prolong_full}.
Fig.~\ref{fig:main}(\textbf{b}) shows that AMO remains consistently better than both strong Muon baselines at all three prolonged-compute settings. 
At 2$\times$, 5$\times$, and 10$\times$ Chinchilla ratios, AMO reaches an average downstream accuracy of 37.90, 39.27, and 40.06, improving over the strongest baseline by +0.39, +0.23, and +0.29, respectively. 
Thus, the advantage of AMO is not only an early-training effect near the standard compute budget, but persists when the same model is trained substantially longer. 
This suggests that online operator-type-aware allocation continues to make useful budget decisions even as optimization enters later training stages, where a fixed coefficient schedule may become increasingly mismatched to the evolving geometry of different operator type groups.

\textbf{AMO is effective under continual pre-training.}
We further evaluate AMO in continual pre-training (CPT), starting from a Qwen3-0.6B checkpoint pre-trained on 60B tokens and training on an additional 36B tokens from SlimPajama~\citep{meta_rater,slimpajama}.
As shown in Fig.~\ref{fig:main}(\textbf{c}), AMO outperforms both Muon-KJ and Muon-PE at every evaluated step.
Its margin over the strongest baseline grows from +0.12 at step 2000 to +0.26 at step 6000 and remains positive at the end, where AMO reaches 40.31 versus 40.15 for Muon-PE and 40.01 for Muon-KJ.
This indicates that AMO remains effective in the more practical CPT setting, where the optimizer must improve downstream performance without disrupting an already-trained model.

\subsection{Cost Analysis}

We compare the cost of AMO with baseline methods from both theoretical and practical perspectives. 
More details are provided in App.~\ref{app:cost} and we discuss the quality-compute tradeoff of AMO in Sec.~\ref{subsec:budget_control}.

\textbf{Theoretical complexity.}
Let $P=P_{\mu}+P_{\mathrm{adam}}$ denote the total trainable parameters,
where $P_{\mu}$ are 2D matrices updated by the Muon-family rule and
$P_{\mathrm{adam}}$ are handled by AdamW. For the $i$-th matrix of layer type $t$, one NS iteration costs
$C_{t,i}=\Theta\!\bigl(\min(m_{t,i},n_{t,i})^{2}\max(m_{t,i},n_{t,i})\bigr)$.
In terms of \textbf{memory}, AdamW and AdaMuon lie in the heavier $2P$
class, whereas Muon, Muon-PE, and AMO all fall into the lighter
$P_{\mu}+2P_{\mathrm{adam}}$ class (AMO adds only $\mathcal{O}(|\mathcal{T}|)$ control metadata).
In terms of \textbf{time}, every Muon-family method is dominated by
$\sum_{t}k_{t}\sum_{i}C_{t,i}$; AdaMuon adds only a linear
$P_{\mu}$ term on top. 
The key distinction of AMO is that it replaces the uniform rule $k_{t}\!\equiv\!5$ with a per-type allocation
$k_{t}^{\star}$ under a global budget $\sum_{t}k_{t}^{\star}=B$.
Because the aggregate weights $\sum_{i}C_{t,i}$ differ substantially
across layer types, fixing the nominal budget does \emph{not} fix the
real computation: AMO can reduce actual cost by shifting steps from
expensive types to cheaper ones, optimizing the weighted objective
$\sum_{t}k_{t}\sum_{i}C_{t,i}$ rather than the average NS depth.
This \textit{budget-aware cross-type reallocation} places AMO in a
distinct complexity regime from static fixed-budget Muon variants.

\begin{table}[t]
\centering
\small
\caption{\footnotesize Practical time and memory cost on pre-training Qwen3-1.7B,
measured over 6912 steps under an identical batch configuration.
Arrows ($\downarrow$ and $\uparrow$) indicate whether lower or higher is
better.}
\label{tab:practical-cost}
\setlength{\tabcolsep}{8pt}
\begin{tabular}{lcccc}
\toprule
\textbf{Optimizer} & \textbf{Step Time} $\downarrow$ & \textbf{Optimizer Update Time} $\downarrow$ & \textbf{Throughput} $\uparrow$ & \textbf{Peak Memory} $\downarrow$ \\
 & (s) & (ms) & (M token/s) & (GB) \\
\midrule
AdamW               & \textbf{5.96}    & \textbf{4.48}      & \textbf{0.82}    & \underline{63.75} \\
Muon                & \underline{6.14} & 195.81             & \underline{0.80} & \textbf{61.12} \\
AdaMuon             & 6.19             & 241.65             & 0.79             & 69.00 \\
\textbf{AMO (ours)} & 6.17             & \underline{191.47} & \underline{0.80} & \textbf{61.12} \\
\bottomrule
\end{tabular}

\end{table}

\textbf{Practical cost.}
Tab.~\ref{tab:practical-cost} reports wall-clock measurements of
pre-training Qwen3-1.7B for 6912 steps using different optimizers. 
Although Muon-family optimizer updates are roughly 43-54x more
expensive than AdamW's element-wise step (191.47-241.65 ms vs 4.48 ms),
end-to-end step time grows by only 3.0\%-3.8\%, because the
forward/backward pass ($\sim$5.82 s/step) dominate the training loop.
Within the Muon family, AMO's optimizer update (191.47 ms) is lower than
Muon (195.81 ms) and AdaMuon (241.65 ms), while its peak memory
(61.12 GB) ties with Muon at the lowest value and saves 7.88 GB relative
to AdaMuon.
In this run the budget ratio is set to $1.0$, so AMO does not reduce the
nominal total NS budget but redistributes it across operator types, locking
to an average of 5.0 NS steps across the seven operator types.
The residual wall-clock gap to Muon stems from the low-frequency
exact-SVD observation steps ($\sim$15.63 s/observation step, of which
9.47 s is the post-step observation overhead) rather than from the
locked-phase update rule; once the schedule is locked, AMO matches Muon
at 6.13 s/step and 0.80M tokens/s (training loss curve in App.~\ref{app:ba_algo}).

\vspace{-.5mm}
\subsection{Ablation Study}
\vspace{-.5mm}
\label{subsec:ablation}

\begin{figure}
    \centering
    \includegraphics[width=1.0\linewidth]{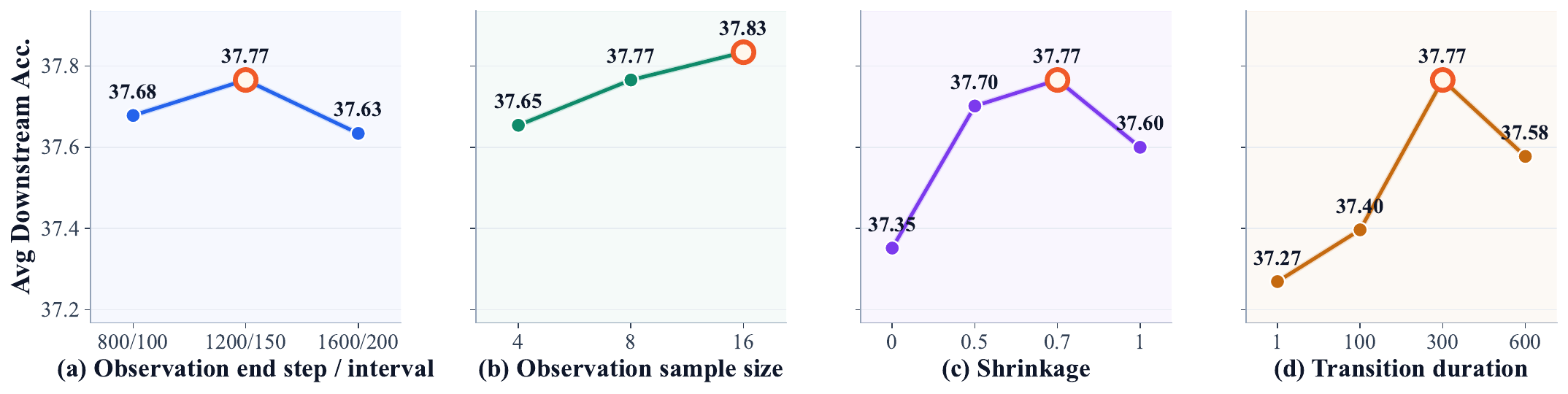}
    \caption{\footnotesize Downstream results of four ablations on pre-training Qwen3-0.6B. Each subplot isolates one AMO hyperparameter configuration, and orange circles mark the best-performing settings.}
    \label{fig:ablation}
\end{figure}

We conduct an ablation study on four key design choices in AMO and report the average downstream accuracy of the resulting Qwen3-0.6B models after pre-training. Full results are provided in App.~\ref{app:ablation_full}.

\textbf{Effect of observation.}
We jointly examine three observation-stage hyperparameters: the horizon $T_{\mathrm{obs}}$ and interval $\Delta_{\mathrm{obs}}$, which control when planning is invoked and how many snapshots feed the median-of-medians estimate; and the sample size $n_{\mathrm{smp}}$, the number of matrices drawn per operator type to estimate $\ell_{\mathrm{eff}}^{(i)}$.
For the horizon and interval, we compare $(1200, 150)$ against $(800, 100)$ and $(1600, 200)$, keeping the total observation count fixed at $T_{\mathrm{obs}}/\Delta_{\mathrm{obs}} = 8$.
Fig.~\ref{fig:ablation}(\textbf{a}) shows all settings within a narrow band (37.63--37.77) with the middle marginally best, indicating that AMO is largely insensitive to window placement once the observation count is fixed.
For the sample size, Fig.~\ref{fig:ablation}(\textbf{b}) shows a monotone but saturating trend across $n_{\mathrm{smp}} \in \{4, 8, 16\}$ (37.65 $\to$ 37.77 $\to$ 37.83): the estimator stabilizes with few samples, so the default $n_{\mathrm{smp}}=8$ balances stability against capture cost.

\textbf{Effect of shrinkage.}
The shrinkage coefficient $\alpha$ interpolates between the measured $\ell_{\mathrm{eff}}^{(i)}$ and the conservative baseline $\ell_{\mathrm{base}} = 10^{-3}$. 
Notably, $\alpha = 1$ trusts the measurement fully while $\alpha = 0$ entirely recovers the non-adaptive baseline. 
Fig.~\ref{fig:ablation}(\textbf{c}) shows a pronounced inverted-U curve over $\alpha \in \{0, 0.5, 0.7, 1.0\}$: $\alpha = 0$ falls to the non-adaptive floor (37.35), $\alpha = 1$ recovers only 37.60, and the shrunk estimators dominate, peaking at $\alpha = 0.7$ (37.77). 
The two endpoints fail for opposite reasons: $\alpha = 0$ discards the measured signal entirely, while $\alpha = 1$ inherits the residual variance of per-operator estimates, so partial shrinkage yields a better bias-variance trade-off.

\textbf{Effect of transition duration.}
The transition duration $\Delta$ controls how abruptly AMO moves from the uniform baseline to the AMO target schedule: a very short $\Delta$ approximates a step change and risks momentum-buffer inconsistency, while a very long $\Delta$ delays the benefit of the optimized allocation. 
Comparing $\Delta \in \{1, 100, 300, 600\}$, Fig.~\ref{fig:ablation}(\textbf{d}) shows that the near-instant switch $\Delta = 1$ is clearly weakest (37.27) and $\Delta = 100$ remains suboptimal (37.40), confirming that abrupt transitions destabilize the momentum buffers; performance peaks at $\Delta = 300$ (37.77) and slightly retreats at $\Delta = 600$ (37.58). 
Therefore, the benefit is captured once the transition is smooth enough to absorb the schedule change and further extension only delays entry into the optimized regime.
\section{Analysis and Discussion}

\subsection{Budget Control}
\label{subsec:budget_control}

\begin{figure}[t]
    \centering
    \includegraphics[width=1.0\linewidth]{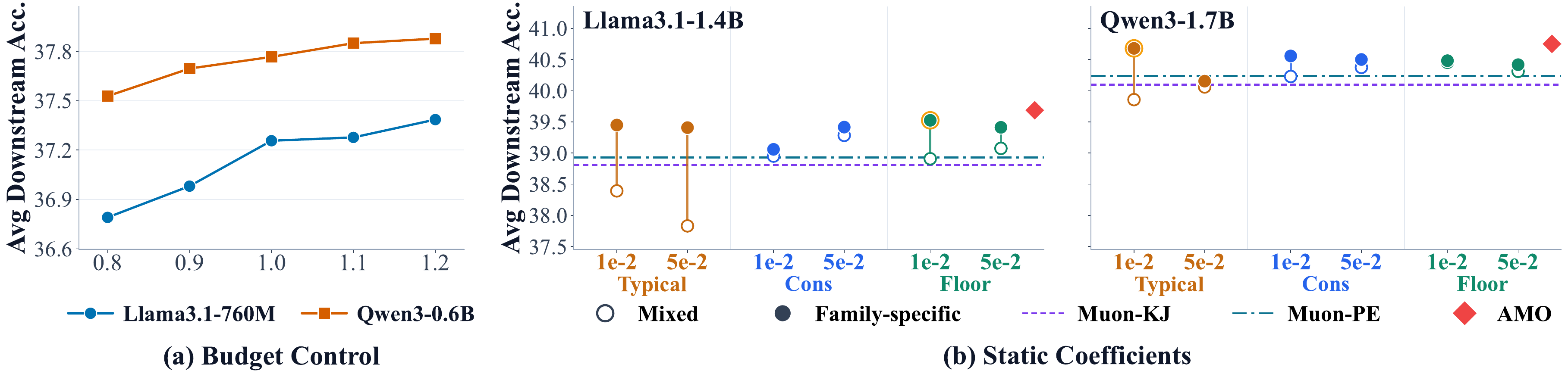}
    \caption{\footnotesize Downstream results of Llama / Qwen models pre-trained \textbf{(a)} with five different AMO budget ratios (0.8, 0.9, 1.0, 1.1, 1.2); \textbf{(b)} under 12 \textit{mixed} and \textit{family-specific} static NS schedules. 
    }
    \label{fig:analysis}
\end{figure}

To test whether AMO's gains come \textit{from better allocation or simply from more orthogonalization compute}, we vary the NS budget ratio $r$ on Llama3.1-760M and Qwen3-0.6B, keeping all other settings unchanged (full results in App.~\ref{app:budget_control_full}).
Since the Muon baseline uses 5 NS steps for each of the 7 operator types (a total budget of 35), we set the execution-phase budget to $B=\mathrm{round}(35r)$ with $r\in\{0.8,0.9,1.0,1.1,1.2\}$.
All runs share the same observation phase under the uniform 5-step baseline and differ only in the planning budget, so orthogonalization FLOPs scale approximately linearly with $r$ while all other training compute stays fixed.
As shown in Fig.~\ref{fig:analysis}(\textbf{a}), AMO yields a monotone compute--quality curve on both models, improving from 37.53 to 37.88 on Qwen3-0.6B and from 36.79 to 37.39 on Llama3.1-760M as $r$ increases from 0.8 to 1.2.
The curve is also asymmetric: reducing the budget to 0.9 saves $\sim$8.6\% FLOPs at a cost of only 0.07 (Qwen) and 0.28 (Llama) points, whereas increasing it to $r=1.2$ adds 20\% FLOPs for gains of only 0.11 and 0.13.

\subsection{Static Coefficients}
\label{subsec:static_coeff}

To isolate the role of \textit{adaptivity}, we compare AMO with static coefficient schedules inspired by PE and CANS, precomputed from the observation data in Sec.~\ref{subsec:model_geometry} (details in App.~\ref{app:static_coeff}).
For Llama3.1-1.4B and Qwen3-1.7B, we sweep three aggregation rules (\texttt{Typical}, \texttt{Conservative}, \texttt{Floor}) and two NS error tolerances ($10^{-2}$, $5\times10^{-2}$), estimating each operator-type-wise schedule from either a \textit{mixed} pool across all models or a \textit{family-specific} pool, yielding six paired variants per model.
Fig.~\ref{fig:analysis}(\textbf{b}) reveals three patterns.
\textit{(i) Family-specific estimation consistently beats mixed estimation:} it wins all 12 matched pairs, with mean gains of +0.64 on Llama3.1-1.4B and +0.25 on Qwen3-1.7B.
\textit{(ii) Even among family-specific schedules, the best aggregation rule is model-dependent:} \texttt{Floor-1e-2} is best on Llama3.1-1.4B (39.52), while \texttt{Typical-1e-2} is best on Qwen3-1.7B (40.68).
\textit{(iii) AMO surpasses the best static schedule on both models without model-specific tuning,} reaching 39.69 (+0.17) on Llama3.1-1.4B and 40.75 (+0.07) on Qwen3-1.7B.
These results expose the brittleness of static designs, which must commit before training to an aggregation rule, error tolerance, and estimation pool that do not transfer robustly across model families; AMO sidesteps these choices by deriving its schedule online from the model's own geometry.
\subsection{Allocation Granularity}
\label{subsec:granularity}

Our geometry analysis motivates operator type as a compact allocation unit, but the modest layer-wise differences leave open whether a finer, depth-aware subdivision could help.
We therefore construct a \textbf{Type $\times$ Depth AMO} variant that splits each of the seven operator types into shallow and deep groups, yielding 14 allocation units instead of 7.
On Llama3.1-1.4B, Type $\times$ Depth AMO improves the average score over Muon (38.91 vs.\ 38.66) but trails operator-type AMO (39.55) by 0.64 points.
Finer granularity therefore does not help under the same observation and budget setting, suggesting that the depth split does not contribute enough allocation signal to offset the finer partition.
This result does not imply that depth information is inherently uninformative or that operator type is universally optimal.
Rather, it shows that \textbf{operator-type allocation outperforms the tested shallow/deep subdivision under our fixed-budget setting}.
Combined with the weak and non-stationary layer differences in Sec.~\ref{subsec:model_geometry}, this supports operator type as an effective and compact granularity for AMO.
\section{Concluding Remarks}

In this paper, we challenge the assumption that Muon's NS orthogonalization should be applied uniformly to all parameter matrices.
Our geometric analysis shows that per-matrix heterogeneity is intrinsic, structured at the operator-type level, and stable in cross-type ranking.
Building on this, we propose AMO, which reallocates a fixed NS budget across operator types based on the model's own geometry, without extra orthogonalization compute or model-specific tuning.
Across Llama-3.1 and Qwen3 models up to 1.7B parameters, AMO consistently outperforms uniform-schedule Muon in both pre-training and continual pre-training, and surpasses tuned static schedules.
These results suggest treating orthogonalization not as a fixed routine, but as a compute resource to be allocated according to each operator's geometry.

\subsection*{AI use statement}

In this work, we used generative AI tools to provide feedback on research methodology and experimental design, assist with the implementation of methods, support the interpretation of experimental results, and assist with the writing, editing, and translation of the manuscript. 
We have not used generative AI tools to propose and refine research hypotheses, generate synthetic datasets, formulate mathematical claims, provide ingredients for or assist with mathematical proofs, clean or reformat datasets, or conduct qualitative or thematic data analysis.
The remaining required disclosure tasks are not applicable to this work. Additionally, we used generative AI tools to improve the organization, clarity, and presentation of the manuscript. 
All AI-assisted written content were independently reviewed by the authors. 
AI-assisted code and experimental configurations were verified through manual inspection and testing, and AI-assisted interpretations and research suggestions were checked against experimental evidence and relevant literature. 
We take responsibility for the final content of this work, including text, claims, or artifacts produced with the aid of generative AI.

\subsection*{Reproducibility statement}
We have made every effort to ensure the reproducibility of our results.
Sec.~\ref{sec:exp_setup} describes the experimental setup, including model configurations, training data, token budgets, and optimizer hyperparameters.
The complete AMO procedure, including the observation, planning, and execution phases and the NS budget allocation algorithm, is detailed in App.~\ref{app:ba_algo}.
Pilot analysis setups are provided in App.~\ref{app:data_curriculum} and App.~\ref{app:model_geometry}
Implementation details and hyperparameters of all baseline optimizers are given in App.~\ref{app:subsec_exp_setting}.

\bibliography{iclr2027_conference}
\bibliographystyle{iclr2027_conference}

\clearpage
\newpage
\newpage
\appendix
\onecolumn
\addcontentsline{toc}{section}{Appendix}
\renewcommand \thepart{}
\renewcommand \partname{}
\part{\Large{\centerline{Appendix}}}
\parttoc

\newpage

\section{Limitations}

Our study has three main limitations.
First, AMO adopts an observe-then-commit strategy: the NS allocation is fixed after the observation phase and cannot respond to substantial geometry shifts later in training, which periodic re-observation could address.
Second, although operator-type allocation outperforms the tested Type $\times$ Depth variant, it is not necessarily optimal, and other granularities or richer allocation signals may prove more effective.
Finally, our experiments focus on dense Transformers up to 4B parameters. Extending AMO to other architectures, such as mixture-of-experts models, 
and to substantially larger-scale pre-training remains future work.

\section{Related Work}

\paragraph{Adam optimizer.}

Adam \cite{adam} and its decoupled weight decay variant AdamW \cite{adamw} have long served as the default optimizers for LLM pre-training. \cite{why_adam_better_than_sgd} provides a theoretical justification from a Hessian perspective, showing that the loss landscape of Transformers exhibits highly heterogeneous curvature across parameter dimensions, which necessitates the coordinate-wise adaptivity that Adam provides and that SGD fundamentally lacks. 
Building on this understanding, one line of work seeks to go beyond Adam's diagonal second-moment preconditioner by incorporating richer curvature information. 
Sophia \citep{sophia} uses a lightweight diagonal Hessian estimator paired with element-wise clipping to achieve a 2$\times$ speedup over Adam in pre-training steps with negligible per-step overhead. 
SOAP \citep{soap} establishes a formal equivalence between Shampoo and Adafactor in a rotated basis, and runs Adam in Shampoo's eigenbasis to combine the benefits of Kronecker-factored preconditioning with the simplicity of Adam. 
A second line of work improves the update rule itself. 
ADOPT \citep{adopt} identifies the correlation between the current gradient and the second-moment estimate as the root cause of Adam's theoretical non-convergence, and resolves it by excluding the current gradient from the second-moment computation and reordering the momentum and normalization steps, achieving the optimal $\mathcal{O}(1/\sqrt{T})$ rate with any $\beta_2$.
Adan \citep{adan} integrates Nesterov momentum into the adaptive gradient framework, and MARS \citep{mars} is the first to successfully bring variance reduction into large-scale training by reconciling preconditioned gradient methods with a scaled stochastic recursive momentum, consistently outperforming AdamW on GPT-2 pre-training. 
Finally, a practical but important direction aims at reducing the overhead of Adam-style optimizers. 
Adam-mini \citep{adam_mini} exploits the block-diagonal Hessian structure to partition parameters into groups that share a single learning rate, substantially cutting optimizer memory while matching Adam's performance. Prodigy \citep{prodigy} takes a complementary approach by eliminating the need for learning rate tuning altogether through a D-adaptation mechanism that automatically estimates the optimal step size during training.

\paragraph{Muon optimizer.}

Muon \cite{muon} is a recently proposed optimizer designed for the hidden-layer parameters of neural networks. 
It orthogonalizes the momentum buffer via Newton-Schulz (NS) iteration, replacing the gradient's singular values with uniform ones so that the update utilizes all directions in parameter space equally. 
This design has demonstrated substantial speedups over AdamW in large-scale LLM pre-training \citep{kimi_k2,glm5,dmuon}. 
~\cite{muon_tail} provide a theoretical explanation for Muon's advantage from the perspective of associative memory learning. 
A central practical concern with Muon is the cost of the multi-step NS iteration process. 
Therefore, several works explore accelerating the NS iteration itself.
Polar Express \citep{pe} derives optimal polynomial coefficients for the matrix sign iteration, CANS \citep{cans} replaces hand-tuned coefficients with Chebyshev-type polynomials for faster spectral convergence, Turbo-Muon \citep{turbo_muon} preconditions the input matrix to concentrate its spectrum and thus reduce the required number of NS steps, and PRISM \citep{prism} introduces a distribution-free adaptive framework for general matrix function computation including orthogonalization. 
More recently, Gram Newton-Schulz \citep{GramMuon} exploits Gram-matrix structure for a hardware-aware implementation.
A parallel line of work reduces or replaces the NS step altogether to lower communication and computation overhead. 
Dion \citep{dion} substitutes NS with amortized power iteration and low-rank approximation, enabling faster updates without full-matrix reconstruction under FSDP and tensor parallelism. 
Drop-Muon \citep{drop_muon} and MuonBP \citep{muonbp} perform orthogonalization only periodically while provably maintaining convergence.
NorMuon \cite{normuon} replaces NS with column normalization to achieve comparable performance at a fraction of the cost. 
Finally, because Muon's orthogonalization discards singular-value magnitudes, it forgoes the coordinate-wise or curvature-aware adaptivity that Adam provides. 
Several methods address this by hybridizing orthogonal updates with adaptive mechanisms. 
AdaMuon \citep{adamuon} reintroduces element-wise adaptive learning rates in Adam into the Muon framework, and \cite{adagrad_muon} equips orthogonal updates with an AdaGrad-Norm step size. 
COSMOS \citep{cosmos} fuses ideas from SOAP and Muon by running Adam in Muon's eigenbasis for memory-efficient hybrid preconditioning. 
ASGO \citep{asgo} designs a single-sided adaptive preconditioner that exploits the low-rank gradient and block-diagonal Hessian structure observed in Transformers, and ROOT \citep{root} augments orthogonalized updates with robustness mechanisms to improve training stability.

\paragraph{Benchmarking and understanding optimizers.}

Alongside the development of new optimizers, a parallel line of work focuses on systematically benchmarking, deconstructing, and understanding what makes an optimizer effective for LLM pre-training. 
\cite{deconstructing} deconstructs the design choices behind optimizers for autoregressive language models, identifying the key components, such as sign-based updates and momentum, that most strongly influence performance. 
In a similar spirit, \cite{fantastic_optimizers} conducts a comprehensive empirical study across a range of pre-training optimizers, characterizing the conditions under which different methods excel and offering practical guidelines for optimizer selection. 
In an investigation of orthogonalization-based methods, \cite{what_really_matters} isolate key design variables in whitening optimizers. 
Their findings suggest that empirical improvement is attributable to only a few critical choices, particularly the application of momentum before the orthogonalization step.
From a more theoretical perspective, \cite{adam_secret_sauce} investigates the mechanisms underlying Adam's robustness, seeking to identify the secret sauce that makes it consistently competitive despite its simplicity. 
Similarly, \cite{practical_efficiency} evaluates the practical wall-clock efficiency of Muon at scale, reporting that while Muon achieves faster convergence in terms of steps, its per-step overhead from NS iteration can offset gains depending on hardware and parallelism configuration. 

\paragraph{Data selection for pre-training LLMs.}

Recently, a growing body of work has explored how to curate high-quality subsets from large-scale corpora for pre-training LLMs, and existing methods can be broadly organized into several categories. 
The first line of work performs \textit{reference-model-based scoring}, using a pre-trained model to assign quality signals to individual documents. DSIR \citep{dsir} selects examples whose distribution aligns with a curated target corpus via importance resampling, while \cite{ppl} investigate perplexity-based pruning and find that small reference models can serve as effective proxies. 
Beyond single-score filtering, \textit{multi-dimensional quality rating} methods assess data along multiple interpretable dimensions. 
QuRating \citep{qurating} scores documents on facets such as writing style and expertise, Meta-Rater \citep{meta_rater} learns a joint multi-dimensional scoring function, and DataRater \citep{datarater} meta-learns a curation policy that generalizes across tasks. 
A third category, \textit{model-aware selection}, ties data quality to the evolving state of the model being trained. 
DsDm \citep{dsdm} estimates each example's marginal contribution via datamodels \citep{datamodels}, MATES \citep{mates} periodically updates lightweight influence models during pre-training, \cite{optimal_control} dynamically adjusts the training distribution as a sequential decision problem, and \cite{predictive} identifies the most instructive examples by their ability to predict learning outcomes. 
Complementing instance-level methods, several approaches address \textit{corpus-level diversity}. 
D4 \citep{d4} combines document de-duplication with diversification, \cite{organize} constructs coherent domain clusters for balanced curation, and the diversity-aware framework of \cite{harnessing} explicitly optimizes distributional coverage. 
Finally, \textit{system-level strategies} utilize multiple signals within a unified pipeline. 
\cite{multi_actor} coordinates specialized selection actors to jointly curate the corpus.

\section{Details of AMO}
\label{app:proofs}

This appendix fills in the \textsc{PE-simulate} from \cite{pe} and our \textsc{BudgetAllocate} operation invoked in Alg.~\ref{alg:amo}.
Also, we provide an optimality guarantee for \textsc{BudgetAllocate} output under standard convexity assumptions.
Sec.~\ref{app:ba_algo} gives the full algorithm and implementation details, and Sec.~\ref{app:ba_theory} states and proves the optimality theorem.

\subsection{Algorithm}
\label{app:ba_algo}

\textsc{PE-simulate}$(\ell, k)$ is a scalar simulation of the PE-composed NS trajectory: starting from an input-side $\ell$, it applies $k$ NS iterations with coefficients $\{(a_j, b_j, c_j)\}_{j=1}^{k}$ produced by the PE composer at $(\ell, k)$, and returns the resulting post-NS effective $\ell$, which is close to $1$ when orthogonalization succeeds~\citep{pe}. 
\textsc{BudgetAllocate} takes as input the seven projected error curves $\{E_i(k)\}_{i=1}^{|\mathcal{T}|}$ with $|\mathcal{T}| = 7$, 
where  $E_i(k) = 1 - \textsc{PE-simulate}(\ell_{\mathrm{target}}^{(i)}, k)$ 
is the post-NS error for operator type $i$ at step count $k$, together with the total budget $B$ and the step range $[k_{\min}, k_{\max}]$, and returns a per-type allocation $\{T_i^\star\}_{i=1}^{|\mathcal{T}|}$ satisfying $\sum_i T_i^\star = B$ and $T_i^\star \in [k_{\min}, k_{\max}]$. 
Alg.~\ref{alg:budget_allocate} gives the full process of \textsc{BudgetAllocate}: it initializes all types at $T_{\mathrm{base}}$ clipped to the range, then applies three passes: \textit{add} until the budget is met, \textit{remove} if over budget, and \textit{transfer} to eliminate any remaining pairwise improvement. 
The \textit{add} and \textit{remove} passes alone cannot resolve cases where a lower-error reallocation exists without crossing the budget boundary; the \textit{transfer} pass is what lets \textsc{BudgetAllocate} terminate at a global rather than merely local optimum (Sec.~\ref{app:ba_theory}).

\begin{algorithm}[t]
\small
\caption{\textsc{BudgetAllocate}}
\label{alg:budget_allocate}
\begin{algorithmic}[1]
\Require projected error curves $\{E_i(k)\}_{i=1}^{|\mathcal{T}|}$, 
budget $B$, step range $[k_{\min}, k_{\max}]$, baseline 
$T_{\mathrm{base}}$
\State \textbf{Optional relaxation:} if $B < |\mathcal{T}| \cdot k_{\min}$, 
decrease $k_{\min}$ (floor at $1$) until feasible; if 
$B > |\mathcal{T}| \cdot k_{\max}$, increase $k_{\max}$ until feasible.
\State Initialize $T_i \gets \mathrm{clip}(T_{\mathrm{base}}, k_{\min}, k_{\max})$ for all $i$
\State \textbf{// \textit{Add} pass}
\While{$\sum_i T_i < B$}
  \State $i^\star \gets \arg\max_{i\,:\, T_i < k_{\max}} 
  \bigl[ E_i(T_i) - E_i(T_i + 1) \bigr]$
  \State $T_{i^\star} \gets T_{i^\star} + 1$
\EndWhile
\State \textbf{// \textit{Remove} pass}
\While{$\sum_i T_i > B$}
  \State $i^\star \gets \arg\min_{i\,:\, T_i > k_{\min}} 
  \bigl[ E_i(T_i - 1) - E_i(T_i) \bigr]$
  \State $T_{i^\star} \gets T_{i^\star} - 1$
\EndWhile
\State \textbf{// \textit{Transfer} pass}
\While{$\exists (i, j)$ with $T_i < k_{\max}$, $T_j > k_{\min}$, 
and $\bigl[E_i(T_i) - E_i(T_i+1)\bigr] > \bigl[E_j(T_j-1) - E_j(T_j)\bigr]$}
  \State $T_i \gets T_i + 1$; $T_j \gets T_j - 1$
\EndWhile
\State \Return $\{T_i^\star\} \gets \{T_i\}$
\end{algorithmic}
\end{algorithm}

Two details are worth noting. 
The optional relaxation step is triggered only when the derived budget 
$B = r \cdot |\mathcal{T}| \cdot T_{\mathrm{base}}$ is infeasible under $[k_{\min}, k_{\max}]$, which happens when $r$ falls outside $[k_{\min}/T_{\mathrm{base}},\, k_{\max}/T_{\mathrm{base}}]$.
In all experiments reported in this paper, 
$[k_{\min}, k_{\max}] = [3, 7]$, $T_{\mathrm{base}} = 5$, and 
$r = 1.0$, so the derived budget $B = 35$ sits strictly inside the 
feasible range $[21, 49]$ and no relaxation is ever triggered. 
Computationally, \textsc{BudgetAllocate} operates entirely on 
scalar projected errors and terminates within a few dozen 
iterations in all cases, so its runtime is \textbf{negligible} relative to 
a single training step.

\paragraph{Implementations of AMO.}

For AMO, all runs start from the uniform PE schedule with 5-step NS iteration for each of the seven Muon layer types
(\texttt{attn\_q}, \texttt{attn\_k}, \texttt{attn\_v}, \texttt{attn\_o}, \texttt{mlp\_gate}, \texttt{mlp\_up}, \texttt{mlp\_down}). 
The observation horizon is 1200 steps for Qwen3-0.6B and Llama-3.1-760M, and 3000 steps for Qwen3-1.7B and Llama-3.1-1.4B. 
We estimate $\sigma_{\min}$ with exact SVD in fp32, use budget ratio 1.0, and search PE schedules over NS step counts $[3,7]$. 
Before allocation, the robust observed $\ell$ is shrunk toward the baseline by $\ell_{\mathrm{target}}=0.7\ell_{\mathrm{obs}}+0.3\cdot 10^{-3}$.
The resulting target schedules are reached through a 300-step transition duration and then locked for the remainder of training, with no post-lock refresh.
All configurations of AMO are provided in Tab.~\ref{tab:amo_configs} and details of final AMO NS schedules locked after transition are provided in Tab.~\ref{tab:amo_schedule_configs}.

\begin{table*}[t]
\centering
\scriptsize
\setlength{\tabcolsep}{1.5pt}
\caption{Details of final locked NS schedules of AMO under budget ratio $r=1.0$. $\ell_{\mathrm{target}}$ is the effective target value logged at transition planning and coefficients are recomputed with PE.}
\begin{tabular}{llclp{0.57\linewidth}}
\toprule
\textbf{Model} & \textbf{Operator} & \textbf{NS Step} $T$ & $\ell_{\mathrm{target}}$ & \textbf{NS Coefficient} $(a_k,b_k,c_k)_{k=1}^T$ \\
\midrule
\multirow{14}{*}{Qwen3-0.6B}
& \texttt{attn\_q} & 5 & $1.36334{\times}10^{-3}$ & (8.2244,-23.121,16.654); (4.0628,-2.8881,0.52601); (3.8427,-2.8118,0.52947); (3.0757,-2.2601,0.46810); (2.1522,-1.5322,0.40416) \\
& \texttt{attn\_k} & 6 & $3.02222{\times}10^{-4}$ & (8.2612,-23.225,16.729); (4.1194,-2.9025,0.52399); (4.0687,-2.8896,0.52579); (3.8707,-2.8312,0.53163); (3.1377,-2.3064,0.47322); (2.1930,-1.5705,0.40824) \\
& \texttt{attn\_v} & 6 & $3.00606{\times}10^{-4}$ & (8.2612,-23.225,16.729); (4.1194,-2.9025,0.52399); (4.0691,-2.8897,0.52577); (3.8723,-2.8324,0.53176); (3.1415,-2.3092,0.47353); (2.1956,-1.5730,0.40850) \\
& \texttt{attn\_o} & 5 & $1.48010{\times}10^{-3}$ & (8.2204,-23.110,16.646); (4.0567,-2.8865,0.52624); (3.8145,-2.7922,0.52728); (3.0160,-2.2152,0.46314); (2.1167,-1.4983,0.40058) \\
& \texttt{mlp\_gate} & 4 & $4.51844{\times}10^{-3}$ & (8.1170,-22.819,16.437); (3.9063,-2.8501,0.53278); (3.2316,-2.3758,0.48091); (2.2633,-1.6351,0.41517) \\
& \texttt{mlp\_up} & 5 & $3.92888{\times}10^{-3}$ & (8.1369,-22.875,16.477); (3.9343,-2.8566,0.53143); (3.3254,-2.4443,0.48852); (2.3218,-1.6582,0.40255); (1.8981,-1.2755,0.37756) \\
& \texttt{mlp\_down} & 4 & $5.87911{\times}10^{-3}$ & (8.0715,-22.692,16.345); (3.8286,-2.8020,0.52837); (3.0454,-2.2374,0.46559); (2.1338,-1.5147,0.40230) \\
\midrule
\multirow{14}{*}{Qwen3-1.7B}
& \texttt{attn\_q} & 6 & $3.00377{\times}10^{-4}$ & (8.2613,-23.225,16.729); (4.1195,-2.9026,0.52399); (4.0691,-2.8897,0.52577); (3.8726,-2.8325,0.53178); (3.1420,-2.3096,0.47357); (2.1960,-1.5733,0.40854) \\
& \texttt{attn\_k} & 5 & $1.26601{\times}10^{-3}$ & (8.2278,-23.131,16.661); (4.0679,-2.8893,0.52582); (3.8667,-2.8284,0.53132); (3.1287,-2.2996,0.47247); (2.1868,-1.5647,0.40762) \\
& \texttt{attn\_v} & 5 & $7.42367{\times}10^{-4}$ & (8.2459,-23.182,16.698); (4.0956,-2.8964,0.52481); (3.9760,-2.8666,0.52953); (3.4770,-2.5538,0.50069); (2.5062,-1.8479,0.43834) \\
& \texttt{attn\_o} & 6 & $3.00361{\times}10^{-4}$ & (8.2613,-23.225,16.729); (4.1195,-2.9026,0.52399); (4.0691,-2.8897,0.52577); (3.8726,-2.8325,0.53178); (3.1421,-2.3096,0.47357); (2.1960,-1.5734,0.40854) \\
& \texttt{mlp\_gate} & 4 & $2.05821{\times}10^{-3}$ & (8.2005,-23.054,16.606); (4.0268,-2.8791,0.52740); (3.6822,-2.6995,0.51693); (2.7982,-2.0880,0.46496) \\
& \texttt{mlp\_up} & 5 & $1.73407{\times}10^{-3}$ & (8.2116,-23.086,16.629); (4.0435,-2.8832,0.52674); (3.7549,-2.7506,0.52263); (2.8987,-2.1259,0.45330); (2.0567,-1.4398,0.39443) \\
& \texttt{mlp\_down} & 4 & $4.25243{\times}10^{-3}$ & (8.1259,-22.845,16.455); (3.9189,-2.8530,0.53216); (3.2730,-2.4061,0.48427); (2.2978,-1.6663,0.41853) \\
\midrule
\multirow{14}{*}{Llama-3.1-760M}
& \texttt{attn\_q} & 6 & $3.00856{\times}10^{-4}$ & (8.2612,-23.225,16.729); (4.1194,-2.9025,0.52399); (4.0690,-2.8896,0.52578); (3.8721,-2.8322,0.53174); (3.1409,-2.3087,0.47348); (2.1952,-1.5726,0.40846) \\
& \texttt{attn\_k} & 4 & $3.07460{\times}10^{-3}$ & (8.1658,-22.957,16.536); (3.9758,-2.8665,0.52954); (3.4762,-2.5532,0.50063); (2.5052,-1.8471,0.43825) \\
& \texttt{attn\_v} & 5 & $1.70416{\times}10^{-3}$ & (8.2127,-23.088,16.631); (4.0450,-2.8836,0.52668); (3.7618,-2.7554,0.52317); (2.9117,-2.1359,0.45439); (2.0627,-1.4458,0.39505) \\
& \texttt{attn\_o} & 6 & $3.00986{\times}10^{-4}$ & (8.2612,-23.225,16.729); (4.1194,-2.9025,0.52399); (4.0690,-2.8896,0.52578); (3.8719,-2.8321,0.53173); (3.1406,-2.3085,0.47345); (2.1950,-1.5724,0.40844) \\
& \texttt{mlp\_gate} & 5 & $2.54507{\times}10^{-3}$ & (8.1838,-23.007,16.572); (4.0022,-2.8730,0.52841); (3.5796,-2.6270,0.50884); (2.6139,-1.9020,0.42880); (1.9545,-1.3357,0.38367) \\
& \texttt{mlp\_up} & 5 & $1.97798{\times}10^{-3}$ & (8.2033,-23.062,16.612); (4.0309,-2.8801,0.52724); (3.6999,-2.7119,0.51832); (2.8003,-2.0498,0.44494); (2.0150,-1.3981,0.39009) \\
& \texttt{mlp\_down} & 4 & $3.76471{\times}10^{-3}$ & (8.1424,-22.891,16.488); (3.9422,-2.8585,0.53106); (3.3529,-2.4644,0.49074); (2.3714,-1.7317,0.42562) \\
\midrule
\multirow{14}{*}{Llama-3.1-1.4B}
& \texttt{attn\_q} & 6 & $3.00393{\times}10^{-4}$ & (8.2613,-23.225,16.729); (4.1195,-2.9026,0.52399); (4.0691,-2.8897,0.52577); (3.8725,-2.8325,0.53178); (3.1420,-2.3095,0.47357); (2.1960,-1.5733,0.40854) \\
& \texttt{attn\_k} & 4 & $2.93098{\times}10^{-3}$ & (8.1707,-22.970,16.546); (3.9829,-2.8683,0.52923); (3.5034,-2.5727,0.50280); (2.5386,-1.8753,0.44135) \\
& \texttt{attn\_v} & 5 & $1.20276{\times}10^{-3}$ & (8.2300,-23.137,16.666); (4.0712,-2.8902,0.52570); (3.8824,-2.8394,0.53254); (3.1647,-2.3264,0.47544); (2.2121,-1.5883,0.41014) \\
& \texttt{attn\_o} & 6 & $3.00379{\times}10^{-4}$ & (8.2613,-23.225,16.729); (4.1195,-2.9026,0.52399); (4.0691,-2.8897,0.52577); (3.8726,-2.8325,0.53178); (3.1420,-2.3096,0.47357); (2.1960,-1.5733,0.40854) \\
& \texttt{mlp\_gate} & 5 & $1.98534{\times}10^{-3}$ & (8.2030,-23.061,16.611); (4.0306,-2.8800,0.52725); (3.6982,-2.7108,0.51819); (2.7975,-2.0476,0.44470); (2.0139,-1.3970,0.38998) \\
& \texttt{mlp\_up} & 5 & $1.58149{\times}10^{-3}$ & (8.2169,-23.100,16.639); (4.0514,-2.8852,0.52644); (3.7904,-2.7754,0.52540); (2.9672,-2.1782,0.45906); (2.0903,-1.4727,0.39788) \\
& \texttt{mlp\_down} & 4 & $4.32035{\times}10^{-3}$ & (8.1237,-22.838,16.450); (3.9156,-2.8522,0.53232); (3.2623,-2.3982,0.48340); (2.2886,-1.6581,0.41764) \\
\bottomrule
\end{tabular}
\label{tab:amo_schedule_configs}
\end{table*}

\paragraph{Training loss dynamics.}

\begin{figure}
    \centering
    \includegraphics[width=1.0\linewidth]{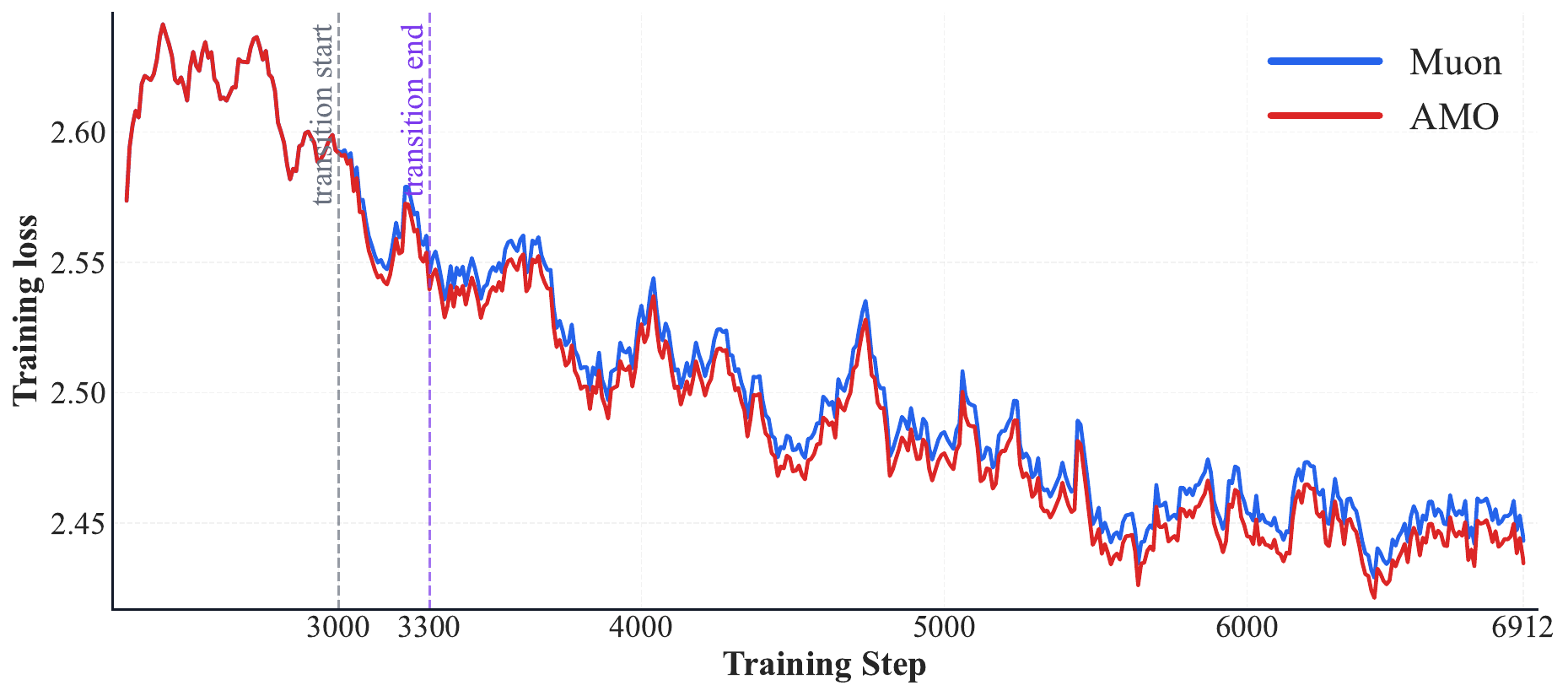}
    \caption{Training loss curves of AMO and Muon in pre-training Qwen3-1.7B.}
    \label{fig:amo_loss_curve}
\end{figure}

To gain deeper insights into the training dynamics, we present a comparison of the training loss curves of AMO and Muon during the pre-training of Qwen3-1.7B in Fig.~\ref{fig:amo_loss_curve}. 
It can be clearly observed that AMO consistently attains a lower training loss than Muon from the transition stage onward, indicating that the adaptive update yields more effective optimization under the same training budget (with the default budget ratio set to $r=1.0$).

\begin{table}
\centering
\small
\setlength{\tabcolsep}{2.4pt}
\caption{Full configurations of AMO for pre-training Llama3.1-760M/1.4B and Qwen3-0.6B/1.7B models.}
\begin{tabular}{lccc}
\toprule
\textbf{Model} & \textbf{Observation Horizon \& Interval \& Sample Size} & \textbf{Shrinkage} & \textbf{Transition Duration} \\ \midrule
Qwen3-0.6B     & (1200, 150, 8)                                           & 0.7                & 300                          \\
Qwen3-1.7B     & (3000, 375, 8)                                           & 0.7                & 300                          \\
Llama-3.1-760M & (1200, 150, 8)                                           & 0.7                & 300                          \\
Llama-3.1-1.4B & (3000, 375, 8)                                           & 0.7                & 300                          \\ \bottomrule
\end{tabular}
\label{tab:amo_configs}
\end{table}

\subsection{Optimality Analysis}
\label{app:ba_theory}

We analyze \textsc{BudgetAllocate} under four assumptions. 
Three are by construction or follow from established results: (A1) each $T_i$ ranges over integers in $[k_{\min}, k_{\max}]$; (A2) each $E_i$ is strictly decreasing on this range, as more NS iterations cannot increase the post-NS error under the Chebyshev-based bound that PE-composed polynomials inherit~\citep{pe,cans}; and (A4) the budget satisfies $|\mathcal{T}| \cdot k_{\min} \le B \le |\mathcal{T}| \cdot k_{\max}$ under the (possibly relaxed) range. 
The remaining assumption is the only non-trivial one: (A3) each $E_i$ is \textit{discretely convex}, meaning the marginal gain $\Delta E_i(k) := E_i(k) - E_i(k+1)$ is non-increasing in $k$, or equivalently $E_i(k+1) + E_i(k-1) \ge 2 E_i(k)$. 
Intuitively, the benefit of adding one more NS iteration diminishes as the step count grows. 

Under these assumptions, the allocation problem is a separable minimization of discretely convex functions under an integer box constraint with one linear equality, a classical instance of separable discrete convex optimization, and \textsc{BudgetAllocate} solves it to global optimality:

\begin{theorem}[Optimality of \textsc{BudgetAllocate}]
\label{thm:budget_allocate}
Under (A1)–(A4), the allocation $\{T_i^\star\}$ returned by 
Alg.~\ref{alg:budget_allocate} satisfies
\[
\{T_i^\star\} \;\in\; \arg\min_{\{T_i\}} \;\sum_{i=1}^{|\mathcal{T}|} E_i(T_i),
\quad \text{subject to} \quad \sum_i T_i = B,\ \ T_i \in [k_{\min}, k_{\max}].
\]
\end{theorem}

\begin{proof}[Proof of Theorem~\ref{thm:budget_allocate}.]
We first establish a necessary and sufficient optimality condition, then show that \textsc{BudgetAllocate} terminates at a point satisfying it.

\paragraph{Optimality condition.}

We claim that a feasible allocation $\{T_i\}$ minimizes 
$\sum_i E_i(T_i)$ subject to $\sum_i T_i = B$ and 
$T_i \in [k_{\min}, k_{\max}]$ \emph{if and only if}
\begin{equation}
\Delta E_i(T_i) \;\le\; \Delta E_j(T_j - 1)
\qquad \text{for all } (i, j) \text{ with } T_i < k_{\max} 
\text{ and } T_j > k_{\min}.
\tag{$\star$}
\label{eq:opt_cond}
\end{equation}

The "only if" direction is immediate. 
If $(\star)$ fails at some pair $(i, j)$, then 
$\Delta E_i(T_i) > \Delta E_j(T_j - 1)$, and the transfer 
$(T_i, T_j) \to (T_i + 1, T_j - 1)$ is feasible (since 
$T_i < k_{\max}$ and $T_j > k_{\min}$) and changes the objective 
by
\[
\bigl[E_i(T_i + 1) - E_i(T_i)\bigr] + \bigl[E_j(T_j - 1) - E_j(T_j)\bigr]
= -\Delta E_i(T_i) + \Delta E_j(T_j - 1) 
< 0,
\]
strictly decreasing the objective and contradicting optimality.

The "if" direction uses A3. 
Suppose $(\star)$ holds at $\{T_i\}$ but there is a strictly better feasible 
$\{T_i'\}$, so $\sum_i E_i(T_i') < \sum_i E_i(T_i)$. 
Let $I^+ = \{i : T_i' > T_i\}$ and $I^- = \{i : T_i' < T_i\}$. 
Since both $\sum_i T_i = B$ and $\sum_i T_i' = B$, we have 
$\sum_{i \in I^+} (T_i' - T_i) = \sum_{i \in I^-} (T_i - T_i') 
=: L$. We can therefore construct a sequence of $L$ 
single-unit transfers that transforms $\{T_i\}$ into $\{T_i'\}$: 
at each step, pick any $i \in I^+$ with the current $T_i$ below 
its target $T_i'$ and any $j \in I^-$ with the current $T_j$ 
above its target, and perform 
$(T_i, T_j) \to (T_i + 1, T_j - 1)$. 
Every intermediate 
allocation in this sequence remains in 
$[k_{\min}, k_{\max}]^{|\mathcal{T}|}$ and has total $B$, so 
each transfer is feasible.

We now track the objective change along this path. 
Consider the 
$\ell$-th transfer, applied to a current allocation $\{T_i^{(\ell)}\}$ 
with the chosen pair $(i_\ell, j_\ell)$. The transfer changes the 
objective by 
$-\Delta E_{i_\ell}(T_{i_\ell}^{(\ell)}) + \Delta E_{j_\ell}(T_{j_\ell}^{(\ell)} - 1)$. 
At the \emph{starting} allocation $\{T_i\}$, condition 
$(\star)$ gives
$\Delta E_{i_\ell}(T_{i_\ell}) \le \Delta E_{j_\ell}(T_{j_\ell} - 1)$. 
Along the path, $T_{i_\ell}$ only increases (since 
$i_\ell \in I^+$) and $T_{j_\ell}$ only decreases (since 
$j_\ell \in I^-$), so $T_{i_\ell}^{(\ell)} \ge T_{i_\ell}$ and 
$T_{j_\ell}^{(\ell)} - 1 \le T_{j_\ell} - 1$. By A3 (marginal 
gain non-increasing in $k$),
\[
\Delta E_{i_\ell}(T_{i_\ell}^{(\ell)}) \le 
\Delta E_{i_\ell}(T_{i_\ell}),
\qquad
\Delta E_{j_\ell}(T_{j_\ell}^{(\ell)} - 1) \ge 
\Delta E_{j_\ell}(T_{j_\ell} - 1),
\]
which combined with $(\star)$ yields
\[
\Delta E_{i_\ell}(T_{i_\ell}^{(\ell)}) \le 
\Delta E_{i_\ell}(T_{i_\ell}) \le 
\Delta E_{j_\ell}(T_{j_\ell} - 1) \le 
\Delta E_{j_\ell}(T_{j_\ell}^{(\ell)} - 1).
\]
Hence the objective change at the $\ell$-th transfer is 
non-negative:
\[
-\Delta E_{i_\ell}(T_{i_\ell}^{(\ell)}) + 
\Delta E_{j_\ell}(T_{j_\ell}^{(\ell)} - 1) \ge 0.
\]
Summing over $\ell = 1, \ldots, L$, the total objective change 
from $\{T_i\}$ to $\{T_i'\}$ is non-negative, i.e.\ 
$\sum_i E_i(T_i') \ge \sum_i E_i(T_i)$, contradicting the 
assumption $\sum_i E_i(T_i') < \sum_i E_i(T_i)$. This proves the 
"if" direction, so $(\star)$ characterizes global optimality.

\paragraph{BudgetAllocate terminates at a point satisfying $(\star)$.}
It remains to show that the output of 
Alg.~\ref{alg:budget_allocate} satisfies $(\star)$.

First, the add and remove passes terminate and ensure 
$\sum_i T_i = B$ at their end. If initialization already 
satisfies $\sum_i T_i = B$, both while-loops exit immediately. 
Otherwise exactly one is entered; the add pass (the remove case 
is symmetric) executes at most $B - \sum_i T_i$ iterations, each 
increasing $\sum_i T_i$ by $1$, after which $\sum_i T_i = B$ and 
the loop exits. Because each step picks 
$\arg\max_i [E_i(T_i) - E_i(T_i + 1)]$ among types with 
$T_i < k_{\max}$, these passes respect the box constraint.

The transfer pass then strictly decreases $\sum_i E_i(T_i)$ at 
every iteration, by the loop condition 
$\Delta E_i(T_i) > \Delta E_j(T_j - 1)$ combined with the 
identity 
$-\Delta E_i(T_i) + \Delta E_j(T_j - 1) < 0$ derived above. 
Since each $E_i$ takes finitely many values on 
$\{k_{\min}, \ldots, k_{\max}\}$, the objective 
$\sum_i E_i(T_i)$ also takes finitely many values, so strict 
decrease forces termination in finitely many iterations. At 
termination, no pair $(i, j)$ with $T_i < k_{\max}$ and 
$T_j > k_{\min}$ satisfies 
$\Delta E_i(T_i) > \Delta E_j(T_j - 1)$, which is exactly 
$(\star)$.

Combining: at the output of \textsc{BudgetAllocate}, 
$\sum_i T_i = B$, $T_i \in [k_{\min}, k_{\max}]$, and 
$(\star)$ holds, so the output is a global minimizer of 
$\sum_i E_i(T_i)$ under the stated constraints.
\end{proof}

\section{Pilot Analysis}

In these two pilot experiments, we select a total of eight models for observation.
Full details of model architecture, datasets, and training hyperparameters are provided in App.~\ref{app:subsec_exp_setting}. 

\subsection{Data Curriculum}
\label{app:data_curriculum}

\paragraph{Setup.}

We conduct a factorial study over two base models, \textbf{Llama3.1-760M} and \textbf{Qwen3-0.6B}, two optimizers (Muon and AdamW), four learning-rate schedulers, and three data orders. 
All runs use the same \texttt{FineWeb-Edu} dataset, packed training sequences of length 2048, bf16 mixed precision, gradient clipping at 1.0, and seed 42. 
Total training tokens follow 1x Chinchilla optimal ratio of 20.

The data curriculum is implemented as a fixed-pool reordering scheme. 
The \texttt{uniform} setting applies a seeded shuffle to the pool. 
The \texttt{ascend} and \texttt{descend} settings sort examples by the \textit{Educational Value} score field in increasing or decreasing order, respectively. 
Therefore, all curriculum variants see the same data pool and differ only in the \textbf{presentation order} of examples.

For learning-rate control, we compare four scheduler families. We first define the reference stable learning rate by the empirical rule from \cite{lr_law}
\[
\eta_{\mathrm{stable}} = 10^{-4} \cdot 38.4588 \cdot N^{-0.2219} \cdot D^{-0.3509},
\]
where $N$ and $D$ denote model size and reference training tokens in billions. 
This gives $\eta_{\mathrm{stable}}=1.573\times 10^{-3}$ for Llama3.1-760M and $\eta_{\mathrm{stable}}=1.801\times 10^{-3}$ for Qwen3-0.6B. 

As listed in Tab.~\ref{tab:data_curriculum_schedulers}, we consider four LR schedulers. 
\textbf{WS (Warmup then Stable)} uses linear warmup to $\eta_{\mathrm{stable}}$ followed by a constant plateau. 
\textbf{WSD (Warmup, Stable, then Decay)} uses linear warmup, a stable phase at $\eta_{\mathrm{stable}}$, and then a linear decay to $0.1\,\eta_{\mathrm{stable}}$. 
\textbf{Cosine} uses linear warmup followed by cosine decay toward $0.1\,\eta_{\mathrm{stable}}$. 
\textbf{Constant} keeps learning rate fixed throughout training.

For Qwen3-0.6B, WS uses 143 warmup steps and 2717 stable steps; WSD uses 143/2431/286 warmup/stable/decay steps; Cosine uses 2860 total steps with 143 warmup steps; and the Constant baseline uses a fixed learning rate of $1.8\times 10^{-3}$, effectively matching $\eta_{\mathrm{stable}}$. 

For Llama3.1-760M, WS uses 181 warmup steps and 3441 stable steps; WSD uses 181/3079/362 warmup/stable/decay steps; Cosine is configured with 3866 total steps and 193 warmup steps, while the actual grid runs for 3622 optimizer steps, so the final learning rate remains slightly above the nominal cosine floor; and the Constant baseline uses a smaller fixed learning rate of $3.0\times 10^{-4}$.

Optimizer hyperparameters are otherwise fixed within each model. 
AdamW uses $\beta=(0.9,0.95)$, $\epsilon=10^{-8}$, and weight decay 0.1. Muon uses momentum 0.95, Nesterov momentum, five NS steps, and weight decay 0.1; non-matrix parameters are handled by AdamW. 

\begin{table}[t]
\centering
\small
\setlength{\tabcolsep}{2pt}
\caption{LR scheduler configurations used in data curriculum. WS denotes warmup-stable, WSD denotes warmup-stable-decay, and Cosine denotes warmup and cosine decay.  The same scheduler settings are used for both Muon and AdamW within each model.}
\begin{tabular}{lcccccc}
\toprule
\textbf{Model} & \textbf{Total Steps} & $\eta_{\mathrm{stable}}$ & \textbf{WS} $(w,s)$ & \textbf{WSD} $(w,s,d)$ & \textbf{Cosine} $(w,d)$ & \textbf{Constant} \\
\midrule
Llama3.1-760M & 3622 & $1.573\times10^{-3}$ & $(181,3441)$ & $(181,3079,362)$ & $(193,3866)$ & $3.0\times10^{-4}$ \\
Qwen3-0.6B    & 2860 & $1.801\times10^{-3}$ & $(143,2717)$ & $(143,2431,286)$ & $(143,2860)$ & $1.8\times10^{-3}$ \\
\bottomrule
\end{tabular}
\label{tab:data_curriculum_schedulers}
\end{table}

\begin{table*}[t]
\centering
\footnotesize
\caption{Data curriculum results of Llama3.1-760M with (a) Muon and (b) AdamW. Within each LR scheduler setting, the pattern $\texttt{Uniform} > \texttt{Descend} > \texttt{Ascend}$ holds consistently. Abbreviations are as follows. AE: ARC-Easy; AC: ARC-Challenge; MM: MMLU; MM-P: MMLU-pro; HS: Hellaswag; OBQA: OpenBookQA; WG: WinoGrande; CSQA: CommonsenseQA; AGI: AGIEval-en.}
\label{tab:llama_curriculum}
\setlength{\tabcolsep}{1pt}

\begin{subtable}{\textwidth}
\centering
\caption{Muon}
\label{tab:llama_curriculum_muon}
\begin{tabular}{@{}llccccccccccccc@{}}
\toprule
\textbf{LR Scheduler} & \textbf{Order} & \textbf{AE} & \textbf{AC} & \textbf{SciQ} & \textbf{MM} & \textbf{MM-P} & \textbf{HS} & \textbf{OBQA} & \textbf{PIQA} & \textbf{RACE} & \textbf{WG} & \textbf{CSQA} & \textbf{AGI} & \textbf{Avg} \\ \midrule
\multirow{3}{*}{Const}  & Ascend  & 49.45 & 22.70 & 74.00 & 22.84 & 8.75 & 30.68 & 19.40 & 63.87 & 30.62 & 50.67 & 19.49 & 17.06 & 34.13 \\
                        & Descend & 53.96 & 21.42 & 76.40 & 24.40 & 7.73 & 30.10 & 21.80 & 64.15 & 29.67 & 50.43 & 19.57 & 17.11 & 34.73 \\
                        & Uniform & 56.90 & 23.21 & 75.40 & 23.00 & 7.26 & 30.49 & 20.60 & 64.58 & 28.23 & 51.54 & 19.66 & 17.26 & \textbf{34.84} \\ \midrule
\multirow{3}{*}{WS}     & Ascend  & 55.26 & 22.61 & 76.50 & 22.97 & 5.35 & 33.61 & 21.00 & 66.21 & 31.10 & 50.91 & 18.84 & 16.90 & 35.11 \\
                        & Descend & 57.87 & 25.17 & 77.60 & 25.35 & 5.53 & 32.65 & 19.60 & 66.00 & 29.76 & 49.17 & 19.74 & 17.19 & 35.47 \\
                        & Uniform & 59.22 & 26.88 & 81.30 & 23.28 & 4.55 & 33.42 & 21.40 & 66.38 & 30.62 & 51.14 & 19.33 & 17.55 & \textbf{36.26} \\ \midrule
\multirow{3}{*}{WSD}    & Ascend  & 55.93 & 24.57 & 79.40 & 24.33 & 6.84 & 33.63 & 21.60 & 67.30 & 31.20 & 51.30 & 18.51 & 17.96 & 36.05 \\
                        & Descend & 60.94 & 26.70 & 80.80 & 24.56 & 6.67 & 32.87 & 24.80 & 67.30 & 31.39 & 52.96 & 21.87 & 17.55 & 37.37 \\
                        & Uniform & 61.49 & 26.88 & 81.10 & 24.76 & 7.73 & 33.76 & 24.60 & 67.41 & 31.29 & 52.46 & 21.87 & 17.08 & \textbf{37.54} \\ \midrule
\multirow{3}{*}{Cosine} & Ascend  & 54.29 & 21.59 & 79.50 & 24.11 & 7.48 & 33.19 & 21.00 & 66.65 & 30.72 & 51.54 & 19.74 & 17.13 & 35.58 \\
                        & Descend & 59.34 & 25.09 & 79.90 & 23.41 & 6.67 & 31.96 & 20.60 & 66.10 & 31.00 & 52.49 & 20.07 & 16.41 & 36.09 \\
                        & Uniform & 59.60 & 25.68 & 82.70 & 25.49 & 7.42 & 33.00 & 22.00 & 66.27 & 29.86 & 51.30 & 19.25 & 17.32 & \textbf{36.66} \\ \bottomrule
\end{tabular}
\end{subtable}

\vspace{0.8em}

\begin{subtable}{\textwidth}
\centering
\caption{AdamW}
\label{tab:llama_curriculum_adamw}
\begin{tabular}{@{}llccccccccccccc@{}}
\toprule
\textbf{LR Scheduler} & \textbf{Order} & \textbf{AE} & \textbf{AC} & \textbf{SciQ} & \textbf{MM} & \textbf{MM-P} & \textbf{HS} & \textbf{OBQA} & \textbf{PIQA} & \textbf{RACE} & \textbf{WG} & \textbf{CSQA} & \textbf{AGI} & \textbf{Avg} \\ \midrule
\multirow{3}{*}{Const}  & Ascend  & 48.32 & 20.73 & 69.90 & 23.02 & 6.77 & 28.90 & 16.00 & 61.92 & 28.33 & 52.33 & 19.57 & 17.60 & 32.78 \\
                        & Descend & 49.03 & 20.65 & 71.50 & 23.88 & 4.98 & 28.96 & 19.00 & 61.92 & 27.94 & 51.54 & 19.41 & 17.26 & 33.01 \\
                        & Uniform & 50.51 & 21.26 & 72.30 & 23.07 & 5.28 & 28.72 & 19.20 & 61.26 & 27.66 & 51.20 & 19.57 & 17.87 & \textbf{33.16} \\ \midrule
\multirow{3}{*}{WS}     & Ascend  & 54.46 & 23.29 & 76.70 & 23.68 & 6.93 & 32.29 & 21.00 & 66.43 & 29.67 & 51.78 & 19.82 & 17.52 & 35.30 \\
                        & Descend & 57.95 & 23.46 & 80.30 & 25.16 & 6.13 & 32.29 & 19.60 & 66.10 & 29.86 & 52.25 & 20.07 & 17.26 & 35.87 \\
                        & Uniform & 58.25 & 26.02 & 79.80 & 23.69 & 8.28 & 32.32 & 21.60 & 66.16 & 29.86 & 52.09 & 19.74 & 16.95 & \textbf{36.23} \\ \midrule
\multirow{3}{*}{WSD}    & Ascend  & 54.67 & 24.15 & 77.30 & 22.97 & 7.97 & 32.94 & 20.40 & 66.54 & 32.44 & 52.09 & 19.57 & 16.90 & 35.66 \\
                        & Descend & 60.31 & 24.66 & 80.60 & 26.58 & 6.18 & 32.82 & 20.60 & 66.27 & 30.05 & 52.33 & 20.15 & 17.03 & 36.47 \\
                        & Uniform & 59.39 & 25.68 & 80.80 & 24.63 & 6.46 & 32.78 & 24.40 & 66.92 & 29.19 & 51.78 & 19.57 & 17.65 & \textbf{36.60} \\ \midrule
\multirow{3}{*}{Cosine} & Ascend  & 53.24 & 23.29 & 77.60 & 25.30 & 7.72 & 32.36 & 19.80 & 65.83 & 29.57 & 51.46 & 19.90 & 17.21 & 35.27 \\
                        & Descend & 58.84 & 25.00 & 81.80 & 24.74 & 7.87 & 30.68 & 19.00 & 65.29 & 30.14 & 51.20 & 19.82 & 16.64 & 35.92 \\
                        & Uniform & 59.22 & 25.77 & 79.40 & 24.19 & 5.38 & 32.59 & 22.00 & 66.21 & 29.47 & 51.38 & 20.64 & 17.52 & \textbf{36.15} \\ \bottomrule
\end{tabular}
\end{subtable}
\end{table*}

\begin{table*}[t]
\centering
\footnotesize
\caption{Data curriculum results of Qwen3-0.6B with (a) Muon and (b) AdamW. Within each LR scheduler setting, the pattern $\texttt{Uniform} > \texttt{Descend} > \texttt{Ascend}$ holds consistently. Abbreviations follow Table~\ref{tab:llama_curriculum}.}
\label{tab:qwen_curriculum}
\setlength{\tabcolsep}{1pt}

\begin{subtable}{\textwidth}
\centering
\caption{Muon}
\label{tab:qwen_curriculum_muon}
\begin{tabular}{@{}llccccccccccccc@{}}
\toprule
\textbf{LR Scheduler} & \textbf{Order} & \textbf{AE} & \textbf{AC} & \textbf{SciQ} & \textbf{MM} & \textbf{MM-P} & \textbf{HS} & \textbf{OBQA} & \textbf{PIQA} & \textbf{RACE} & \textbf{WG} & \textbf{CSQA} & \textbf{AGI} & \textbf{Avg} \\ \midrule
\multirow{3}{*}{Const}  & Ascend  & 56.57 & 23.12 & 81.10 & 23.58 & 6.92 & 34.52 & 21.60 & 68.01 & 31.20 & 53.35 & 20.39 & 17.29 & 36.47 \\
                        & Descend & 60.73 & 26.02 & 83.50 & 25.60 & 6.48 & 34.00 & 21.80 & 65.12 & 30.81 & 51.93 & 19.50 & 16.46 & 36.83 \\
                        & Uniform & 62.00 & 26.71 & 82.20 & 25.17 & 5.24 & 34.70 & 22.80 & 66.97 & 30.62 & 52.09 & 19.98 & 17.29 & \textbf{37.15} \\ \midrule
\multirow{3}{*}{WS}     & Ascend  & 55.93 & 23.89 & 77.60 & 23.23 & 5.88 & 34.84 & 23.20 & 67.68 & 30.72 & 53.35 & 19.57 & 17.45 & 36.11 \\
                        & Descend & 61.74 & 26.54 & 81.60 & 23.61 & 5.31 & 33.38 & 22.00 & 66.05 & 31.20 & 52.01 & 19.74 & 17.58 & 36.73 \\
                        & Uniform & 59.09 & 27.39 & 82.60 & 25.24 & 5.47 & 34.48 & 23.20 & 66.54 & 32.15 & 52.72 & 20.07 & 17.11 & \textbf{37.17} \\ \midrule
\multirow{3}{*}{WSD}    & Ascend  & 56.36 & 24.32 & 80.90 & 23.52 & 5.98 & 35.11 & 23.20 & 68.44 & 30.91 & 53.20 & 20.23 & 17.60 & 36.65 \\
                        & Descend & 61.78 & 26.62 & 83.10 & 24.63 & 5.82 & 34.47 & 22.40 & 66.49 & 30.53 & 51.93 & 19.57 & 17.55 & 37.07 \\
                        & Uniform & 61.74 & 26.45 & 82.00 & 24.73 & 5.22 & 34.97 & 23.60 & 67.90 & 31.96 & 51.22 & 20.31 & 17.55 & \textbf{37.30} \\ \midrule
\multirow{3}{*}{Cosine} & Ascend  & 54.38 & 22.61 & 79.60 & 24.31 & 6.56 & 34.37 & 20.60 & 67.36 & 30.72 & 51.54 & 19.16 & 17.68 & 35.74 \\
                        & Descend & 61.83 & 25.94 & 83.00 & 25.00 & 6.48 & 33.44 & 22.40 & 66.00 & 30.81 & 49.49 & 20.97 & 17.37 & 36.89 \\
                        & Uniform & 62.42 & 26.45 & 81.50 & 23.59 & 6.77 & 34.51 & 21.40 & 67.74 & 30.62 & 51.70 & 19.90 & 17.50 & \textbf{37.01} \\ \bottomrule
\end{tabular}
\end{subtable}

\vspace{0.8em}

\begin{subtable}{\textwidth}
\centering
\caption{AdamW}
\label{tab:qwen_curriculum_adamw}
\begin{tabular}{@{}llccccccccccccc@{}}
\toprule
\textbf{LR Scheduler} & \textbf{Order} & \textbf{AE} & \textbf{AC} & \textbf{SciQ} & \textbf{MM} & \textbf{MM-P} & \textbf{HS} & \textbf{OBQA} & \textbf{PIQA} & \textbf{RACE} & \textbf{WG} & \textbf{CSQA} & \textbf{AGI} & \textbf{Avg} \\ \midrule
\multirow{3}{*}{Const}  & Ascend  & 50.42 & 21.42 & 71.70 & 22.93 & 6.60 & 30.01 & 17.40 & 64.58 & 28.42 & 50.36 & 19.57 & 17.11 & 33.38 \\
                        & Descend & 53.91 & 22.70 & 76.20 & 22.90 & 6.07 & 29.90 & 18.60 & 63.00 & 29.09 & 49.01 & 19.66 & 16.87 & 33.99 \\
                        & Uniform & 52.74 & 22.44 & 75.10 & 24.18 & 5.41 & 29.88 & 20.00 & 63.92 & 30.43 & 49.80 & 19.66 & 17.42 & \textbf{34.25} \\ \midrule
\multirow{3}{*}{WS}     & Ascend  & 50.51 & 21.42 & 76.10 & 23.04 & 6.77 & 31.76 & 19.60 & 64.80 & 28.80 & 49.96 & 19.82 & 17.52 & 34.18 \\
                        & Descend & 55.85 & 23.04 & 78.90 & 23.49 & 4.65 & 31.74 & 20.60 & 65.02 & 30.14 & 49.64 & 19.49 & 17.37 & 34.99 \\
                        & Uniform & 57.91 & 24.40 & 80.40 & 24.63 & 7.14 & 31.77 & 20.00 & 65.77 & 29.95 & 52.48 & 19.74 & 17.47 & \textbf{35.97} \\ \midrule
\multirow{3}{*}{WSD}    & Ascend  & 52.44 & 23.72 & 75.90 & 23.07 & 7.00 & 32.43 & 20.20 & 66.10 & 29.38 & 51.78 & 19.57 & 17.94 & 34.96 \\
                        & Descend & 57.83 & 24.83 & 80.80 & 23.66 & 5.69 & 31.78 & 21.80 & 64.91 & 29.67 & 49.57 & 19.57 & 17.42 & 35.63 \\
                        & Uniform & 57.83 & 24.91 & 80.40 & 25.94 & 6.80 & 32.45 & 21.80 & 66.49 & 30.14 & 49.72 & 20.31 & 16.95 & \textbf{36.15} \\ \midrule
\multirow{3}{*}{Cosine} & Ascend  & 52.40 & 19.97 & 75.70 & 23.11 & 8.30 & 31.03 & 20.00 & 65.89 & 30.33 & 50.43 & 19.82 & 17.50 & 34.54 \\
                        & Descend & 56.82 & 22.78 & 78.60 & 24.21 & 6.62 & 30.55 & 21.20 & 63.98 & 28.90 & 51.46 & 19.41 & 17.26 & 35.15 \\
                        & Uniform & 55.98 & 23.98 & 79.30 & 23.74 & 7.82 & 31.31 & 19.60 & 64.31 & 30.91 & 50.20 & 19.00 & 17.45 & \textbf{35.30} \\ \bottomrule
\end{tabular}
\end{subtable}
\end{table*}

\begin{table}[t]
\centering
\footnotesize
\caption{Data curriculum results of Llama3.1-760M. The pattern $\texttt{Uniform} > \texttt{Descend} > \texttt{Ascend}$ holds consistently except under Cosine Scheduler. }
\setlength{\tabcolsep}{1.pt}
\begin{tabular}{@{}llccccccccccccc@{}}
\toprule
\textbf{\begin{tabular}[c]{@{}l@{}}LR\\ Scheduler\end{tabular}}
& \textbf{Order}
& \textbf{AE}
& \textbf{AC}
& \textbf{SciQ}
& \textbf{MM}
& \textbf{MM-P}
& \textbf{HS}
& \textbf{OBQA}
& \textbf{PIQA}
& \textbf{RACE}
& \textbf{WG}
& \textbf{CSQA}
& \textbf{AGI}
& \textbf{Avg} \\ \midrule

\multirow{3}{*}{Cosine}
& Ascend  & 43.01 & 21.25 & 73.90 & 22.97 & 7.84 & 29.60 & 16.60 & 60.88 & 30.53 & 51.62 & 19.57 & 17.11 & 32.91 \\
& Descend & 49.12 & 19.71 & 74.30 & 23.16 & 7.06 & 30.05 & 16.40 & 60.83 & 28.71 & 50.99 & 20.07 & 16.93 & \textbf{33.11} \\
& Uniform & 48.06 & 20.22 & 75.80 & 23.90 & 6.37 & 29.65 & 16.80 & 60.28 & 29.00 & 49.96 & 19.57 & 17.24 & 33.07 \\

\multirow{3}{*}{WS}
& Ascend  & 48.57 & 19.71 & 76.40 & 24.08 & 6.30 & 30.24 & 16.40 & 60.77 & 29.86 & 49.64 & 19.82 & 17.58 & 33.28 \\
& Descend & 47.10 & 20.82 & 72.80 & 24.36 & 6.42 & 30.32 & 17.80 & 61.10 & 29.86 & 51.62 & 19.98 & 17.29 & 33.29 \\
& Uniform & 48.40 & 22.27 & 77.20 & 23.02 & 7.43 & 30.35 & 18.80 & 61.59 & 31.10 & 49.96 & 20.64 & 16.85 & \textbf{33.97} \\

\multirow{3}{*}{WSD}
& Ascend  & 49.58 & 21.76 & 74.50 & 23.17 & 7.01 & 30.63 & 16.80 & 61.59 & 30.62 & 50.04 & 19.74 & 17.11 & 33.55 \\
& Descend & 50.67 & 21.42 & 76.00 & 24.85 & 5.09 & 30.28 & 17.20 & 61.21 & 29.57 & 50.36 & 19.82 & 17.58 & 33.67 \\
& Uniform & 51.26 & 23.38 & 78.50 & 24.04 & 6.77 & 30.26 & 18.80 & 62.79 & 29.57 & 50.43 & 20.88 & 17.24 & \textbf{34.49} \\

\multirow{3}{*}{Const}
& Ascend  & 45.83 & 18.69 & 71.30 & 23.40 & 7.66 & 28.66 & 15.80 & 59.79 & 27.37 & 50.83 & 19.66 & 16.85 & 32.15 \\
& Descend & 44.82 & 19.71 & 71.00 & 22.95 & 7.75 & 28.36 & 17.40 & 60.07 & 29.47 & 50.75 & 19.57 & 17.45 & 32.44 \\
& Uniform & 45.03 & 19.45 & 72.80 & 23.22 & 8.07 & 28.62 & 15.80 & 60.50 & 28.80 & 51.22 & 19.57 & 17.11 & \textbf{32.52} \\

\bottomrule
\end{tabular}
\label{tab:llama_slimpajama_curriculum}
\end{table}

\paragraph{Results.}

Full results are provided in Tabs.~\ref{tab:llama_curriculum} and ~\ref{tab:qwen_curriculum}.
Additionally, we pre-train Llama3.1-760M on another training pool, \texttt{Slimpajama}, using metric \textit{dsir\_books} under Muon, and similar performance trends are observed in Tab.~\ref{tab:llama_slimpajama_curriculum}.

\subsection{Model Geometry}
\label{app:model_geometry}


\begin{figure}[t]
    \centering
    \includegraphics[width=1.0\linewidth]{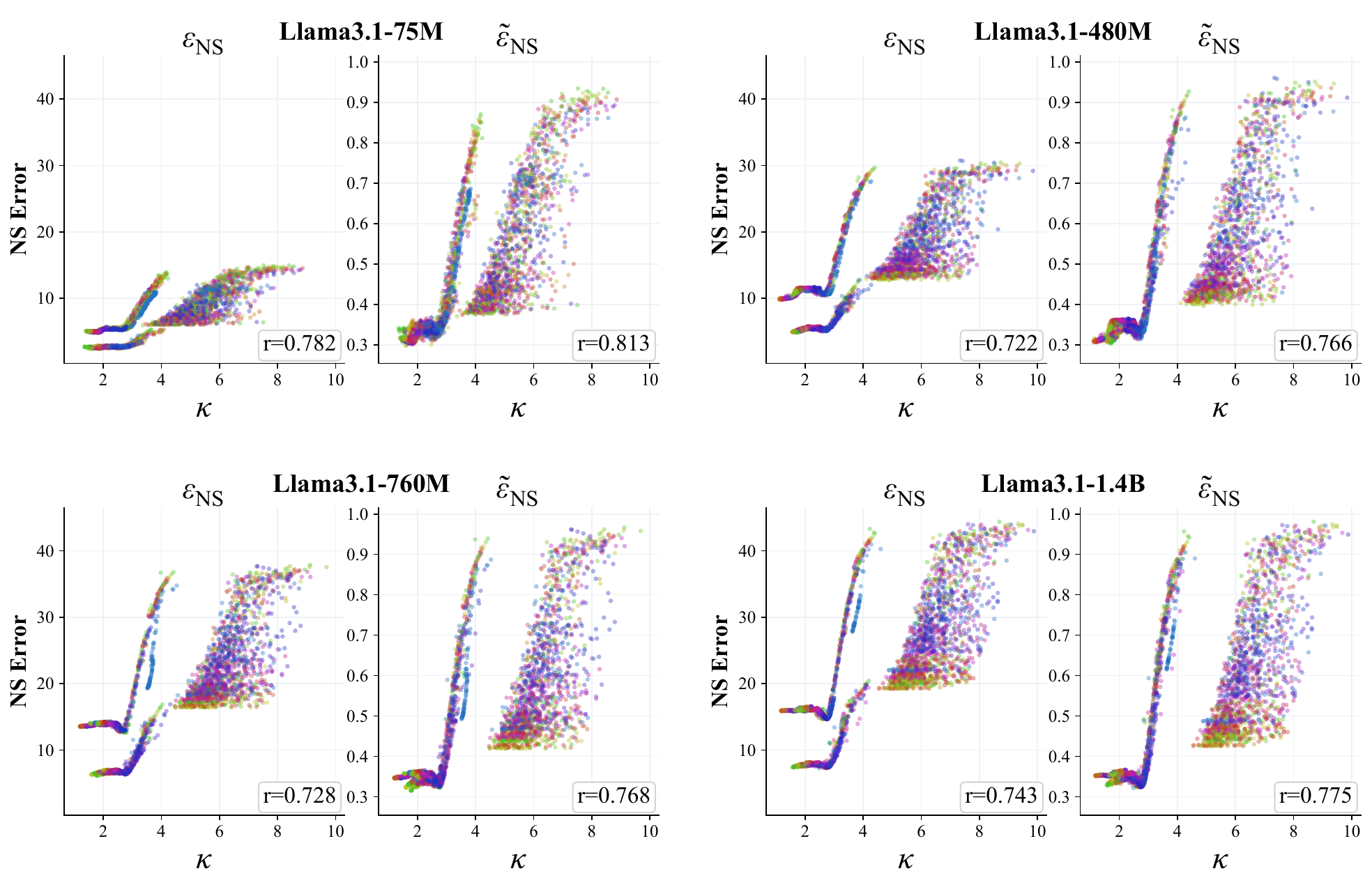}
    \caption{Condition number vs NS residual results for Llama3.1-family models. Colors indicate layer depth, so each point reflects the geometry of one parameter matrix at one training step, grouped visually by the layer it belongs to. $\kappa$ is displayed in $\log_{10}$ scale.
    Across all scales, as indicated from Pearson Correlation coefficients (from 0.722 to 0.813), worse conditioning of the Muon update source is consistently associated with larger orthogonalization error, indicating a stable coupling between matrix geometry and NS convergence quality.}
    \label{fig:app-llama-corr}
\end{figure}

\begin{figure}
    \centering
    \includegraphics[width=1.0\linewidth]{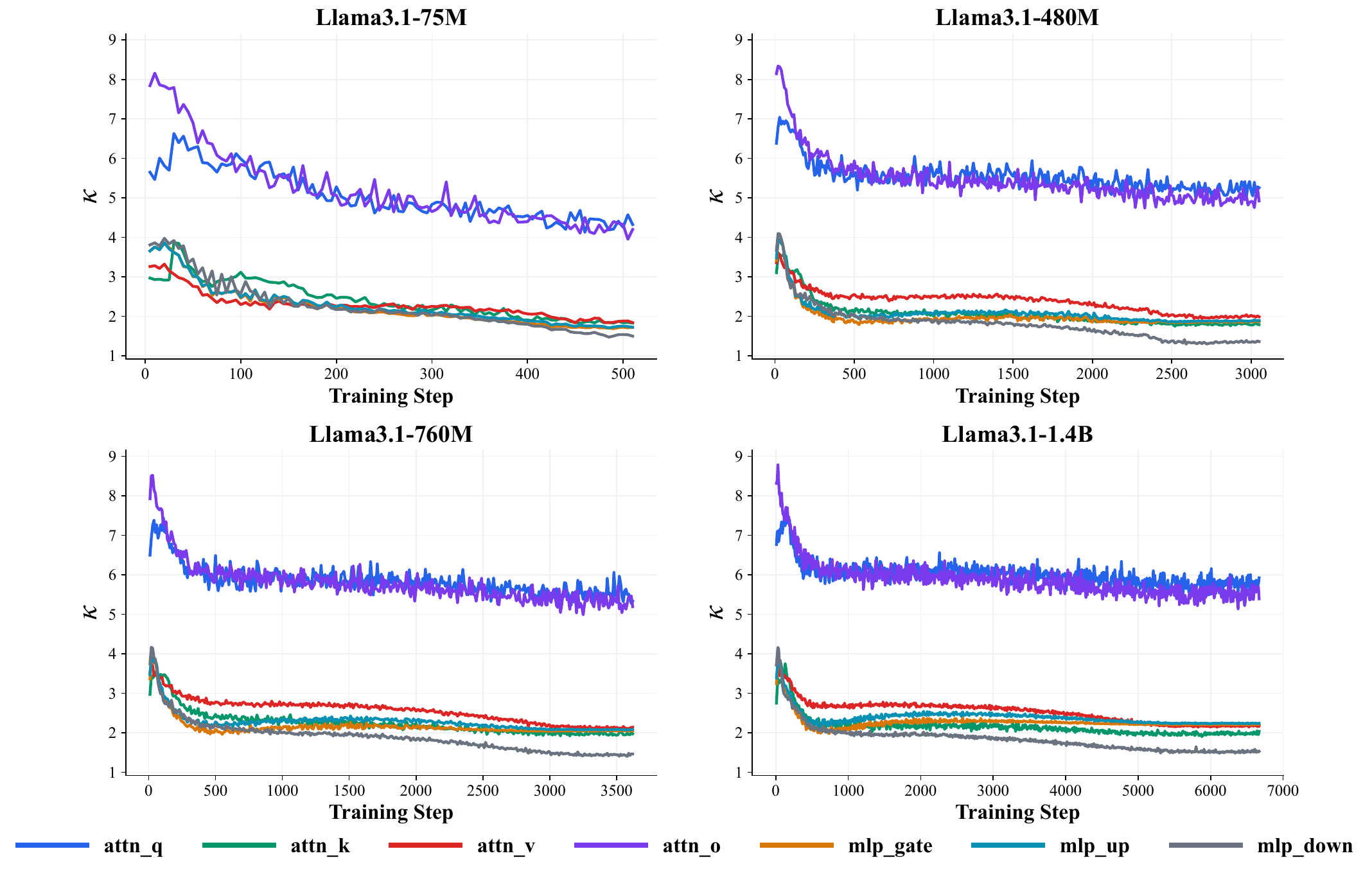}
    \caption{Evolution of median condition number by operator type across training for Llama3.1-family models. $\kappa$ is displayed in $\log_{10}$ scale. The trajectories reveal strong operator-level heterogeneity: geometry difficulty is not uniformly distributed across matrix types, a small subset of operators often dominates the hardest regimes while \texttt{mlp\_down} often shows easiest regimes.}
    \label{fig:app-llama-type_kappa_evolution}
\end{figure}

\begin{figure}
    \centering
    \includegraphics[width=1.0\linewidth]{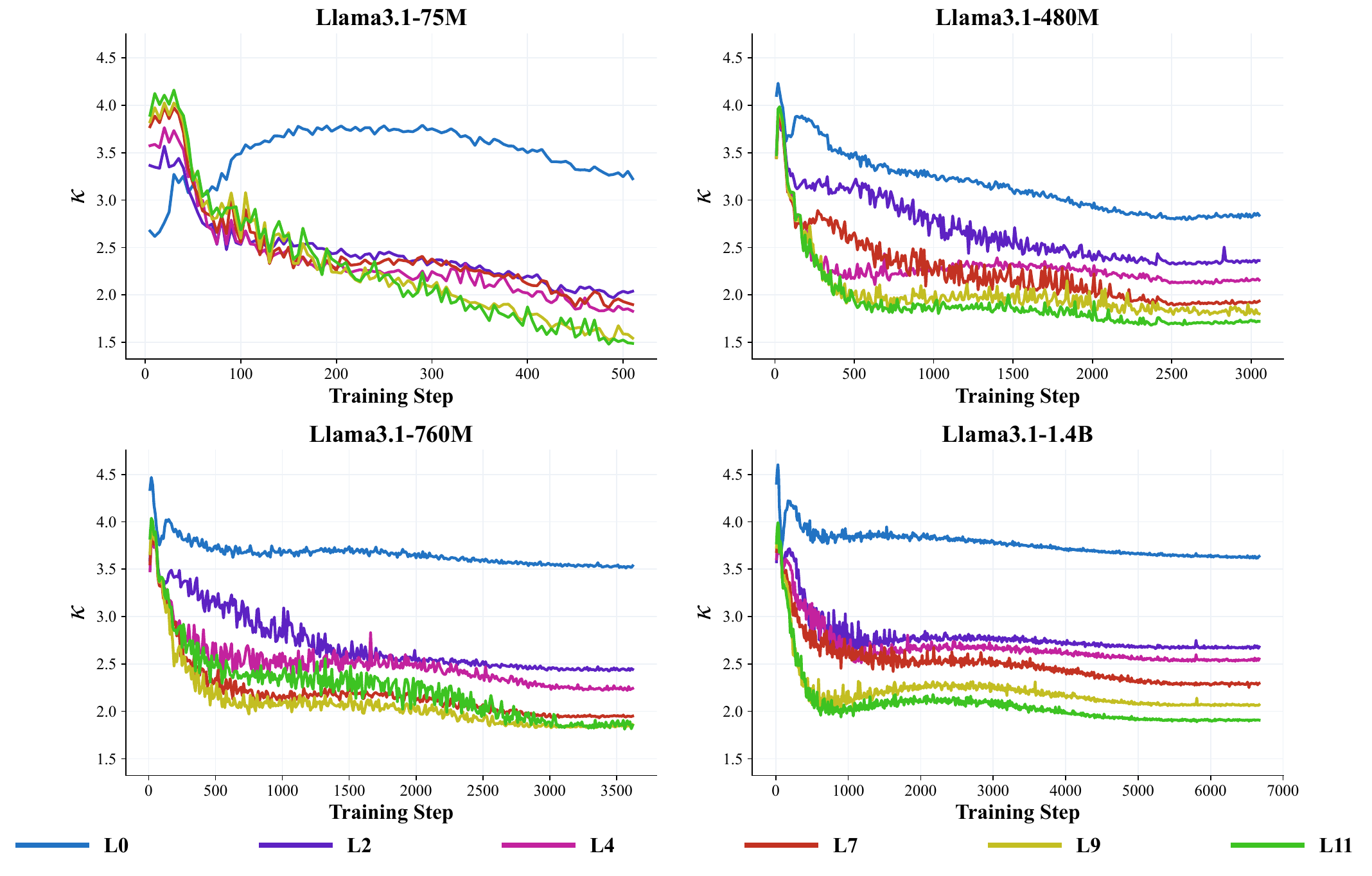}
    \caption{Layer-wise profiles of median condition number at representative training snapshots for Llama3.1-family models. $\kappa$ is displayed in $\log_{10}$ scale. The depth dependence of conditioning changes substantially over training, showing that geometry hotspots are dynamically reorganized rather than fixed to a static set of layers.}
    \label{fig:app-llama-layer_evolution}
\end{figure}

\begin{figure}
    \centering
    \includegraphics[width=1.0\linewidth]{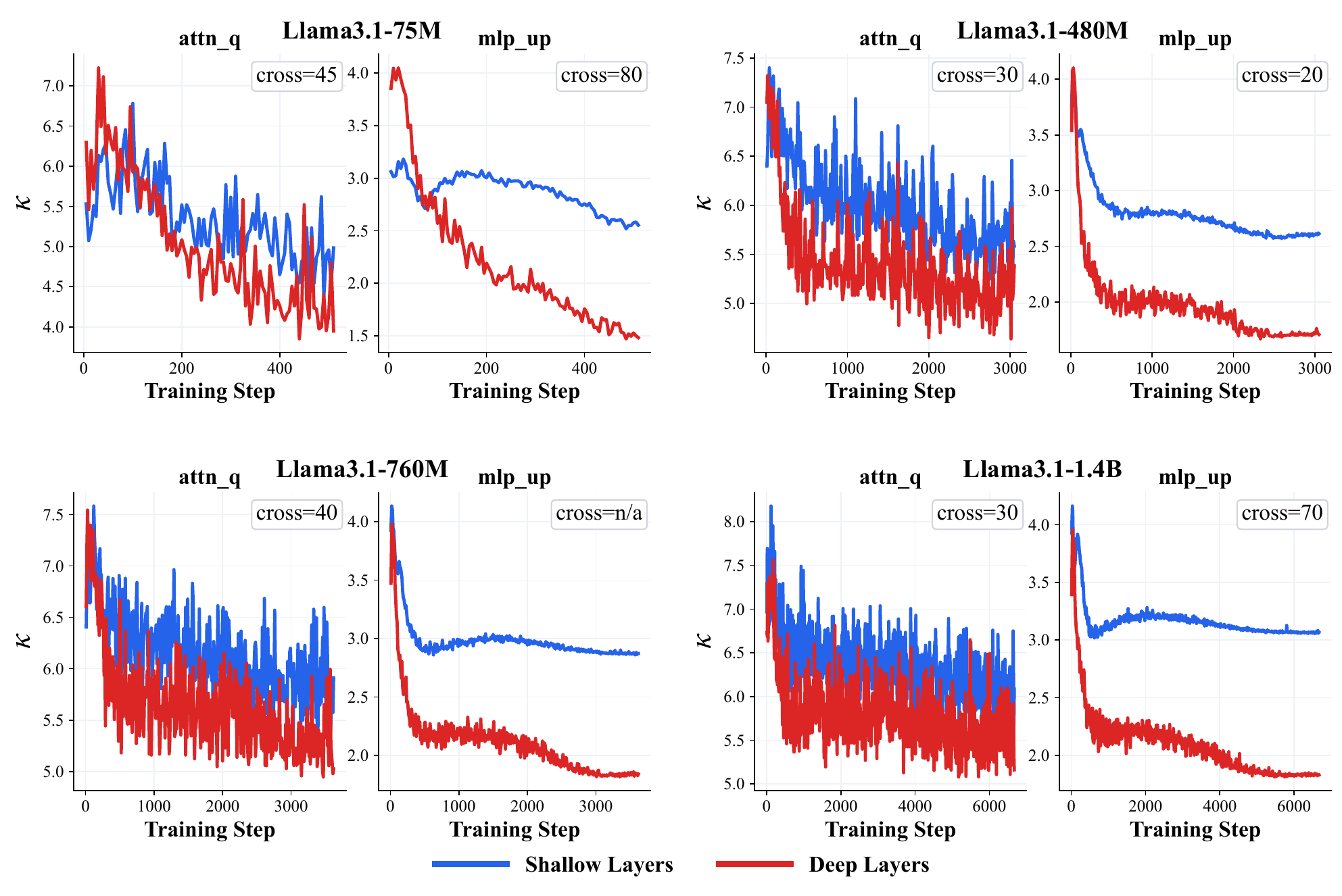}
    \caption{Shallow vs deep layer training trajectories of median condition number for two representative operator types, \texttt{attn\_q} and \texttt{mlp\_up}, across Llama3.1-family models. 
    For each model, the shallow trajectory is computed from the \textbf{first three} layers, while the deep trajectory is computed from the \textbf{last three} layers. 
    \texttt{cross} marks the first logged training step at which the shallow and deep trajectories intersect or reverse their relative ordering and \texttt{cross=n/a} indicates that no such crossing is observed over the logged training window. 
    Across model scales, shallow and deep layers can exchange their relative ordering during training, indicating that the geometry hierarchy over depth is dynamic rather than fixed.}
    \label{fig:app-llama-cross}
\end{figure}


\begin{figure}
    \centering
    \includegraphics[width=1.0\linewidth]{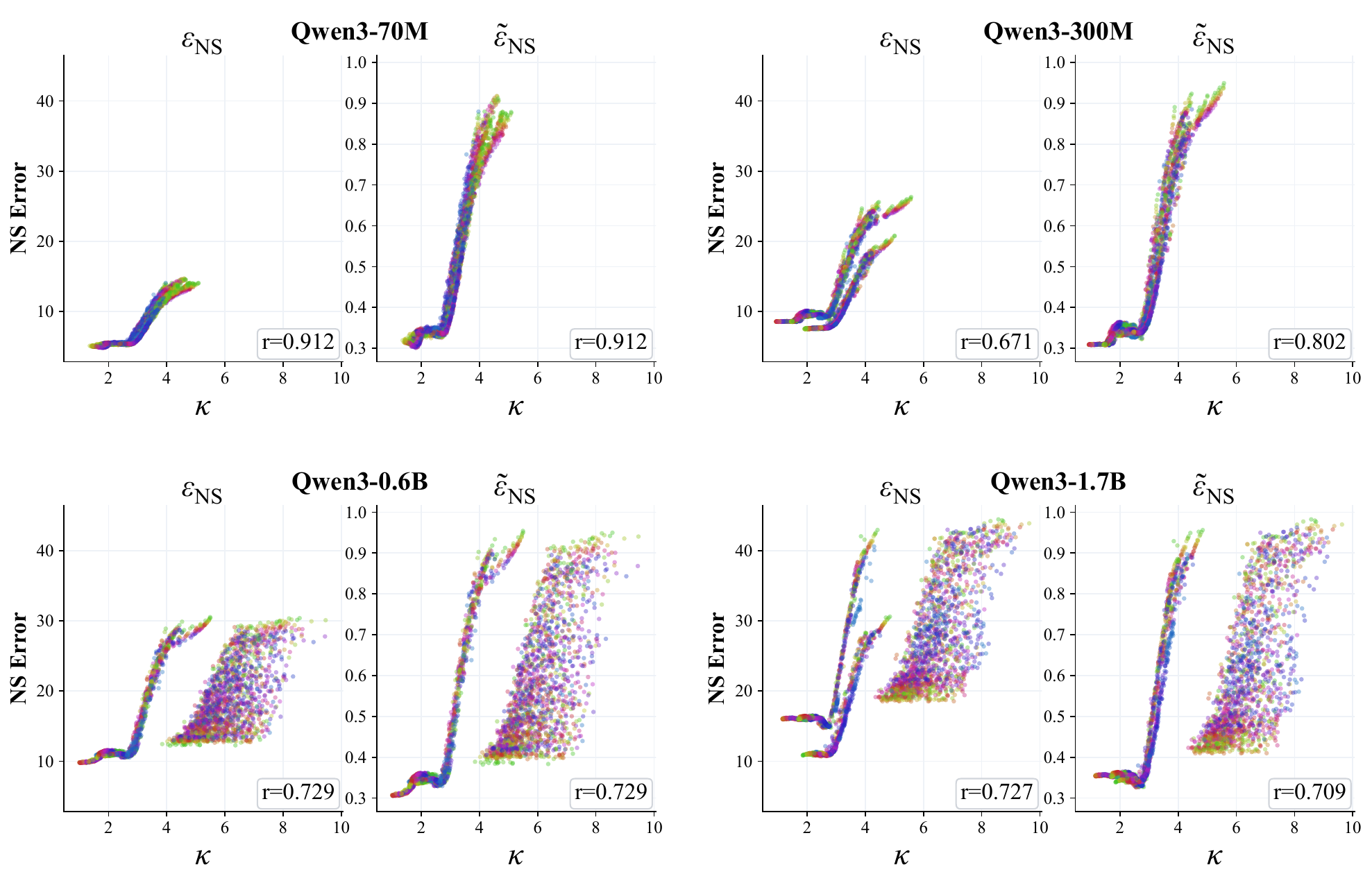}
    \caption{Condition number vs NS residual results for Qwen3-family models. Colors indicate layer depth, so each point reflects the geometry of one parameter matrix at one training step, grouped visually by the layer it belongs to. $\kappa$ is displayed in $\log_{10}$ scale.
    Across all scales, as indicated from Pearson Correlation coefficients (from 0.671 to 0.912), worse conditioning of the Muon update source is consistently associated with larger orthogonalization error, indicating a stable coupling between matrix geometry and NS convergence quality.}
    \label{fig:app-qwen-corr}
\end{figure}

\begin{figure}
    \centering
    \includegraphics[width=1.0\linewidth]{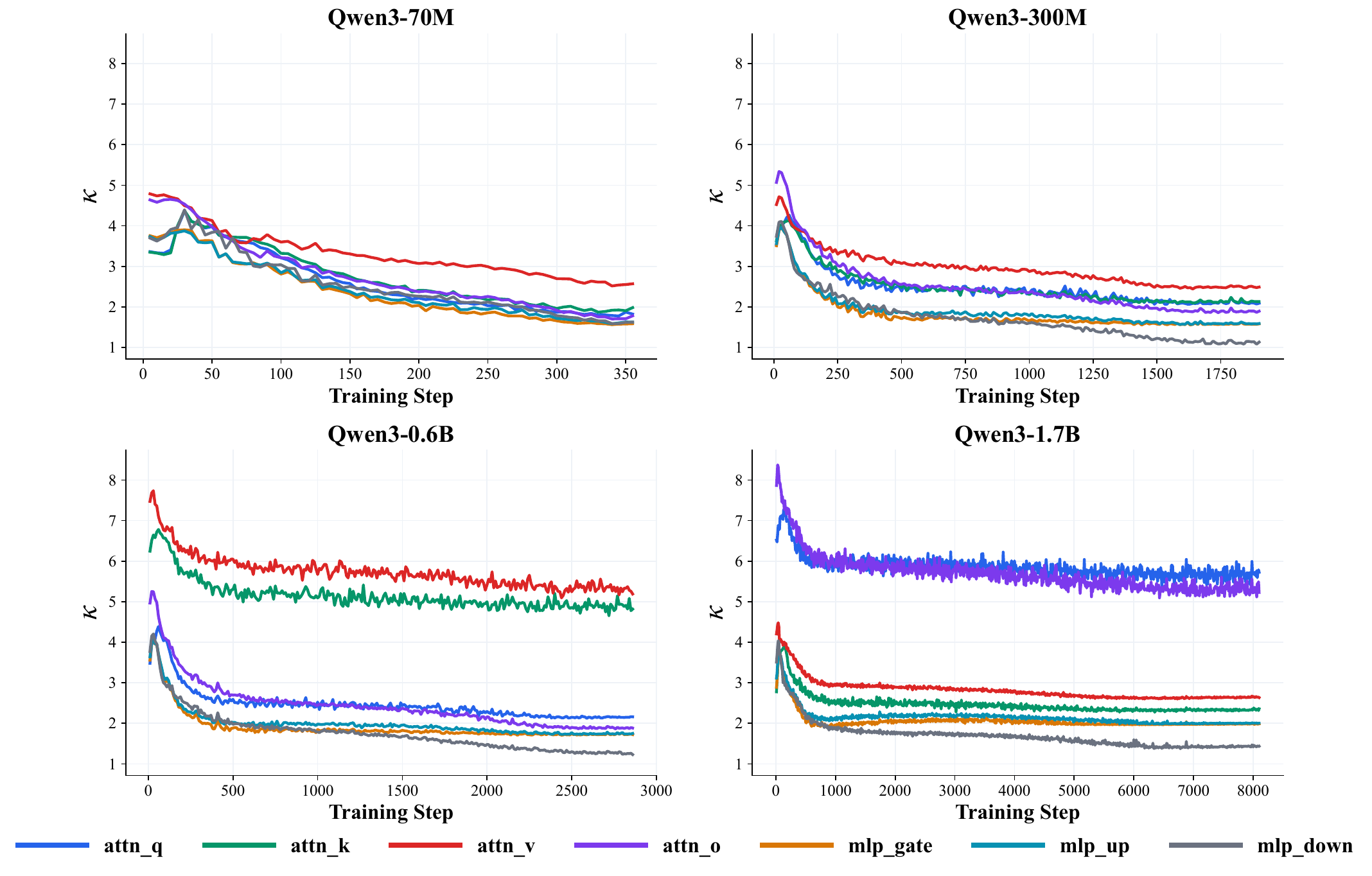}
    \caption{Evolution of median condition number by operator type across training for Qwen3-family models. $\kappa$ is displayed in $\log_{10}$ scale. The trajectories reveal strong operator-level heterogeneity: geometry difficulty is not uniformly distributed across matrix types, a small subset of operators often dominates the hardest regimes while \texttt{mlp\_down} often shows easiest regimes.}
    \label{fig:app-qwen-type_kappa_evolution}
\end{figure}

\begin{figure}
    \centering
    \includegraphics[width=1.0\linewidth]{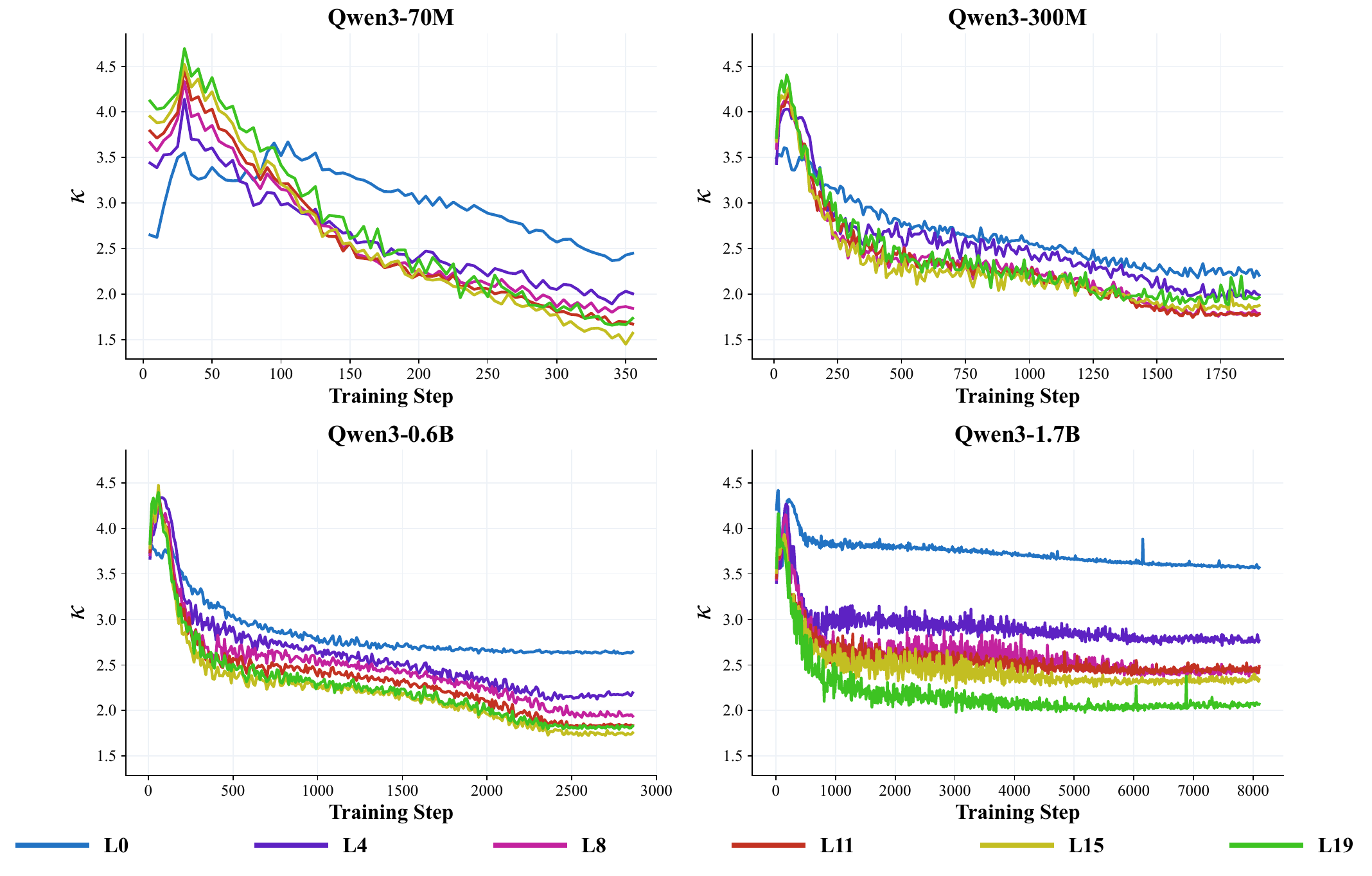}
    \caption{Layer-wise profiles of median condition number at representative training snapshots for Qwen3-family models. $\kappa$ is displayed in $\log_{10}$ scale. The depth dependence of conditioning changes substantially over training, showing that geometry hotspots are dynamically reorganized rather than fixed to a static set of layers.}
    \label{fig:app-qwen-layer_evolution}
\end{figure}

\begin{figure}[tb]
    \centering
    \includegraphics[width=1.0\linewidth]{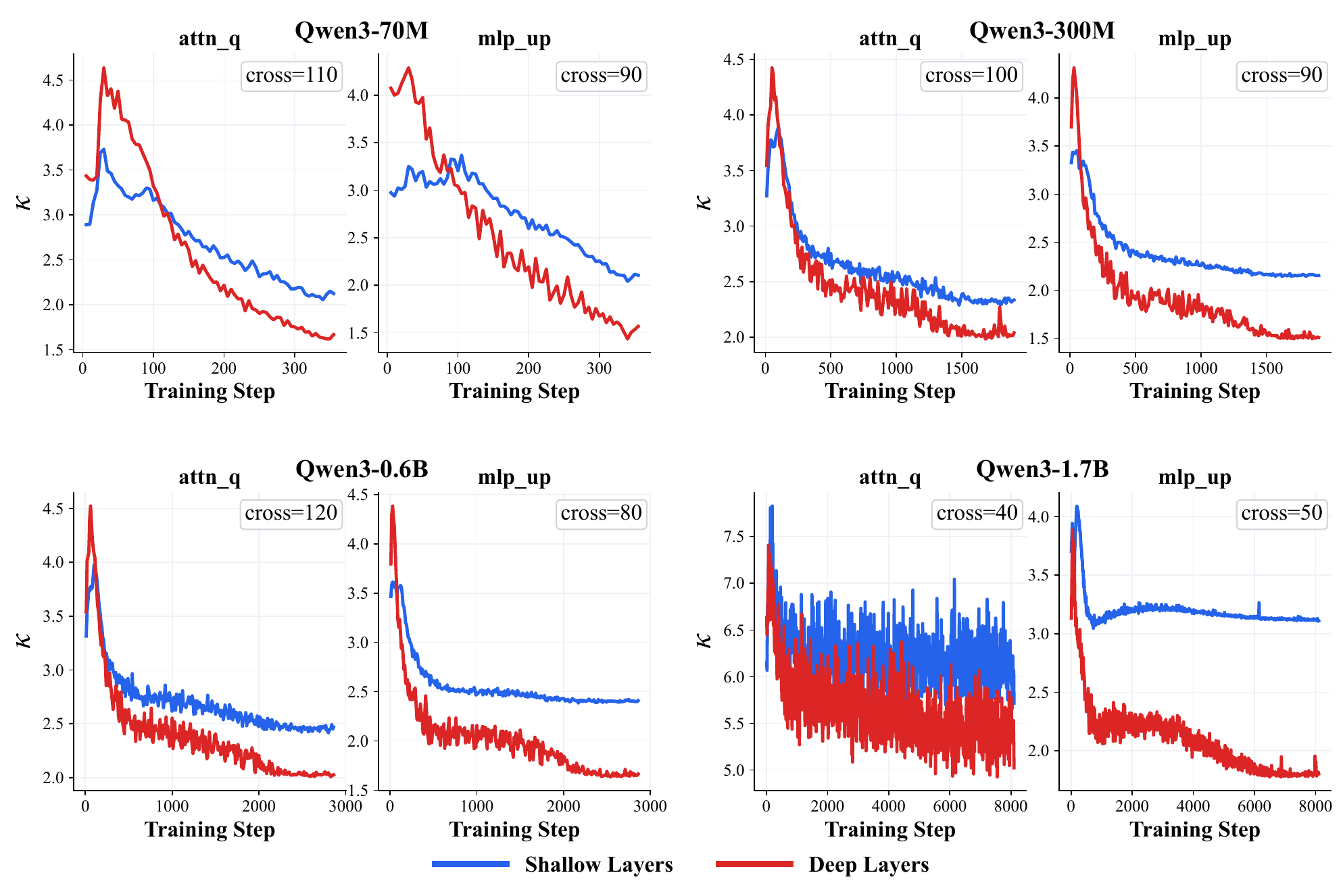}
    \caption{Shallow vs deep layer training trajectories of median condition number for two representative operator types, \texttt{attn\_q} and \texttt{mlp\_up}, across Qwen3-family models. 
    For each model, the shallow trajectory is computed from the \textbf{first three} layers, while the deep trajectory is computed from the \textbf{last three} layers. 
    \texttt{cross} marks the first logged training step at which the shallow and deep trajectories intersect or reverse their relative ordering and \texttt{cross=n/a} indicates that no such crossing is observed over the logged training window. 
    Across model scales, shallow and deep layers can exchange their relative ordering during training, indicating that the geometry hierarchy over depth is dynamic rather than fixed.}
    \label{fig:app-qwen-cross}
\end{figure}

\paragraph{Setup.}

We track training trajectories of all models with geometry observation on all Muon-updated 2D matrices in each transformer block, namely seven types of matrices: \texttt{attn\_q}, \texttt{attn\_k}, \texttt{attn\_v}, \texttt{attn\_o}, \texttt{mlp\_gate}, \texttt{mlp\_up}, and \texttt{mlp\_down}. 

For a given parameter matrix, let $G_t$ denote the stochastic gradient at step $t$, and let $B_t$ be the momentum buffer updated as
\[
B_t = \mu B_{t-1} + G_t,
\]
where $\mu$ is the momentum coefficient. 
Muon applies NS orthogonalization not directly to $G_t$, but to the update-source matrix
\[
M_t =
\begin{cases}
G_t + \mu B_t, & \text{with Nesterov momentum},\\
B_t, & \text{otherwise}.
\end{cases}
\]
We record both the pre-orthogonalization matrix $M_t$ and its post-NS counterpart $M_t'$.
For both $M_t$ and $M_t'$, we compute singular-value-based geometry statistics. 
In particular, let
\[
\sigma_1(M) \ge \sigma_2(M) \ge \cdots
\]
denote the singular values of $M$. 
We define $\sigma_{\max}(M)=\sigma_1(M)$, and we define the smallest nontrivial singular value as
\[
\sigma_{\min}^{(\tau)}(M)
=
\min \{\sigma_i(M) : \sigma_i(M) > \tau\},
\]
where $\tau$ is a very small numerical threshold used to discard singular values that are effectively zero up to numerical precision. 
In other words, $\sigma_{\min}^{(\tau)}$ is the smallest singular value that is still treated as numerically meaningful, rather than a near-zero mode caused by rank deficiency or floating-point noise. 
We then define the condition number as
\[
\kappa(M)=\frac{\sigma_{\max}(M)}{\sigma_{\min}^{(\tau)}(M)}.
\]
In addition, we record the spectral entropy and the effective rank, where the latter is defined as the minimum number of singular directions needed to explain $99\%$ of the total spectral energy.

To measure how well NS orthogonalizes the update, we additionally compute the NS error of $M_t'$ relative to the identity on its smaller dimension:
\[
\epsilon_{\mathrm{NS}} =
\begin{cases}
\lVert M_t' {M_t'}^\top - I \rVert_F, & m \le n,\\
\lVert {M_t'}^\top M_t' - I \rVert_F, & m > n,
\end{cases}
\]
where $M_t' \in \mathbb{R}^{m \times n}$. 
Then we also report the shape-normalized residual
\[
\tilde{\epsilon}_{\mathrm{NS}} = \epsilon_{\mathrm{NS}} / \sqrt{\min(m,n)},
\]
to make residual magnitudes comparable across matrices of different sizes.

Throughout this analysis, we focus on $\kappa(M_t)$ and $\tilde{\epsilon}_{\mathrm{NS}}$, since these two quantities most directly reflect the conditioning of the Muon input and the realized quality of NS orthogonalization. All geometry runs use exact SVD metrics, effective-rank energy $0.99$, and singular-value threshold $\tau=10^{-10}$. 
Geometry tracking is logged every 10 training steps for most models, and every 5 steps for the two smallest models (Llama3.1-75M and Qwen3-70M).

\paragraph{Llama3.1 Results.}

Figs.~\ref{fig:app-llama-corr}, \ref{fig:app-llama-type_kappa_evolution}, \ref{fig:app-llama-layer_evolution}
and \ref{fig:app-llama-cross} summarize the geometry observations for the four Llama3.1 scales. 
As shown in Fig.~\ref{fig:app-llama-corr}, $\kappa(M_t)$ is positively correlated with the NS residual across all scales, indicating that poorly conditioned Muon update sources are \textit{consistently harder} to orthogonalize.
Fig.~\ref{fig:app-llama-type_kappa_evolution} further shows persistent operator-level heterogeneity: \texttt{mlp\_up}, \texttt{mlp\_gate}, and \texttt{attn\_o} are frequently among the most geometry-sensitive matrix types, although no single operator dominates uniformly at every scale.
From the layerwise snapshot profiles in Fig.~\ref{fig:app-llama-layer_evolution}, we observe that the depth profile of conditioning changes substantially over training, with geometry hotspots \textit{shifting} across layers rather than remaining fixed.
Finally, Fig.~\ref{fig:app-llama-cross} shows that shallow and deep layers can \textit{exchange} their relative ordering early in training, indicating that the geometry hierarchy over depth is itself dynamic.
Taken together, these results suggest that a fixed, depth-independent NS budget is unlikely to be uniformly optimal throughout optimization.

\paragraph{Qwen3 Results.} 

Figs.~\ref{fig:app-qwen-corr}, \ref{fig:app-qwen-type_kappa_evolution}, \ref{fig:app-qwen-layer_evolution}, and \ref{fig:app-qwen-cross} show the corresponding geometry observations for the four Qwen3 scales.
As in Llama3.1, Fig.~\ref{fig:app-qwen-corr} shows a clear \textit{positive correlation} between $\kappa(M_t)$ and the NS residual across all scales, confirming that worse conditioning of the Muon update source is consistently associated with poorer orthogonalization quality.
Compared with Llama3.1, however, Qwen3 exhibits stronger scale-to-scale heterogeneity.
As shown in Fig.~\ref{fig:app-qwen-type_kappa_evolution}, the operator-type trajectories are often more concentrated, with \texttt{attn\_v} more frequently emerging as a pronounced geometry-sensitive operator, especially through sharp early-training peaks.
The layerwise snapshot profiles in Fig.~\ref{fig:app-qwen-layer_evolution} further show that geometry hotspots are \textit{more localized} and can evolve more sharply over depth and training stage.
Finally, Fig.~\ref{fig:app-qwen-cross} shows that shallow and deep layers again \textit{exchange} their relative ordering, but in several Qwen3 models these crossings occur later or more abruptly than in Llama3.1.
Taken together, these results reinforce the same overall picture as in Llama3.1 while also suggesting that geometry dynamics remain family-dependent, with Qwen3 showing stronger operator- and scale-specific concentration.

\section{Cost Analysis}
\label{app:cost}

We provide a comprehensive cost analysis from both theoretical and practical perspectives.

\subsection{Theoretical Complexity}

\begin{table}[t]
\centering
\small
\caption{Comparison of persistent optimizer-state memory and per-step update time complexity across optimizers. $P=P_{\mu}+P_{\text{adam}}$ denotes total trainable parameters, where $P_{\mu}$ are 2D matrices updated by the Muon-family rule and $P_{\text{adam}}$ are the rest handled by AdamW. $C_{t,i}=\Theta\bigl(\min(m_{t,i},n_{t,i})^{2}\max(m_{t,i},n_{t,i})\bigr)$ is the cost of one NS iteration on the $i$-th matrix of layer type $t$. $k_t$ is the NS step count assigned to type $t$, and $B=\sum_{t}k_t^{\star}$ is AMO's global step budget.}
\label{tab:theoretical-complexity}
\setlength{\tabcolsep}{6pt}
\resizebox{\linewidth}{!}{
\begin{tabular}{lccc}
\toprule
\textbf{Optimizer} & \textbf{Memory Cost} & \textbf{Per-step Time Cost} & \textbf{NS Budget} \\
\midrule
AdamW       & $2P$ & $\Theta(P)$ & -- \\
AdaMuon     & $2P$ & $\Theta\!\left(P_{\text{adam}}+P_{\mu}+\sum_{t}k_{t}\sum_{i=1}^{N_t}C_{t,i}\right)$ & Uniform, $k_t\!\equiv\!5$ \\
\midrule
Muon/Muon-PE & $P_{\mu}+2P_{\text{adam}}$ & $\Theta\!\left(P_{\text{adam}}+\sum_{t}k_{t}\sum_{i=1}^{N_t}C_{t,i}\right)$ & Uniform, $k_t\!\equiv\!5$ \\
\midrule
\textbf{AMO (ours)} & $P_{\mu}+2P_{\text{adam}}+\mathcal{O}(|\mathcal{T}|)$ & $\Theta\!\left(P_{\text{adam}}+\sum_{t}k_{t}^{\star}\sum_{i=1}^{N_t}C_{t,i}\right)$ & Adaptive, $\sum_{t}k_t^{\star}\!=\!B$ \\
\bottomrule
\end{tabular}
}
\end{table}

We analyze algorithm-level cost in terms of \textbf{persistent optimizer-state memory complexity} and \textbf{per-step update time complexity}, intentionally separating these quantities from backend-dependent effects such as fused kernels or implementation-specific temporary allocations. 
A summarization of theoretical cost is provided in Tab.~\ref{tab:theoretical-complexity}.

Let the total number of trainable parameters be
\[
P = P_{\mu} + P_{\mathrm{adam}},
\]
where $P_{\mu}$ denotes the parameters updated by the Muon-family rule (i.e., the 2D weight matrices selected for orthogonalized updates), and $P_{\mathrm{adam}}$ denotes the remaining parameters handled by AdamW (e.g., embeddings, output heads, and normalization parameters). 
We further partition $P_\mu$ into layer types $t \in \mathcal{T}$. 
For a given layer type $t$, let there be $N_t$ matrices, and let the $i$-th matrix in this type have shape $m_{t,i} \times n_{t,i}$. 
One NS iteration on this matrix has dominant cost
\[
C_{t,i}
=
\Theta\!\bigl(\min(m_{t,i},n_{t,i})^2 \max(m_{t,i},n_{t,i})\bigr).
\]
Therefore, the true cost of a Muon-family update is a weighted sum over all matrices,
\[
\sum_{t \in \mathcal{T}} k_t \sum_{i=1}^{N_t} C_{t,i},
\]
where $k_t$ is the NS step count assigned to layer type $t$. 
This notation is important because, although the optimizer shares a single step count within each layer type, matrices inside the same type need not have identical shapes in general.

From the \textbf{memory} perspective, standard AdamW maintains two momentum states for all parameters, giving
\[
M_{\mathrm{AdamW}} = 2P.
\]
By contrast, Muon, Muon-PE, and our AMO all maintain only a single full-size momentum state on $P_\mu$, while retaining two AdamW momentums on $P_{\mathrm{adam}}$, hence
\[
M_{\mathrm{Muon/Muon\mbox{-}PE/AMO}} = P_{\mu} + 2P_{\mathrm{adam}}.
\]
Here, Muon-PE only replaces the NS coefficients, and therefore introduces no additional optimizer state. 
Likewise, AMO only stores a small amount of control metadata (e.g., observation history, target steps, and per-type override coefficients), whose size is $O(|\mathcal{T}|)$ and is negligible relative to parameter-sized states. 
AdaMuon additionally maintains a full-size second-momentum buffer, so
\[
M_{\mathrm{AdaMuon}} = 2P_{\mu} + 2P_{\mathrm{adam}} = 2P,
\]
which places it in the same memory class as AdamW.

Thus, at the level of persistent algorithmic memory, the methods fall into three groups: AdamW and AdaMuon at $2P$ and Muon/Muon-PE/AMO at $P_{\mu}+2P_{\mathrm{adam}}$.

From the \textbf{time} perspective, AdamW has linear per-step update cost,
\[
T_{\mathrm{AdamW}} = \Theta(P).
\]
For Muon and Muon-PE, the dominant cost comes from applying $k_t$ NS iterations to every 2D matrices in $P_\mu$, giving
\[
T_{\mathrm{Muon/Muon\mbox{-}PE}}
=
\Theta\!\left(
P_{\mathrm{adam}}
+
\sum_{t \in \mathcal{T}} k_t \sum_{i=1}^{N_t} C_{t,i}
\right).
\]
For Muon or Muon-PE, $k_t \equiv 5$ for all $t$.
Muon-PE stays in the same asymptotic class as Muon because its benefit comes from improved fixed coefficients, rather than a change in complexity order. 
AdaMuon adds extra amplitude normalization and second-moment updates on top of Muon, contributing only a linear term:
\[
T_{\mathrm{AdaMuon}}
=
\Theta\!\left(
P_{\mathrm{adam}}
+
\sum_{t \in \mathcal{T}} k_t \sum_{i=1}^{N_t} C_{t,i}
+
P_{\mu}
\right).
\]
Because the additional $P_{\mu}$ term is linear, the NS term remains dominant for large models and nontrivial step counts.

The key distinction of our AMO is that it does not use a uniform or purely static $k_t$. Instead, it allocates per-type step counts under a global budget constraint,
\[
\sum_{t \in \mathcal{T}} k_t^\star = B,
\]
and its locked-phase steady-state complexity becomes
\[
T_{\mathrm{AMO,\,lock}}
=
\Theta\!\left(
P_{\mathrm{adam}}
+
\sum_{t \in \mathcal{T}} k_t^\star \sum_{i=1}^{N_t} C_{t,i}
\right).
\]
This expression makes the core advantage of AMO explicit: fixing the total step budget $B$ does not fix the true computation, because the aggregate weights $\sum_{i=1}^{N_t} C_{t,i}$ can differ substantially across layer types. 
Consequently, AMO can reduce the real update cost without reducing the nominal budget by transferring steps from expensive layer types to cheaper ones. 
In other words, AMO optimizes the weighted objective
\[
\sum_{t \in \mathcal{T}} k_t \sum_{i=1}^{N_t} C_{t,i}
\]
rather than merely the average NS depth. 
This is precisely why AMO should be viewed as a distinct complexity class relative to fixed-budget Muon variants: its gain comes from \textit{budget-aware cross-type reallocation}.

\subsection{Practical Cost}

\begin{table}[t]
\centering
\small
\caption{Practical time and memory cost on pre-training \textbf{Qwen3-1.7B}, measured over 6912 training steps under an identical batch configuration. \textit{Step Time} is end-to-end wall-clock per step; \textit{Optimizer Update Time} is the time spent in the optimizer step; \textit{Throughput} is tokens processed per second; \textit{Peak memory} is the maximum allocated GPU memory during training. Arrows ($\downarrow$/$\uparrow$) indicate whether lower or higher is better. \textbf{Best} and \underline{second-best} values are highlighted.}
\label{tab:practical-cost-detail}
\setlength{\tabcolsep}{8pt}
\begin{tabular}{lcccc}
\toprule
\textbf{Optimizer} & \textbf{Step Time} $\downarrow$ & \textbf{Optimizer Update Time} $\downarrow$ & \textbf{Throughput} $\uparrow$ & \textbf{Peak Memory} $\downarrow$ \\
 & (s) & (ms) & (M token/s) & (GB) \\
\midrule
AdamW               & \textbf{5.96}    & \textbf{4.48}      & \textbf{0.82}    & 63.75 \\
Muon                & \underline{6.14} & 195.81             & \underline{0.80} & \textbf{61.12} \\
AdaMuon             & 6.19             & 241.65             & 0.79             & 69.00 \\
\textbf{AMO (ours)} & 6.17             & \underline{191.47} & \underline{0.80} & \textbf{61.12} \\
\bottomrule
\end{tabular}
\end{table}

We next report implementation-level measurements of time and memory cost.
For an apples-to-apples comparison, we train AdamW, Muon, AdaMuon,
and AMO on \textbf{Qwen3-1.7B} for the same 6912 steps under an identical
batch configuration.
Full results are summarized in Tab.~\ref{tab:practical-cost-detail}.

As shown in Tab.~\ref{tab:practical-cost-detail}, AdamW attains the best
raw step time, optimizer update cost, and throughput, which is expected
given that its update rule is a pure element-wise operation.
The Muon-family methods are roughly $43$--$54\times$ more expensive in
the optimizer update itself (from $\sim\!191$ to $\sim\!242$ ms versus
AdamW's $4.48$ ms), yet the end-to-end step time grows by only
$\approx 3.0\%$--$3.8\%$.
This gap is small because the forward/backward pass continues to dominate
the training loop, consistently accounting for roughly $5.82$ s per step,
while the Muon-family optimizer update stays within $0.19$--$0.24$ s.
In other words, the optimizer is measurably more expensive, but its cost
is largely amortized by the much heavier non-optimizer portion of each
step, making the overhead negligible in practice.

Within the Muon family, the ordering is also informative.
AMO has the lowest optimizer-update cost among the Muon-family methods
($191.47$ ms), slightly below Muon ($195.81$ ms) and substantially below
AdaMuon ($241.65$ ms), while avoiding Muon's uniform $k_t \equiv 5$ rule
and instead using a budget-aware allocation.
In terms of peak memory, Muon and AMO tie at the lowest value
($61.12$ GB), consistent with their shared persistent-state class
$P_\mu + 2P_{\text{adam}}$; AdaMuon peaks higher at $69.00$ GB, even above
AdamW's $63.75$ GB.
Since both AdaMuon and AdamW lie in the same $2P$ persistent-state class
theoretically, this reversal is best explained by transient allocations
along the Muon-style update path rather than by any contradiction with the
state-complexity analysis: the peak allocation reported in
Tab.~\ref{tab:practical-cost} is an implementation-level system quantity
that includes temporary tensors and backend effects such as fused kernels,
not a pure algorithmic-state metric.

In this run the budget ratio is set to $1.0$, so AMO does not reduce the
nominal total NS step budget; it instead redistributes the budget across
layer types and locks to a schedule whose mean remains $5.0$ NS steps
across the seven layer types.
The dominant runtime penalty of AMO does not arise from its locked-phase
update rule, but from its low-frequency observation steps: the twenty
observation steps average $15.63$ s/step, of which $9.47$ s is spent in
the post-step exact-SVD observation probe, whereas non-observation steps
average only about $6.14$ s/step.
Once the schedule is locked, AMO settles at $6.13$ s/step and
$0.80$M tokens/s, essentially matching Muon's tail-region average of
$6.14$ s/step and $0.80$M tokens/s.
Consequently, in the present implementation the theoretical savings from AMO's lower locked-phase weighted orthogonalization cost are largely offset, at the wall-clock level of the full run, by the exact-SVD observation overhead and by the fact that the optimizer update accounts for only a small fraction of total step time.

\section{Experiments}

\subsection{Experimental Settings}

\label{app:subsec_exp_setting}

\paragraph{Models.}

We consider two families of dense decoder-only Transformer language models: Qwen3 and Llama-3.1 \citep{qwen3, llama3}. 
Specifically, we use four Qwen3 variants (\textbf{70M}, \textbf{300M}, \textbf{0.6B}, and \textbf{1.7B}) and four Llama-3.1 variants (\textbf{75M}, \textbf{480M}, \textbf{760M}, and \textbf{1.4B}). 
Following official recommendations for model initialization, after instantiation from configuration, all linear and embedding weights are initialized from $\mathcal{N}(0, 0.02^2)$, linear biases are set to zero, and normalization parameters keep their default values.
Tab.~\ref{tab:model_specs} summarizes their main architectural hyperparameters. 

All models adopt grouped-query attention \citep{gqa}, RMSNorm \citep{rmsnorm}, and SwiGLU-style feed-forward blocks \citep{swiglu}, but the two families differ systematically in several aspects. 
Qwen3 models use a vocabulary size of 151,936, rotary position embeddings with $\theta = 10^6$, and tied input/output embeddings in all scales considered here. 
In contrast, Llama-3.1 models use a vocabulary size of 128,256, the Llama-3 rotary scaling scheme with $\theta = 5\times10^5$, and untied output embeddings. 
Across comparable scales, the Qwen3 variants are generally deeper, whereas the Llama-3.1 variants are typically wider. 
Moreover, the smaller Qwen3 models employ expanded query projections with $d_q > d$, which leads to a noticeably different compute profile even at similar parameter scales.

\begin{table*}[t]
\centering
\small
\caption{Architectural configurations of eight models used in our experiments. All models are dense decoder-only Transformers. $L$ denotes the number of Transformer layers, $d$ the hidden size, $d_{\mathrm{ff}}$ the intermediate size of the feed-forward block, and $d_q / d$ the ratio between the query projection output dimension and the hidden size. We use a training context length of 2048 tokens for all models.}
\setlength{\tabcolsep}{5pt}
\begin{tabular}{lrrrrrrr}
\toprule
\textbf{Model} & \textbf{Total Params} & $L$ & $d$ & $d_{\mathrm{ff}}$ & \textbf{Heads / KV Heads} & $d_q / d$ & \textbf{Context Length} \\
\midrule
Qwen3-70M  & 70,368,512   & 20 & 256  & 1024 & 8 / 4   & 4.0 & 2048 \\
Qwen3-300M & 321,195,776  & 20 & 768  & 3072 & 12 / 4  & 2.0 & 2048 \\
Qwen3-0.6B & 596,049,920  & 28 & 1024 & 3072 & 16 / 8  & 2.0 & 2048 \\
Qwen3-1.7B & 1,720,574,976 & 28 & 2048 & 6144 & 16 / 8  & 1.0 & 2048 \\
\midrule
Llama-3.1-75M  & 74,717,440    & 12 & 256  & 768  & 8 / 2   & 1.0 & 2048 \\
Llama-3.1-480M & 480,805,888   & 16 & 1024 & 3584 & 8 / 2   & 1.0 & 2048 \\
Llama-3.1-760M & 762,091,008   & 12 & 1536 & 5376 & 12 / 3  & 1.0 & 2048 \\
Llama-3.1-1.4B & 1,397,819,392 & 16 & 2048 & 7168 & 16 / 4  & 1.0 & 2048 \\
\bottomrule
\end{tabular}
\label{tab:model_specs}
\end{table*}

\paragraph{Training Datasets.}
All pre-training runs use web documents sampled from the training set of \texttt{Fineweb-edu} \citep{fineweb}. 
We do not use a separate validation split during pre-training. 
During the training process, training data are streamed from disk and shuffled with a buffer size of 10,000 and a random seed of 42.

Tokenization is model-family-specific: Qwen3 models use the Qwen tokenizer, while Llama-3.1 models use the Llama-3.1 tokenizer. 
In both cases, text is tokenized without prepended BOS tokens, terminated with EOS, and packed into fixed-length sequences of 2048 tokens. 
Residual fragments shorter than the target sequence length are discarded. Because the two model families use different tokenizers, token counts are tokenizer-dependent; accordingly, we distinguish between the raw corpus identifier and the effective number of training tokens consumed by each run.

\paragraph{Pre-training.}

\begin{table*}[t]
\centering
\small
\caption{Pre-training hyperparameters of all experiments. \textit{Global Batch} denotes the global batch size, while \textit{Micro Batch} denotes the per-device micro-batch size. \textit{Grad Accum} denotes the number of gradient accumulation steps. \textit{Token budget} follows 1x Chinchilla  optimal ratio of 20 \citep{scaling_law} and is reported in billions ($10^9$) of packed training tokens.}
\setlength{\tabcolsep}{5pt}
\begin{tabular}{lrrrrrr}
\toprule
\textbf{Model} & \textbf{\#GPUs} & \textbf{Global Batch} & \textbf{Micro Batch} & \textbf{Grad Accum} & \textbf{Total Steps} & \textbf{Token Budget} \\
\midrule
\multicolumn{7}{l}{\textbf{AdamW runs}} \\
Qwen3-70M      & 8  & 2048 & 8  & 32  & 335  & 1.41 \\
Qwen3-300M     & 16  & 2048 & 8  & 16  & 1430 & 5.98 \\
Qwen3-0.6B     & 16 & 2048 & 8  & 16  & 2860 & 12.00 \\
Qwen3-1.7B     & 32  & 2400 & 5  & 15 & 6912 & 33.97 \\
Llama-3.1-75M  & 8  & 2048 & 8  & 32  & 360  & 1.51 \\
Llama-3.1-480M & 16 & 1536 & 12 & 8   & 3051 & 9.60 \\
Llama-3.1-760M & 16 & 2048 & 4  & 32  & 3622 & 15.20 \\
Llama-3.1-1.4B & 32 & 2048 & 8  & 8   & 6670 & 27.98 \\
\midrule
\multicolumn{7}{l}{\textbf{Muon runs}} \\
Qwen3-70M      & 6  & 1728 & 12 & 24 & 398  & 1.41 \\
Qwen3-300M     & 16  & 2048 & 8  & 16 & 1430 & 5.98 \\
Qwen3-0.6B     & 16 & 2048 & 8  & 16 & 2860 & 12.00 \\
Qwen3-1.7B     & 32 & 2400 & 5  & 15 & 6912 & 33.97 \\
Llama-3.1-75M  & 6  & 1440 & 10 & 24 & 512  & 1.51 \\
Llama-3.1-480M & 16 & 1536 & 12 & 8  & 3051 & 9.60 \\
Llama-3.1-760M & 16 & 1920 & 12 & 10 & 3866 & 15.20 \\
Llama-3.1-1.4B & 32 & 2048 & 8  & 8  & 6670 & 27.98 \\
\bottomrule
\end{tabular}
\label{tab:pretrain_setup}
\end{table*}

\begin{table}[]
\centering
\small
\caption{Muon/AdamW parameter split configurations in Muon runs. \textit{Muon Ratio} is computed as the number of Muon-updated Parameters divided by total trainable parameters.}
\begin{tabular}{lccc}
\toprule
\textbf{Model} & \textbf{Muon Params (M)} & \textbf{AdamW Params (M)} & \textbf{Muon Ratio (\%)} \\ \midrule
Llama3.1-75M   & 9.04                     & 65.67                     & 12.10                     \\
Llama3.1-480M  & 218.10                   & 262.70                    & 45.36                     \\
Llama3.1-760M  & 368.05                   & 394.04                    & 48.29                     \\
Llama3.1-1.4B  & 872.42                   & 525.40                    & 62.41                     \\ \midrule
Qwen3-70M      & 31.46                    & 38.91                     & 44.70                     \\
Qwen3-300M     & 204.47                   & 116.72                    & 63.66                     \\
Qwen3-0.6B     & 440.40                   & 155.65                    & 73.89                     \\
Qwen3-1.7B     & 1409.29                  & 311.29                    & 81.91                     \\ \bottomrule
\end{tabular}
\label{tab:model_optimizer_config}
\end{table}

All models are trained from scratch using bfloat16 mixed precision and Flash-Attention-2 \citep{flashattention2} on 6$\times$ to 32$\times$ NVIDIA H100 (80GB) GPUs. 
We use a fixed training sequence length of 2048 tokens for all runs, and apply gradient clipping with a maximum norm of 1.0. 
Unless otherwise specified, all runs use random seed 42 and do not employ validation during pre-training. 
Tab.~\ref{tab:pretrain_setup} summarizes the pre-training hyperparameters for each run, including the effective global batch size, the per-device micro-batch size, the number of gradient accumulation steps, the total number of optimization steps, and the corresponding token budget.

Moreover, since Qwen3 and Llama-3.1 use different tokenizers, token budgets should be interpreted as tokenizer-specific token counts rather than a tokenizer-independent measure of raw corpus size. 
Across the runs considered here, the token budget ranges from roughly 1B to 34B tokens, increasing with model scale, following 1x Chinchilla optimal ratio \citep{scaling_law}.

\paragraph{Optimizers.}

We consider two kinds of optimizers: AdamW and Muon-style variants. 
For both optimizers, the base learning rate is set according to the empirical scaling rule from \citep{lr_law}:
\[
\eta_{\mathrm{stable}}(N, D)
= 10^{-4} \cdot 38.4588 \cdot N^{-0.2219} \cdot D^{-0.3509},
\]
where \(N\) and \(D\) denote the model size and training token budget in billions, respectively.

For AdamW, we use \(\beta_1=0.9\), \(\beta_2=0.95\), \(\epsilon=10^{-8}\), and weight decay \(0.1\). 
AdamW runs use a warmup-stable-decay (WSD) schedule: the learning rate is linearly warmed up to \(\eta_{\mathrm{stable}}\), kept constant during the stable phase, and then linearly decayed to \(0.1\,\eta_{\mathrm{stable}}\).

For all Muon-style variants, Muon is applied to 2D matrices in the Transformer blocks, while embedding layers, output heads (e.g., \texttt{lm\_head} or \texttt{output}), and all remaining parameters such as biases and normalization weights are optimized with AdamW.

Across all Muon runs, we use momentum coefficient of \(0.95\), Nesterov momentum, weight decay \(0.1\), and five-step NS iterations with default KJ coefficients \citep{muon}. 
Concretely, Muon first forms a momentum-smoothed update, orthogonalizes it with NS iteration, and then applies a shape-dependent rescaling step size proportional to \(0.2\sqrt{\max(m,n)}\) for an \(m\times n\) parameter matrix, following \cite{dmuon}. 
Muon runs use a cosine-decay schedule with linear warmup and a final learning-rate floor of \(0.1\,\eta_{\mathrm{stable}}\).
Detailed configurations of Muon runs are provided in Tab.~\ref{tab:model_optimizer_config}.

The details of other baseline optimizers are as follows: 

\begin{itemize}
    \item \textbf{Muon-You}\footnote{You's original implementation is provided at this page: \url{https://gist.github.com/YouJiacheng/393c90cbdc23b09d5688815ba382288b}}. This variant keeps the same parameter partition and outer hyperparameters as the default Muon setup, but replaces the NS coefficient sequence with the five-step \texttt{You} coefficients. 
    No additional input rescaling or preconditioning is applied before the NS iteration.

    \item \textbf{Muon-PE}~\citep{pe}. This variant uses the five-step Polar Express coefficient table. In the current configurations, we additionally apply a small normalization safety factor of 1.01 before the NS iteration. All other settings follow the default Muon setup above.

    \item \textbf{Muon-CANS}~\citep{cans}. This variant keeps the same Muon/AdamW parameter split and cosine-decay schedule, but replaces the NS coefficients with the CANS table. 
    As in the PE baseline, we use a mild pre-NS normalization scale of $1.01$ in the standard configurations.

    \item \textbf{Muon-Turbo}~\citep{turbo_muon}. This variant uses the Turbo coefficient table together with AOL-style preconditioning. The remaining Muon and AdamW settings are unchanged.

    \item \textbf{AdaMuon}~\citep{adamuon}. AdaMuon uses the same Muon/AdamW parameter split as above. For Muon-managed matrices, it maintains a momentum buffer with coefficient \(0.95\), takes the elementwise sign of the momentum-smoothed update, orthogonalizes it with a five-step KJ NS iteration, tracks an elementwise second-moment buffer with the same decay, divides by its square root plus \(10^{-8}\), and finally rescales the update to Frobenius norm \(0.2\sqrt{mn}\) for an \(m \times n\) matrix. Non-Muon parameters use AdamW with \((\beta_1,\beta_2)=(0.9,0.95)\), \(\epsilon=10^{-8}\), and weight decay \(0.1\). AdaMuon runs use the same cosine-decay schedule with linear warmup and a final floor of \(0.1\,\eta_{\mathrm{stable}}\).

    \item \textbf{ROOT}~\citep{root}. ROOT replaces the global NS coefficient table with matrix shape-specific coefficients. 
    Specifically, each Muon-managed matrix shape has a single fitted quintic triplet \((a,b,c)\), obtained by minimizing \(\mathrm{mean}((g^T(\sigma)-1)^2)\) for \(g(x)=ax+bx^3+cx^5\) on a mixed geometry-proxy and random-spectrum calibration set. 
    The fitted triplet is repeated across all five NS iterations, with no additional pre-NS safety rescaling or AOL-style preconditioning.

    \item \textbf{PRISM}~\citep{prism}. We use the PRISM-5 Muon variant. After the momentum update and a pre-NS normalization scale of \(1.01\), each Muon-managed matrix is orthogonalized using three degree-five iterations with the fixed upper-bound parameter \(\alpha=29/20\), equivalently \((a,b,c)=(2.95,-3.4,1.45)\) at every iteration. The parameter partition, AdamW fallback, weight decay, and learning-rate schedule remain identical to the default Muon setup.

    \item \textbf{Drop-Muon}~\citep{drop_muon}. We partition the model in forward order into the input embedding, individual decoder layers, and a final block containing the output normalization and language-model head. At every optimizer step, a cutoff is sampled uniformly over these blocks and only the contiguous suffix from the cutoff to the output is trained; the preceding blocks are frozen before the forward pass, thereby omitting their backward computation and optimizer updates. The cutoff is resampled every step using seed \(42\), and this procedure is active throughout training. 

    \item \textbf{NorMuon}~\citep{normuon}. NorMuon retains the same Muon/AdamW parameter partition. For each Muon-managed matrix, it forms an EMA momentum with \(\beta_1=0.95\), applies Nesterov look-ahead, and orthogonalizes the resulting update using five KJ NS iterations. It then maintains a row-wise second-moment buffer with \(\beta_2=0.95\), divides each row by \(\sqrt{v+10^{-10}}\), and rescales the result to restore the Frobenius norm of the orthogonalized update. The final update uses the standard Muon learning-rate scaling \(0.2\eta\sqrt{\max(m,n)}\); non-Muon parameters use AdamW with \((\beta_1,\beta_2)=(0.9,0.95)\), \(\epsilon=10^{-8}\), and weight decay \(0.1\).

\end{itemize}

\paragraph{Evaluation.}

\begin{table}[t]
\centering
\small
\caption{Evaluation benchmarks and task settings in our experiments. By default, the framework evaluates on the \texttt{test} split when available, and otherwise falls back to the \texttt{validation} split.}
\setlength{\tabcolsep}{5pt}
\begin{tabular}{llccc}
\toprule
\textbf{Benchmark} & \textbf{Framework Task} & \textbf{Evaluation Split} & \textbf{Few-shot} & \textbf{Evaluation Metric} \\
\midrule
ARC-Easy         & \texttt{arc\_easy}        & Test       & 0 & Accuracy \\
ARC-Challenge    & \texttt{arc\_challenge}   & Test       & 0 & Accuracy \\
SciQ             & \texttt{sciq}             & Test       & 0 & Accuracy \\
MMLU             & \texttt{mmlu}             & Test       & 0 & Accuracy \\
MMLU-Pro         & \texttt{mmlu\_pro}        & Test       & 5 & Exact Match \\
HellaSwag        & \texttt{hellaswag}        & Validation & 0 & Accuracy \\
OpenBookQA       & \texttt{openbookqa}       & Test       & 0 & Accuracy \\
PIQA             & \texttt{piqa}             & Validation & 0 & Accuracy \\
RACE (high)      & \texttt{race}             & Test       & 0 & Accuracy \\
WinoGrande (XL)  & \texttt{winogrande}       & Validation & 0 & Accuracy \\
CommonsenseQA    & \texttt{commonsense\_qa}  & Validation & 0 & Accuracy \\
AGIEval-en       & \texttt{agieval\_en}      & Test       & 0 & Accuracy \\
\bottomrule
\end{tabular}
\label{tab:eval_tasks}
\end{table}

We evaluate all checkpoints using \texttt{lm-evaluation-harness} \citep{eval-harness} framework on 12 downstream benchmarks: ARC-Easy, ARC-Challenge \citep{arc}, SciQ \citep{sciq}, MMLU \citep{mmlu}, MMLU-Pro \citep{mmlu_pro}, HellaSwag \citep{hellaswag}, OpenBookQA \citep{openbookqa}, PIQA \citep{piqa}, RACE \citep{race}, WinoGrande \citep{winogrande}, CommonsenseQA \citep{commonsenseqa}, and AGIEval-en \citep{agieval}. 
We use the default task prompts and decoding settings provided by the harness, without additional prompt engineering, chat templates, or manual overrides of the few-shot configuration. 
Following the harness defaults, RACE uses the \texttt{high} subset and WinoGrande uses the \texttt{winogrande\_xl} variant.

As summarized in Tab.~\ref{tab:eval_tasks}, all tasks in our suite are evaluated in zero-shot mode except MMLU-Pro, which uses the task-defined 5-shot setting. 
In particular, MMLU is evaluated in zero-shot mode in our setup.
Most benchmarks are standard multiple-choice evaluations. 
For ARC-Easy, ARC-Challenge, SciQ, HellaSwag, OpenBookQA, and PIQA, the framework reports accuracy scores. 
MMLU is reported as size-weighted accuracy aggregated over four subject groups. 
MMLU-Pro is evaluated as a 14-domain aggregate with greedy generation and exact-match scoring after answer extraction. 
AGIEval-en is evaluated through aggregating tasks over 10 English subtasks. This suite is predominantly multiple choice but includes one generative math subtask, and the final score is reported as size-weighted accuracy.


\subsection{Multi-seed Experiments}
\label{app:multi_seed}

To evaluate the robustness of the observed improvements to random seeds, we repeat the main comparison on Llama3.1-1.4B using three random seeds.
For each seed, Muon, the best static schedule, and AMO use the same model initialization and data order, enabling paired comparisons that reduce seed-level nuisance variation.
The best static schedule corresponds to the \texttt{Floor-1e-2}.

\begin{table}[t]
\small
\centering
\caption{Full results of multi-seed experiments on pre-training Llama3.1-1.4B. Methods within each seed use the same model initialization and data order.}
\setlength{\tabcolsep}{1.8pt}
\begin{tabular}{lccccccccccccc}
\toprule
\textbf{Optimizer} & \textbf{AE} & \textbf{AC} & \textbf{SciQ} & \textbf{MM} & \textbf{MM-P} & \textbf{HS} & \textbf{OBQA} & \textbf{PIQA} & \textbf{RACE} & \textbf{WG} & \textbf{CSQA} & \textbf{AGI} & \textbf{Avg} \\ 
\midrule

\textit{Seed=43} & \multicolumn{1}{l}{} & \multicolumn{1}{l}{} & \multicolumn{1}{l}{} & \multicolumn{1}{l}{} & \multicolumn{1}{l}{} & \multicolumn{1}{l}{} & \multicolumn{1}{l}{} & \multicolumn{1}{l}{} & \multicolumn{1}{l}{} & \multicolumn{1}{l}{} & \multicolumn{1}{l}{} & \multicolumn{1}{l}{} & \multicolumn{1}{l}{} \\
Muon
& 66.12 & 30.20 & 84.10 & 25.13 & 6.47 & 37.64 & 24.00 & 68.99 & 32.54 & 52.17 & 20.07 & 16.54 & 38.66 \\
Best Static
& 66.41 & 29.95 & 85.00 & 24.97 & 6.18 & 37.69 & 24.40 & 69.80 & 32.73 & 52.88 & 20.23 & 16.36 & \underline{38.88} \\
\textbf{AMO (ours)}
& 65.87 & 31.74 & 85.00 & 25.15 & 6.52 & 37.54 & 25.00 & 70.57 & 32.54 & 54.14 & 20.72 & 17.08 & \textbf{39.32} \\ 
\midrule

\textit{Seed=123} & \multicolumn{1}{l}{} & \multicolumn{1}{l}{} & \multicolumn{1}{l}{} & \multicolumn{1}{l}{} & \multicolumn{1}{l}{} & \multicolumn{1}{l}{} & \multicolumn{1}{l}{} & \multicolumn{1}{l}{} & \multicolumn{1}{l}{} & \multicolumn{1}{l}{} & \multicolumn{1}{l}{} & \multicolumn{1}{l}{} & \multicolumn{1}{l}{} \\
Muon
& 64.86 & 30.38 & 85.40 & 24.94 & 6.48 & 37.14 & 24.00 & 69.70 & 32.54 & 54.06 & 19.74 & 17.89 & \underline{38.93} \\
Best Static
& 65.74 & 30.46 & 84.60 & 25.03 & 6.04 & 37.61 & 23.00 & 70.46 & 31.67 & 55.25 & 19.49 & 16.64 & 38.83 \\
\textbf{AMO (ours)}
& 65.32 & 32.68 & 85.00 & 26.65 & 6.87 & 37.30 & 25.60 & 70.55 & 31.77 & 54.78 & 20.48 & 17.39 & \textbf{39.53} \\
\midrule

\textit{Seed=2026} & \multicolumn{1}{l}{} & \multicolumn{1}{l}{} & \multicolumn{1}{l}{} & \multicolumn{1}{l}{} & \multicolumn{1}{l}{} & \multicolumn{1}{l}{} & \multicolumn{1}{l}{} & \multicolumn{1}{l}{} & \multicolumn{1}{l}{} & \multicolumn{1}{l}{} & \multicolumn{1}{l}{} & \multicolumn{1}{l}{} & \multicolumn{1}{l}{} \\
Muon
& 65.53 & 30.72 & 85.20 & 25.19 & 5.28 & 37.68 & 26.00 & 70.95 & 33.01 & 53.91 & 20.80 & 17.34 & \underline{39.30} \\
Best Static
& 65.66 & 30.29 & 84.00 & 25.07 & 6.87 & 37.80 & 24.20 & 68.82 & 32.82 & 55.80 & 19.25 & 16.64 & 38.94 \\
\textbf{AMO (ours)}
& 66.46 & 31.20 & 85.40 & 24.96 & 7.30 & 37.14 & 26.00 & 70.29 & 32.92 & 54.46 & 21.29 & 17.19 & \textbf{39.55} \\

\bottomrule
\end{tabular}
\label{tab:multi_seed}
\end{table}

As shown in Tab.~\ref{tab:multi_seed}, AMO outperforms Muon under all three seeds.
Across paired runs, AMO improves over Muon by +0.50 $\pm$ 0.22 points on average and over the best static schedule by +0.58 $\pm$ 0.13 points.

\subsection{Large-scale Experiments}
\label{app:large_scale}

We further evaluate whether the gains of AMO persist at a larger model scale by pre-training Llama3.1-4B for 100B tokens.

\begin{table}[t]
\small
\centering
\caption{Full results of large-scale experiments on pre-training Llama3.1-4B for 100B tokens.}
\setlength{\tabcolsep}{1.8pt}
\begin{tabular}{lccccccccccccc}
\toprule
\textbf{Optimizer} & \textbf{AE} & \textbf{AC} & \textbf{SciQ} & \textbf{MM} & \textbf{MM-P} & \textbf{HS} & \textbf{OBQA} & \textbf{PIQA} & \textbf{RACE} & \textbf{WG} & \textbf{CSQA} & \textbf{AGI} & \textbf{Avg} \\ 
\midrule

\textit{Llama3.1-4B} & \multicolumn{1}{l}{} & \multicolumn{1}{l}{} & \multicolumn{1}{l}{} & \multicolumn{1}{l}{} & \multicolumn{1}{l}{} & \multicolumn{1}{l}{} & \multicolumn{1}{l}{} & \multicolumn{1}{l}{} & \multicolumn{1}{l}{} & \multicolumn{1}{l}{} & \multicolumn{1}{l}{} & \multicolumn{1}{l}{} & \multicolumn{1}{l}{} \\
Muon
& 72.60 & 38.74 & 90.10 & 25.43 & 8.80 & 45.24 & 29.00 & 73.23 & 34.26 & 56.27 & 19.49 & 17.32 & 42.54 \\
\textbf{AMO (ours)}
& 73.65 & 38.48 & 90.80 & 26.19 & 9.07 & 45.32 & 29.20 & 74.16 & 36.08 & 56.59 & 22.44 & 17.50 & \textbf{43.29} \\

\bottomrule
\end{tabular}
\label{tab:large_scale}
\end{table}

As shown in Tab.~\ref{tab:large_scale}, AMO improves the average
downstream score from 42.54 to 43.29, yielding a gain of
0.75 percentage points over Muon.
This result extends the observed advantage of AMO beyond the 2B regime and shows that its performance improvement persists at the 4B scale.

\begin{table}[t]
\small
\centering
\caption{Full results of main experiments on pre-training Llama3.1-760M/1.4B.}
\setlength{\tabcolsep}{1.8pt}
\begin{tabular}{lccccccccccccc}
\toprule
\textbf{Optimizer}     & \textbf{AE}          & \textbf{AC}          & \textbf{SciQ}        & \textbf{MM}          & \textbf{MM-P}        & \textbf{HS}          & \textbf{OBQA}        & \textbf{PIQA}        & \textbf{RACE}        & \textbf{WG}          & \textbf{CSQA}        & \textbf{AGI}         & \textbf{Avg}         \\ \midrule
\textit{Llama3.1-760M} & \multicolumn{1}{l}{} & \multicolumn{1}{l}{} & \multicolumn{1}{l}{} & \multicolumn{1}{l}{} & \multicolumn{1}{l}{} & \multicolumn{1}{l}{} & \multicolumn{1}{l}{} & \multicolumn{1}{l}{} & \multicolumn{1}{l}{} & \multicolumn{1}{l}{} & \multicolumn{1}{l}{} & \multicolumn{1}{l}{} & \multicolumn{1}{l}{} \\
AdamW                  & 59.89                & 27.22                & 79.90                & 23.93                & 7.36                 & 33.53                & 23.20                & 67.03                & 31.10                & 51.30                & 19.33                & 16.46                & 36.69                \\
Muon-KJ                & 59.68                & 25.00                & 79.90                & 24.78                & 6.91                 & 33.64                & 24.40                & 65.89                & 31.48                & 49.88                & 20.48                & 17.83                & 36.66                \\
Muon-You               & 60.27                & 26.37                & 80.50                & 23.47                & 6.46                 & 33.41                & 23.00                & 67.19                & 29.67                & 51.78                & 19.82                & 17.03                & 36.58                \\
Muon-PE                & 59.89                & 27.39                & 81.80                & 23.79                & 8.15                 & 33.88                & 22.80                & 67.68                & 30.62                & 51.38                & 19.66                & 17.65                & \underline{37.06}    \\
Muon-CANS              & 59.51                & 25.34                & 81.40                & 26.30                & 6.44                 & 33.23                & 22.80                & 66.81                & 29.76                & 50.83                & 20.31                & 17.76                & 36.71                \\
Muon-Turbo             & 59.81                & 25.94                & 79.80                & 23.47                & 6.70                 & 33.39                & 22.40                & 65.67                & 30.81                & 51.22                & 20.39                & 17.03                & 36.39                \\
AdaMuon                & 58.29                & 25.34                & 80.60                & 24.32                & 7.58                 & 32.77                & 23.00                & 67.30                & 31.10                & 51.30                & 19.25                & 17.91                & 36.56                \\
ROOT                   & 59.85                & 26.19                & 80.80                & 25.41                & 7.52                 & 33.54                & 19.60                & 67.36                & 30.62                & 51.54                & 20.88                & 17.26                & 36.71                \\
\textbf{AMO (ours)}    & 59.43                & 26.79                & 81.50                & 26.24                & 6.76                 & 33.79                & 22.80                & 67.70                & 31.43                & 51.91                & 21.13                & 17.60                & \textbf{37.26}       \\ \midrule
\textit{Llama3.1-1.4B} & \multicolumn{1}{l}{} & \multicolumn{1}{l}{} & \multicolumn{1}{l}{} & \multicolumn{1}{l}{} & \multicolumn{1}{l}{} & \multicolumn{1}{l}{} & \multicolumn{1}{l}{} & \multicolumn{1}{l}{} & \multicolumn{1}{l}{} & \multicolumn{1}{l}{} & \multicolumn{1}{l}{} & \multicolumn{1}{l}{} & \multicolumn{1}{l}{} \\
AdamW                  & 65.84                & 28.67                & 84.20                & 24.95                & 6.13                 & 37.94                & 22.60                & 70.23                & 33.00                & 53.43                & 20.23                & 16.74                & 38.66                \\
Muon-KJ                & 65.32                & 28.41                & 85.20                & 25.47                & 6.94                 & 37.32                & 25.00                & 69.36                & 32.15                & 53.67                & 19.99                & 16.87                & 38.81                \\
Muon-You               & 64.81                & 30.80                & 84.30                & 24.86                & 7.47                 & 37.55                & 25.00                & 69.70                & 32.20                & 53.28                & 19.57                & 17.34                & 38.91                \\
Muon-PE                & 66.41                & 30.03                & 84.80                & 23.81                & 7.53                 & 37.54                & 24.00                & 69.15                & 32.34                & 53.67                & 20.64                & 17.24                & \underline{38.93}    \\
Muon-CANS              & 65.57                & 31.74                & 82.20                & 25.00                & 7.75                 & 37.83                & 24.20                & 69.75                & 32.63                & 52.57                & 20.64                & 17.00                & 38.91                \\
Muon-Turbo             & 64.77                & 30.03                & 84.90                & 25.33                & 7.86                 & 37.12                & 24.00                & 70.13                & 32.44                & 52.72                & 20.23                & 17.29                & 38.90                \\
AdaMuon                & 63.17                & 29.52                & 83.20                & 25.23                & 6.50                 & 36.19                & 26.00                & 68.61                & 31.77                & 52.25                & 21.29                & 17.50                & 38.44                \\
NorMuon                & 65.28                & 29.44                & 84.70                & 24.04                & 7.09                 & 37.86                & 23.40                & 69.48                & 31.48                & 54.93                & 19.98                & 16.82                & 38.71                \\
Drop-Muon              & 58.71                & 26.02                & 79.70                & 26.21                & 8.04                 & 32.05                & 21.60                & 66.54                & 30.33                & 50.43                & 20.88                & 17.26                & 36.48                \\
PRISM                  & 61.78                & 27.47                & 81.50                & 25.02                & 6.09                 & 34.60                & 23.60                & 67.63                & 30.91                & 50.43                & 19.98                & 17.52                & 37.21                \\
ROOT                   & 65.74                & 29.62                & 83.90                & 25.91                & 7.50                 & 37.38                & 24.00                & 69.96                & 31.39                & 53.12                & 20.23                & 16.98                & 38.81                \\
\textbf{AMO (ours)}    & 66.46                & 31.20                & 85.40                & 25.60                & 7.99                 & 37.46                & 25.60                & 70.50                & 33.00                & 53.43                & 21.29                & 18.30                & \textbf{39.69}       \\ \bottomrule
\end{tabular}
\label{tab:main_full_1}
\end{table}

\begin{table}[t]
\small
\centering
\caption{Full results of main experiments on pre-training Qwen3-0.6B/1.7B.}
\setlength{\tabcolsep}{2.2pt}
\begin{tabular}{lccccccccccccc}
\toprule
\textbf{Optimizer}  & \textbf{AE}          & \textbf{AC}          & \textbf{SciQ}        & \textbf{MM}          & \textbf{MM-P}        & \textbf{HS}          & \textbf{OBQA}        & \textbf{PIQA}        & \textbf{RACE}        & \textbf{WG}          & \textbf{CSQA}        & \textbf{AGI}         & \textbf{Avg}         \\ \midrule
\textit{Qwen3-0.6B} & \multicolumn{1}{l}{} & \multicolumn{1}{l}{} & \multicolumn{1}{l}{} & \multicolumn{1}{l}{} & \multicolumn{1}{l}{} & \multicolumn{1}{l}{} & \multicolumn{1}{l}{} & \multicolumn{1}{l}{} & \multicolumn{1}{l}{} & \multicolumn{1}{l}{} & \multicolumn{1}{l}{} & \multicolumn{1}{l}{} & \multicolumn{1}{l}{} \\
AdamW               & 58.96                & 24.49                & 82.00                & 23.51                & 6.71                 & 33.41                & 21.00                & 66.27                & 30.81                & 50.12                & 19.90                & 17.99                & 36.26                \\
Muon-KJ             & 62.16                & 26.79                & 81.90                & 23.80                & 7.30                 & 34.75                & 23.20                & 67.90                & 31.67                & 51.93                & 19.00                & 16.92                & 37.28                \\
Muon-You            & 60.77                & 27.22                & 82.20                & 24.45                & 8.22                 & 34.60                & 22.80                & 67.03                & 31.39                & 53.04                & 19.00                & 17.65                & 37.36                \\
Muon-PE             & 61.53                & 26.79                & 83.50                & 24.01                & 6.09                 & 34.75                & 22.20                & 67.85                & 31.96                & 52.17                & 20.48                & 16.64                & 37.33                \\
Muon-CANS           & 62.33                & 26.19                & 80.30                & 24.09                & 7.95                 & 35.04                & 22.20                & 66.92                & 32.06                & 52.72                & 20.39                & 17.00                & 37.27                \\
Muon-Turbo          & 61.49                & 25.51                & 82.70                & 23.45                & 6.40                 & 35.05                & 24.80                & 68.66                & 30.72                & 52.09                & 19.82                & 17.71                & \underline{37.37}    \\
AdaMuon             & 58.16                & 24.91                & 80.80                & 24.02                & 6.10                 & 33.33                & 22.20                & 66.92                & 31.96                & 52.49                & 20.48                & 17.39                & 36.56                \\
ROOT                & 60.01                & 25.60                & 81.00                & 25.43                & 6.26                 & 33.99                & 20.80                & 66.97                & 29.76                & 50.28                & 21.70                & 17.68                & 36.62                \\
\textbf{AMO (ours)} & 63.17                & 27.19                & 82.60                & 24.89                & 7.80                 & 35.04                & 21.60                & 68.66                & 31.53                & 53.09                & 20.23                & 17.39                & \textbf{37.77}       \\ \midrule
\textit{Qwen3-1.7B} & \multicolumn{1}{l}{} & \multicolumn{1}{l}{} & \multicolumn{1}{l}{} & \multicolumn{1}{l}{} & \multicolumn{1}{l}{} & \multicolumn{1}{l}{} & \multicolumn{1}{l}{} & \multicolumn{1}{l}{} & \multicolumn{1}{l}{} & \multicolumn{1}{l}{} & \multicolumn{1}{l}{} & \multicolumn{1}{l}{} & \multicolumn{1}{l}{} \\
AdamW               & 68.81                & 30.34                & 87.20                & 24.70                & 6.60                 & 40.56                & 25.00                & 71.76                & 33.30                & 55.64                & 18.67                & 16.46                & 39.92                \\
Muon-KJ             & 68.22                & 31.74                & 85.70                & 24.80                & 5.54                 & 40.71                & 26.40                & 72.09                & 33.78                & 55.00                & 20.72                & 16.45                & 40.10                \\
Muon-You            & 67.59                & 30.20                & 86.90                & 24.35                & 6.87                 & 40.47                & 25.40                & 72.52                & 34.07                & 56.12                & 19.25                & 17.89                & 40.14                \\
Muon-PE             & 68.98                & 31.66                & 85.30                & 25.13                & 7.54                 & 40.96                & 26.40                & 71.82                & 32.44                & 54.78                & 20.88                & 16.93                & \underline{40.24}    \\
Muon-CANS           & 69.32                & 32.08                & 86.50                & 24.65                & 6.18                 & 41.04                & 25.00                & 70.84                & 33.78                & 54.22                & 17.28                & 16.46                & 39.78                \\
Muon-Turbo          & 67.17                & 31.74                & 86.20                & 24.50                & 6.81                 & 39.93                & 24.40                & 71.55                & 33.01                & 53.75                & 20.80                & 16.51                & 39.70                \\
AdaMuon             & 67.42                & 31.48                & 85.80                & 24.78                & 5.24                 & 39.66                & 25.60                & 70.89                & 33.68                & 53.04                & 20.39                & 16.90                & 39.57                \\
NorMuon             & 67.17                & 30.80                & 86.00                & 25.07                & 7.40                 & 40.63                & 26.40                & 71.11                & 33.21                & 53.91                & 22.69                & 17.94                & 40.19                \\
Drop-Muon           & 55.93                & 22.44                & 76.40                & 23.90                & 6.67                 & 32.81                & 21.20                & 65.51                & 29.86                & 52.01                & 20.31                & 17.47                & 35.38                \\
PRISM               & 64.10                & 27.82                & 82.60                & 24.48                & 6.54                 & 36.63                & 24.40                & 69.86                & 30.53                & 51.46                & 20.48                & 17.99                & 38.07                \\
ROOT                & 69.95                & 31.74                & 87.70                & 25.25                & 7.68                 & 40.53                & 25.80                & 71.76                & 31.77                & 54.99                & 18.59                & 16.25                & 40.17                \\
\textbf{AMO (ours)} & 69.36                & 32.68                & 86.00                & 25.13                & 7.28                 & 41.07                & 26.40                & 71.76                & 34.00                & 56.20                & 21.13                & 18.00                & \textbf{40.75}       \\ \bottomrule
\end{tabular}
\label{tab:main_full_2}
\end{table}

\begin{table}[t]
\centering
\small
\caption{Full results of prolonged pre-training on Qwen3-0.6B under 2$\times$, 5$\times$, and 10$\times$ Chinchilla optimal token-to-parameter ratios.}
\setlength{\tabcolsep}{2.2pt}
\begin{tabular}{lccccccccccccc}
\toprule
\textbf{Optimizer}  & \textbf{AE}          & \textbf{AC}          & \textbf{SciQ}        & \textbf{MM}          & \textbf{MM-P}        & \textbf{HS}          & \textbf{OBQA}        & \textbf{PIQA}        & \textbf{RACE}        & \textbf{WG}          & \textbf{CSQA}        & \textbf{AGI}         & \textbf{Avg}         \\ \midrule
\textit{2$\times$}         & \multicolumn{1}{l}{} & \multicolumn{1}{l}{} & \multicolumn{1}{l}{} & \multicolumn{1}{l}{} & \multicolumn{1}{l}{} & \multicolumn{1}{l}{} & \multicolumn{1}{l}{} & \multicolumn{1}{l}{} & \multicolumn{1}{l}{} & \multicolumn{1}{l}{} & \multicolumn{1}{l}{} & \multicolumn{1}{l}{} & \multicolumn{1}{l}{} \\
Muon-KJ             & 62.21                & 27.05                & 81.30                & 24.33                & 5.84                 & 35.81                & 22.60                & 68.28                & 31.20                & 52.96                & 18.76                & 17.58                & 37.33                \\
Muon-PE             & 62.25                & 27.03                & 82.00                & 25.00                & 7.77                 & 35.04                & 22.40                & 66.92                & 32.06                & 52.72                & 19.74                & 17.19                & 37.51                \\
\textbf{AMO (ours)} & 62.37                & 27.22                & 82.30                & 25.23                & 7.02                 & 35.74                & 24.00                & 69.15                & 31.58                & 52.17                & 20.33                & 17.65                & \textbf{37.90}       \\ \midrule
\textit{5$\times$}         &                      &                      &                      &                      &                      &                      &                      &                      &                      &                      &                      &                      &                      \\
Muon-KJ             & 65.91                & 28.75                & 85.50                & 25.54                & 6.58                 & 37.29                & 24.20                & 68.88                & 32.73                & 55.01                & 19.66                & 17.32                & 38.95                \\
Muon-PE             & 66.00                & 27.30                & 85.00                & 25.00                & 7.77                 & 37.54                & 24.60                & 69.00                & 32.34                & 55.67                & 20.64                & 17.65                & 39.04                \\
\textbf{AMO (ours)} & 65.76                & 28.33                & 85.70                & 25.02                & 7.77                 & 37.61                & 25.60                & 69.21                & 32.67                & 55.23                & 20.37                & 17.98                & \textbf{39.27}       \\ \midrule
\textit{10$\times$}        &                      &                      &                      &                      &                      &                      &                      &                      &                      &                      &                      &                      &                      \\
Muon-KJ             & 66.25                & 33.03                & 86.20                & 25.35                & 6.89                 & 38.33                & 25.40                & 69.70                & 32.73                & 55.64                & 18.84                & 17.73                & 39.67                \\
Muon-PE             & 66.98                & 32.08                & 86.00                & 25.13                & 7.54                 & 37.83                & 26.00                & 70.23                & 32.44                & 54.78                & 21.29                & 17.00                & 39.78                \\
\textbf{AMO (ours)} & 66.00                & 33.78                & 86.40                & 25.13                & 7.72                 & 38.77                & 25.60                & 69.48                & 33.49                & 55.93                & 20.64                & 17.82                & \textbf{40.06}       \\ \bottomrule
\end{tabular}
\label{tab:prolong_full}
\end{table}

\begin{table}[]
\centering
\small
\caption{Full results of continual pre-training with additional 36B tokens on a Qwen3-0.6B checkpoint pretrained on 60B tokens. }
\setlength{\tabcolsep}{2.2pt}
\begin{tabular}{@{}lccccccccccccc@{}}
\toprule
\textbf{Optimizer}  & \textbf{AE}          & \textbf{AC}          & \textbf{SciQ}        & \textbf{MM}          & \textbf{MM-P}        & \textbf{HS}          & \textbf{OBQA}        & \textbf{PIQA}        & \textbf{RACE}        & \textbf{WG}          & \textbf{CSQA}        & \textbf{AGI}         & \textbf{Avg}         \\ \midrule
\textit{Step=2000}  & \multicolumn{1}{l}{} & \multicolumn{1}{l}{} & \multicolumn{1}{l}{} & \multicolumn{1}{l}{} & \multicolumn{1}{l}{} & \multicolumn{1}{l}{} & \multicolumn{1}{l}{} & \multicolumn{1}{l}{} & \multicolumn{1}{l}{} & \multicolumn{1}{l}{} & \multicolumn{1}{l}{} & \multicolumn{1}{l}{} & \multicolumn{1}{l}{} \\
Muon-KJ             & 64.94                & 28.41                & 87.00                & 24.20                & 6.59                 & 37.05                & 23.20                & 68.50                & 33.11                & 55.01                & 20.72                & 17.60                & 38.86                \\
Muon-PE             & 66.51                & 29.61                & 86.00                & 25.00                & 7.21                 & 38.63                & 22.00                & 70.24                & 33.88                & 52.56                & 19.77                & 16.43                & 38.99                \\
\textbf{AMO (ours)} & 63.72                & 29.27                & 86.20                & 24.79                & 7.01                 & 36.52                & 24.20                & 68.99                & 33.40                & 54.62                & 22.44                & 18.07                & \textbf{39.10}       \\ \midrule
\textit{Step=4000}  &                      &                      &                      &                      &                      &                      &                      &                      &                      &                      &                      &                      &                      \\
Muon-KJ             & 65.32                & 29.61                & 86.40                & 24.63                & 6.46                 & 37.94                & 23.60                & 68.66                & 32.54                & 53.43                & 19.98                & 17.96                & 38.88                \\
Muon-PE             & 65.17                & 31.23                & 86.00                & 24.89                & 7.80                 & 37.94                & 24.00                & 68.00                & 31.53                & 54.06                & 19.77                & 17.72                & 39.01                \\
\textbf{AMO (ours)} & 65.53                & 31.23                & 86.00                & 24.55                & 5.72                 & 37.74                & 23.40                & 70.02                & 33.68                & 53.99                & 21.05                & 17.72                & \textbf{39.22}       \\ \midrule
\textit{Step=6000}  &                      &                      &                      &                      &                      &                      &                      &                      &                      &                      &                      &                      &                      \\
Muon-KJ             & 67.47                & 31.40                & 87.10                & 26.51                & 6.99                 & 39.14                & 24.40                & 69.59                & 32.73                & 54.06                & 20.15                & 17.86                & 39.78                \\
\textbf{Muon-PE}    & 68.32                & 32.08                & 87.00                & 25.00                & 6.18                 & 39.00                & 25.00                & 70.84                & 33.00                & 54.22                & 19.74                & 16.46                & 39.74                \\
\textbf{AMO (ours)} & 68.64                & 31.97                & 86.80                & 26.04                & 6.23                 & 39.05                & 25.40                & 69.70                & 33.78                & 54.96                & 20.31                & 17.60                & \textbf{40.04}       \\ \midrule
\textit{Step=8000}           &                      &                      &                      &                      &                      &                      &                      &                      &                      &                      &                      &                      &                      \\
Muon-KJ             & 66.88                & 31.74                & 86.80                & 26.21                & 7.25                 & 39.07                & 25.00                & 70.35                & 33.59                & 53.67                & 21.62                & 17.71                & 39.99                \\
Muon-PE             & 68.30                & 31.00                & 87.00                & 24.63                & 5.54                 & 39.11                & 26.00                & 72.09                & 33.78                & 55.00                & 20.72                & 17.74                & 40.08                \\
\textbf{AMO (ours)} & 68.39                & 30.29                & 86.60                & 25.06                & 7.77                 & 39.18                & 26.40                & 70.38                & 33.88                & 54.96                & 21.64                & 17.19                & \textbf{40.15}       \\ \midrule
\textit{Step=8580}           &                      &                      &                      &                      &                      &                      &                      &                      &                      &                      &                      &                      &                      \\
Muon-KJ             & 66.50                & 31.57                & 86.60                & 26.20                & 7.77                 & 39.12                & 25.00                & 70.08                & 33.40                & 54.22                & 21.70                & 17.94                & 40.01                \\
Muon-PE             & 66.00                & 32.00                & 88.00                & 25.60                & 7.99                 & 39.10                & 25.60                & 70.50                & 34.00                & 53.43                & 21.29                & 18.30                & 40.15                \\
\textbf{AMO (ours)} & 68.56                & 31.74                & 88.20                & 25.00                & 7.99                 & 39.10                & 26.00                & 70.18                & 34.45                & 54.56                & 20.23                & 17.68                & \textbf{40.31}       \\ \bottomrule
\end{tabular}
\label{tab:cpt_full}
\end{table}

\subsection{Main Experiments}
\label{app:main_full}

Full results of pre-training Llama3.1-760M/1.4B and Qwen3-0.6B/1.7B using all optimizers in this paper are provided in Tabs.~\ref{tab:main_full_1} and~\ref{tab:main_full_2}.

\subsection{Prolonged Pre-training Experiments}
\label{app:prolong_full}

Full results of prolonged pre-training are provided in Tab.~\ref{tab:prolong_full}.

\subsection{Continual Pre-training Experiments}
\label{app:cpt_full}

Full results of continual pre-training are provided in Tab.~\ref{tab:cpt_full}.

\subsection{Ablation Experiments}
\label{app:ablation_full}

Full results of ablation experiments on \textit{Observation horizon and interval}, \textit{Observation sample size}, \textit{Shrinkage}, and \textit{Transition duration} of AMO are provided in Tab.~\ref{tab:ablation_full}.

\begin{table}[t]
\centering
\small
\caption{Full downstream results of Qwen3-0.6B pre-training under four key designs of AMO: \textit{Observation horizon and interval}, \textit{Observation sample size}, \textit{Shrinkage}, and \textit{Transition duration}.}
\setlength{\tabcolsep}{2.2pt}
\begin{tabular}{lccccccccccccc}
\toprule
\textbf{Setting} & \textbf{AE} & \textbf{AC} & \textbf{SciQ} & \textbf{MM} & \textbf{MM-P} & \textbf{HS} & \textbf{OBQA} & \textbf{PIQA} & \textbf{RACE} & \textbf{WG} & \textbf{CSQA} & \textbf{AGI} & \textbf{Avg} \\ \midrule
\multicolumn{14}{l}{\textit{Observation horizon and interval}} \\
(800, 100)  & 62.09 & 27.11 & 82.70 & 25.21 & 7.51 & 34.17 & 21.80 & 68.00 & 31.91 & 52.80 & 20.41 & 18.43 & 37.68 \\
(1200, 150) & 63.17 & 27.19 & 82.60 & 24.89 & 7.80 & 35.04 & 21.60 & 68.66 & 31.53 & 53.09 & 20.23 & 17.39 & \textbf{37.77} \\
(1600, 200) & 63.72 & 26.89 & 82.00 & 25.67 & 6.27 & 33.98 & 21.60 & 68.01 & 31.48 & 53.51 & 20.48 & 18.00 & 37.63 \\ \midrule
\multicolumn{14}{l}{\textit{Observation sample size}} \\
4  & 62.34 & 26.91 & 82.60 & 25.40 & 7.89 & 34.15 & 21.40 & 68.01 & 31.81 & 53.28 & 20.82 & 17.24 & 37.65 \\
8  & 63.17 & 27.19 & 82.60 & 24.89 & 7.80 & 35.04 & 21.60 & 68.66 & 31.53 & 53.09 & 20.23 & 17.39 & 37.77 \\
16 & 63.15 & 27.09 & 82.80 & 25.00 & 7.45 & 35.00 & 21.80 & 68.46 & 31.86 & 53.80 & 20.23 & 17.37 & \textbf{37.83} \\ \midrule
\multicolumn{14}{l}{\textit{Shrinkage}} \\
0   & 62.43 & 24.74 & 82.10 & 24.53 & 6.13 & 33.91 & 23.60 & 66.97 & 31.39 & 53.20 & 22.03 & 17.19 & 37.35 \\
0.5 & 63.00 & 25.17 & 80.30 & 23.25 & 7.29 & 35.95 & 24.60 & 67.03 & 30.91 & 54.54 & 22.36 & 18.02 & 37.70 \\
0.7 & 63.17 & 27.19 & 82.60 & 24.89 & 7.80 & 35.04 & 21.60 & 68.66 & 31.53 & 53.09 & 20.23 & 17.39 & \textbf{37.77} \\
1   & 63.09 & 24.74 & 82.80 & 26.21 & 7.86 & 33.84 & 24.20 & 67.14 & 30.62 & 52.09 & 20.73 & 17.88 & 37.60 \\ \midrule
\multicolumn{14}{l}{\textit{Transition duration}} \\
1   & 60.98 & 25.68 & 81.60 & 25.24 & 6.00 & 33.91 & 23.00 & 68.41 & 31.87 & 53.20 & 19.66 & 17.68 & 37.27 \\
100 & 63.86 & 25.77 & 81.70 & 24.09 & 6.58 & 34.02 & 23.60 & 68.81 & 31.00 & 52.17 & 19.25 & 17.91 & 37.40 \\
300 & 63.17 & 27.19 & 82.60 & 24.89 & 7.80 & 35.04 & 21.60 & 68.66 & 31.53 & 53.09 & 20.23 & 17.39 & \textbf{37.77} \\
600 & 62.60 & 24.49 & 82.80 & 25.84 & 6.77 & 35.79 & 24.00 & 66.70 & 31.29 & 52.17 & 20.88 & 17.60 & 37.58 \\ \bottomrule
\end{tabular}
\label{tab:ablation_full}
\end{table}

\subsection{Budget Control Experiments}
\label{app:budget_control_full}

\paragraph{Setup.}

To study compute-quality tradeoff, we vary only the global budget ratio on \textbf{Qwen3-0.6B}. 
Specifically, we evaluate five settings with budget ratio $r \in \{0.8, 0.9, 1.0, 1.1, 1.2\}$. 
Since the baseline uses 5 NS steps for each of 7 layer types, these settings correspond to total step budgets of 28, 32, 35, 39, and 42, respectively. 
All other components are kept identical across runs, including the observation phase. 
In particular, all runs share the same model, data pipeline, optimizer hyperparameters, learning-rate schedule, observation procedure, and execution-phase settings.

All runs begin with the same observation phase under a uniform PE baseline, where each of the 7 Muon-managed layer types uses 5 NS steps. 
The baseline total budget is therefore \(7 \times 5 = 35\). 
Given a budget ratio \(r\), the total execution-phase budget is set to
\[
B = \mathrm{round}(35\,r).
\]
Thus, the five settings above correspond to \(B \in \{28, 32, 35, 39, 42\}\). 
During the observation phase, our AMO keeps the uniform 5-step schedule fixed and only collects spectral statistics. Afterward, it constructs candidate schedules for each layer type within the same step range \([3,7]\), and solves the final execution-phase allocation under the constraint
\[
\sum_t s_t = B,
\]
where \(s_t\) denotes the NS step count assigned to layer type \(t\). Hence, the only difference across runs is the total orthogonalization budget available when solving the final layer-type-level step allocation.

Changing \(r\) affects only the orthogonalization compute and leaves the rest of training unchanged. 
For a given layer type, one additional NS step appends one more quintic matrix update on the same tensor shape, so the orthogonalization FLOPs per optimizer step scale approximately linearly with the assigned step count. 
Relative to the default setting \(r=1.0\) (\(B=35\)), \(r=0.8\) and \(r=0.9\) reduce the orthogonalization budget by \(20\%\) and \(8.6\%\), while \(r=1.1\) and \(r=1.2\) increase it by \(11.4\%\) and \(20\%\), respectively.

\paragraph{Results.}

Full results are provided in Tab.~\ref{tab:budget_control_full}.

\begin{table}[]
\centering
\small
\caption{Full results of Llama3.1-760M and Qwen3-0.6B pre-training under budget control settings.}
\setlength{\tabcolsep}{1.8pt}
\begin{tabular}{lccccccccccccc}
\toprule
\textbf{\begin{tabular}[c]{@{}l@{}}Budget\\ Ratio\end{tabular}} & \textbf{AE}          & \textbf{AC}          & \textbf{SciQ}        & \textbf{MM}          & \textbf{MM-P}        & \textbf{HS}          & \textbf{OBQA}        & \textbf{PIQA}        & \textbf{RACE}        & \textbf{WG}          & \textbf{CSQA}        & \textbf{AGI}         & \textbf{Avg}         \\ \midrule
\textit{Qwen3-0.6B}                                             & \multicolumn{1}{l}{} & \multicolumn{1}{l}{} & \multicolumn{1}{l}{} & \multicolumn{1}{l}{} & \multicolumn{1}{l}{} & \multicolumn{1}{l}{} & \multicolumn{1}{l}{} & \multicolumn{1}{l}{} & \multicolumn{1}{l}{} & \multicolumn{1}{l}{} & \multicolumn{1}{l}{} & \multicolumn{1}{l}{} & \multicolumn{1}{l}{} \\
0.8                                                             & 62.25                & 26.02                & 82.40                & 24.56                & 7.95                 & 33.89                & 22.60                & 67.76                & 30.72                & 53.51                & 21.48                & 17.19                & 37.53       \\
0.9                                                             & 62.68                & 26.74                & 82.00                & 25.00                & 6.77                 & 33.93                & 22.40                & 68.20                & 31.00                & 53.70                & 21.62                & 18.30                & 37.70                \\
1.0                                                             & 63.17                & 27.19                & 82.60                & 24.89                & 7.80                 & 35.04                & 21.60                & 68.66                & 31.53                & 53.09                & 20.23                & 17.39                & 37.77                \\
1.1                                                             & 64.26                & 27.19                & 83.00                & 25.30                & 7.77                 & 34.04                & 23.00                & 67.68                & 30.05                & 53.04                & 20.97                & 17.89                & 37.85       \\
1.2                                                             & 64.39                & 27.30                & 83.80                & 25.82                & 7.95                 & 34.12                & 23.60                & 67.30                & 30.60                & 50.67                & 20.97                & 18.00                & \textbf{37.88}       \\ \midrule
\textit{Llama3.1-760M}                                          & \multicolumn{1}{l}{} & \multicolumn{1}{l}{} & \multicolumn{1}{l}{} & \multicolumn{1}{l}{} & \multicolumn{1}{l}{} & \multicolumn{1}{l}{} & \multicolumn{1}{l}{} & \multicolumn{1}{l}{} & \multicolumn{1}{l}{} & \multicolumn{1}{l}{} & \multicolumn{1}{l}{} & \multicolumn{1}{l}{} & \multicolumn{1}{l}{} \\
0.8                                                             & 59.93                & 26.74                & 81.00                & 25.00                & 7.77                 & 33.45                & 21.80                & 65.72                & 30.33                & 52.64                & 20.03                & 17.08                & 36.79                \\
0.9                                                             & 59.39                & 25.85                & 82.00                & 26.37                & 7.45                 & 33.43                & 22.00                & 66.00                & 30.53                & 52.72                & 21.00                & 17.03                & 36.98                \\
1.0                                                             & 59.43                & 26.79                & 81.50                & 26.24                & 6.76                 & 33.79                & 22.80                & 67.70                & 31.43                & 51.91                & 21.13                & 17.60                & 37.26                \\
1.1                                                             & 60.18                & 26.19                & 82.80                & 25.43                & 7.80                 & 33.60                & 22.00                & 67.03                & 31.39                & 52.51                & 20.97                & 17.42                & 37.28                \\
1.2                                                             & 60.77                & 26.79                & 82.00                & 25.40                & 7.02                 & 33.26                & 22.80                & 67.80                & 31.72                & 53.51                & 20.13                & 17.42                & \textbf{37.39}       \\ \bottomrule
\end{tabular}
\label{tab:budget_control_full}
\end{table}

\subsection{Static Coefficient Experiments}
\label{app:static_coeff}

\paragraph{Setup.}

Inspired by Muon-PE and Muon-CANS~\citep{pe,cans}, we construct the following static schedules fully offline and freeze them throughout training. 
For each Muon-managed layer type (\texttt{attn\_q}, \texttt{attn\_k}, \texttt{attn\_v}, \texttt{attn\_o}, \texttt{mlp\_gate}, \texttt{mlp\_up}, and \texttt{mlp\_down}), we first estimate an effective PE lower-bound statistic from geometry logs,
\[
\ell_{\mathrm{eff}}=\frac{\max(\sigma_{\min},\tau\sigma_{\max})}{\|M\|_F}.
\]
We then aggregate this statistic offline to obtain a fixed \(\ell_t\) for each layer type. 
The \texttt{Typical}, \texttt{Conservative}, and \texttt{Floor} variants use the \(50\)th, \(25\)th, and \(10\)th percentiles of the aggregated \(\ell_{\mathrm{eff}}\), respectively. 
Given \(\ell_t\), we build a PE schedule by optimal composition and choose the smallest number of NS steps \(s_t\in\{1,\dots,8\}\) whose simulated final approximation error satisfies \(1-l_t^{(s_t)}\le\varepsilon\), where \(\varepsilon\in\{1\times10^{-2}, 5\times10^{-2}\}\). 
The corresponding coefficient triplets \(\{(a_i,b_i,c_i)\}_{i=1}^{s_t}\) are then fixed for that layer type for the entire run. 
Additionally, \textit{llama family} or \textit{qwen family} indicates that the offline statistics are pooled only within the corresponding model family; otherwise we use the generic pooled statistics. 
All other training settings are kept identical to the corresponding main experiments.

\paragraph{Results.}

Full results are provided in Tabs.~\ref{tab:static_coeff_llama} and~\ref{tab:static_coeff_qwen}.

\begin{table*}[tb]
\centering
\small
\caption{Full results of Llama3.1-1.4B pre-training under static coefficient settings. \textit{L} denotes the observation data are only from four Llama-family models from pilot experiments.}
\setlength{\tabcolsep}{2.2pt}
\begin{tabular}{lccccccccccccc}
\toprule
\textbf{Optimizer} & \textbf{AE} & \textbf{AC} & \textbf{SciQ} & \textbf{MM} & \textbf{MM-P} & \textbf{HS} & \textbf{OBQA} & \textbf{PIQA} & \textbf{RACE} & \textbf{WG} & \textbf{CSQA} & \textbf{AGI} & \textbf{Avg}   \\ \midrule
Muon-KJ            & 65.32       & 28.41       & 85.20         & 25.47       & 6.94          & 37.32       & 25.00         & 69.36         & 32.15         & 53.67       & 19.99         & 16.87        & 38.81          \\
Muon-PE            & 66.41       & 30.03       & 84.80         & 23.81       & 7.53          & 37.54       & 24.00         & 69.15         & 32.34         & 53.67       & 20.64         & 17.24        & \textbf{38.93} \\ \midrule
Typical 1e-2       & 66.37       & 29.78       & 84.60         & 24.08       & 5.59          & 37.75       & 23.20         & 69.04         & 31.67         & 51.07       & 20.31         & 17.24        & 38.39          \\
Typical 1e-2 L     & 66.51       & 31.91       & 85.90         & 25.00       & 7.21          & 38.63       & 24.80         & 70.24         & 33.88         & 52.56       & 20.31         & 16.43        & 39.45          \\
Typical 5e-2       & 64.06       & 29.18       & 82.40         & 24.12       & 7.13          & 37.24       & 19.80         & 67.85         & 31.58         & 52.64       & 20.07         & 17.89        & 37.83          \\
Typical 5e-2 L     & 65.07       & 30.20       & 85.60         & 24.65       & 5.78          & 38.93       & 24.20         & 70.18         & 34.83         & 55.56       & 19.90         & 17.98        & 39.41          \\
Cons 1e-2          & 65.78       & 29.79       & 85.00         & 24.90       & 6.08          & 37.84       & 25.00         & 69.59         & 32.63         & 54.30       & 19.25         & 17.24        & 38.95          \\
Cons 1e-2 L        & 65.78       & 31.31       & 84.80         & 23.92       & 6.57          & 37.51       & 24.60         & 70.08         & 31.96         & 53.91       & 21.78         & 16.48        & 39.06          \\
Cons 5e-2          & 64.69       & 29.69       & 85.30         & 25.86       & 7.20          & 37.77       & 25.20         & 69.31         & 33.40         & 53.83       & 21.38         & 17.76        & 39.28          \\
Cons 5e-2 L        & 66.20       & 30.38       & 84.10         & 25.28       & 7.39          & 37.56       & 26.40         & 70.39         & 33.40         & 53.67       & 21.25         & 16.96        & 39.42          \\
Floor 1e-2         & 66.58       & 30.03       & 84.50         & 25.99       & 6.57          & 37.43       & 23.00         & 69.75         & 32.72         & 54.06       & 19.33         & 16.93        & 38.91          \\
Floor 1e-2 L       & 65.53       & 30.89       & 85.30         & 23.94       & 7.10          & 38.63       & 24.80         & 70.73         & 33.01         & 55.33       & 21.05         & 17.98        & \textbf{39.52} \\
Floor 5e-2         & 66.33       & 30.97       & 84.30         & 24.66       & 6.87          & 39.00       & 23.60         & 68.99         & 33.21         & 54.62       & 18.92         & 17.42        & 39.07          \\
Floor 5e-2 L       & 66.98       & 29.78       & 85.20         & 26.00       & 5.59          & 38.63       & 25.40         & 69.70         & 33.59         & 53.99       & 20.23         & 17.85        & 39.41          \\ \bottomrule
\end{tabular}
\label{tab:static_coeff_llama}
\end{table*}

\begin{table*}[tb]
\centering
\small
\caption{Full results of Qwen3-1.7B under static coefficient settings. \textit{Q} denotes the observation data are only from four Qwen-family models from pilot experiments.}
\setlength{\tabcolsep}{2.2pt}
\begin{tabular}{lccccccccccccc}
\toprule
\textbf{Optimizer} & \textbf{AE} & \textbf{AC} & \textbf{SciQ} & \textbf{MM} & \textbf{MM-P} & \textbf{HS} & \textbf{OBQA} & \textbf{PIQA} & \textbf{RACE} & \textbf{WG} & \textbf{CSQA} & \textbf{AGI} & \textbf{Avg}   \\ \midrule
Muon-KJ            & 68.22       & 31.74       & 85.70         & 24.80       & 5.54          & 40.71       & 26.40         & 72.09         & 33.78         & 55.00       & 20.72         & 16.45        & 40.10          \\
Muon-PE            & 68.98       & 31.66       & 85.30         & 25.13       & 7.54          & 40.96       & 26.40         & 71.82         & 32.44         & 54.78       & 20.88         & 16.93        & \textbf{40.24} \\ \midrule
Typical 1e-2       & 67.59       & 32.51       & 86.30         & 25.77       & 6.44          & 40.73       & 25.60         & 70.55         & 33.11         & 53.20       & 19.08         & 17.39        & 39.86          \\
Typical 1e-2 Q     & 69.32       & 32.68       & 87.80         & 24.09       & 7.04          & 40.99       & 27.60         & 72.52         & 34.16         & 56.91       & 18.10         & 16.93        & \textbf{40.68} \\
Typical 5e-2       & 68.31       & 32.51       & 86.20         & 24.85       & 5.36          & 40.79       & 24.60         & 71.82         & 33.59         & 56.91       & 18.35         & 17.42        & 40.06          \\
Typical 5e-2 Q     & 68.98       & 32.59       & 86.00         & 25.00       & 6.08          & 40.82       & 26.80         & 71.33         & 32.73         & 55.96       & 18.67         & 16.87        & 40.15          \\
Cons 1e-2          & 67.72       & 32.85       & 85.60         & 25.95       & 6.62          & 41.12       & 25.60         & 70.78         & 33.21         & 55.17       & 21.13         & 16.98        & 40.23          \\
Cons 1e-2 Q        & 69.36       & 32.94       & 86.50         & 24.14       & 6.74          & 40.96       & 26.80         & 72.74         & 34.35         & 55.01       & 19.25         & 17.89        & 40.56          \\
Cons 5e-2          & 69.32       & 32.51       & 86.60         & 25.02       & 7.57          & 40.90       & 26.60         & 71.43         & 33.21         & 55.81       & 18.43         & 17.08        & 40.37          \\
Cons 5e-2 Q        & 68.64       & 31.83       & 86.20         & 24.90       & 7.40          & 40.98       & 27.00         & 71.76         & 33.59         & 56.35       & 19.74         & 17.61        & 40.50          \\
Floor 1e-2         & 67.93       & 32.68       & 86.10         & 25.54       & 5.86          & 41.13       & 27.20         & 71.76         & 33.97         & 56.20       & 19.74         & 17.34        & 40.45          \\
Floor 1e-2 Q       & 69.82       & 32.59       & 86.90         & 25.23       & 7.22          & 41.07       & 27.40         & 71.87         & 32.82         & 53.00       & 20.48         & 17.37        & 40.38          \\
Floor 5e-2         & 67.76       & 31.31       & 87.10         & 24.30       & 6.80          & 40.94       & 27.20         & 70.95         & 34.26         & 56.99       & 19.49         & 16.61        & 40.31          \\
Floor 5e-2 Q       & 68.64       & 31.83       & 86.20         & 24.90       & 7.40          & 40.98       & 27.00         & 71.76         & 33.59         & 56.35       & 19.74         & 16.61        & 40.42          \\ \bottomrule
\end{tabular}
\label{tab:static_coeff_qwen}
\end{table*}

\section{Further Analysis}

\subsection{Prolonged Pre-training}
\label{app:additional_prolong}

Additional prolonged pre-training results for Llama3.1-1.4B are provided in Tab.~\ref{tab:additional_llama}.

\begin{table}[tb]
\centering
\small
\caption{Full results of pre-training Llama3.1-1.4B under 2$\times$ Chinchilla optimal ratio.}
\setlength{\tabcolsep}{2.2pt}
\begin{tabular}{@{}lccccccccccccc@{}}
\toprule
\textbf{Optimizer}  & \textbf{AE}          & \textbf{AC}          & \textbf{SciQ}        & \textbf{MM}          & \textbf{MM-P}        & \textbf{HS}          & \textbf{OBQA}        & \textbf{PIQA}        & \textbf{RACE}        & \textbf{WG}          & \textbf{CSQA}        & \textbf{AGI}         & \textbf{Avg}         \\ \midrule
\textit{2$\times$}         & \multicolumn{1}{l}{} & \multicolumn{1}{l}{} & \multicolumn{1}{l}{} & \multicolumn{1}{l}{} & \multicolumn{1}{l}{} & \multicolumn{1}{l}{} & \multicolumn{1}{l}{} & \multicolumn{1}{l}{} & \multicolumn{1}{l}{} & \multicolumn{1}{l}{} & \multicolumn{1}{l}{} & \multicolumn{1}{l}{} & \multicolumn{1}{l}{} \\
Muon-KJ             & 65.32                & 30.41                & 85.20                & 25.47                & 6.94                 & 37.32                & 24.60                & 69.36                & 32.15                & 53.67                & 19.99                & 17.40                & 38.99                \\
\textbf{AMO (ours)} & 67.97                & 32.13                & 86.40                & 24.66                & 6.84                 & 39.10                & 24.60                & 70.95                & 34.00                & 53.59                & 21.70                & 17.72                & \textbf{39.97}                \\ \bottomrule
\end{tabular}
\label{tab:additional_llama}
\end{table}

\subsection{NS Step Allocation}

\label{app:ns_step}

\paragraph{Setup.}

In this variant, we keep the optimizer, LR scheduler, coefficient, and training setup unchanged, and only override the number of NS steps according to the layer type for pre-training Qwen-300M/0.6B. 
Based on the observation data, we divide the seven operator types into two coarse groups. 
The \textit{hard} group contains attention query, output, and value projections, while the \textit{easy} group contains attention key and the three MLP projections. 
The five static allocation configurations are summarized in Tab.~\ref{tab:ns-step-allocation}.

The allocation sweep exposes a clear sensitivity of Muon to where the NS budget is spent. 
A1 reduces the step count of the easy group from five to four while keeping the hard group unchanged, testing whether the baseline over-orthogonalizes easier matrices. 
A2 reallocates more budget toward the hard attention projections by increasing \texttt{attn\_q/o/v} to six steps while keeping the easy group at four steps. 
A3 isolates the effect of increasing only the hard group, with all easy layers left at the default five steps.
A4 and A5 further probe the coarse hard/easy taxonomy. 
Instead of uniformly favoring \texttt{attn\_q/o/v}, they assign more steps to \texttt{attn\_v} and \texttt{attn\_k}, while reducing or preserving the MLP projections. 

\begin{table}[t]
\centering
\small
\caption{Static operator-wise NS step allocation configurations. KJ is the default baseline with five NS steps for all Muon-managed matrices.}
\label{tab:ns-step-allocation}
\begin{tabular}{llcccccc}
\toprule
\textbf{Group} & \textbf{Operator Type} & \textbf{KJ} & \textbf{A1} & \textbf{A2} & \textbf{A3} & \textbf{A4} & \textbf{A5} \\
\midrule
Hard
& \texttt{attn\_q}    & 5 & 5 & 6 & 6 & 4 & 4 \\
& \texttt{attn\_o}    & 5 & 5 & 6 & 6 & 4 & 4 \\
& \texttt{attn\_v}    & 5 & 5 & 6 & 6 & 7 & 6 \\
\midrule
Easy
& \texttt{attn\_k}    & 5 & 4 & 4 & 5 & 7 & 6 \\
& \texttt{mlp\_up}    & 5 & 4 & 4 & 5 & 4 & 5 \\
& \texttt{mlp\_gate}  & 5 & 4 & 4 & 5 & 4 & 5 \\
& \texttt{mlp\_down}  & 5 & 4 & 4 & 5 & 4 & 4 \\
\bottomrule
\end{tabular}
\end{table}

\begin{table}[!t]
\centering
\small
\setlength{\tabcolsep}{2.2pt}
\caption{Full results of NS step allocation experiments on pre-training Qwen3-300M/0.6B.}
\begin{tabular}{lccccccccccccc}
\toprule
\textbf{Optimizer}  & \textbf{AE}          & \textbf{AC}          & \textbf{SciQ}        & \textbf{MM}          & \textbf{MM-P}        & \textbf{HS}          & \textbf{OBQA}        & \textbf{PIQA}        & \textbf{RACE}        & \textbf{WG}          & \textbf{CSQA}        & \textbf{AGI}         & \textbf{Avg}   \\ \midrule
\textit{Qwen3-300M} & \multicolumn{1}{l}{} & \multicolumn{1}{l}{} & \multicolumn{1}{l}{} & \multicolumn{1}{l}{} & \multicolumn{1}{l}{} & \multicolumn{1}{l}{} & \multicolumn{1}{l}{} & \multicolumn{1}{l}{} & \multicolumn{1}{l}{} & \multicolumn{1}{l}{} & \multicolumn{1}{l}{} & \multicolumn{1}{l}{} &                \\
Muon-KJ             & 57.15                & 23.55                & 77.60                & 23.07                & 7.16                 & 30.93                & 19.00                & 64.74                & 29.00                & 51.38                & 19.99                & 17.45                & 35.09          \\
A1         & 56.06                & 21.84                & 79.60                & 23.15                & 6.64                 & 31.16                & 20.20                & 65.34                & 29.09                & 50.04                & 19.49                & 17.68                & 35.02          \\
A2                  & 56.01                & 22.44                & 78.90                & 23.16                & 7.75                 & 31.08                & 21.00                & 65.89                & 30.05                & 53.04                & 20.80                & 17.00                & \textbf{35.59} \\
A3                  & 56.61                & 22.78                & 77.60                & 23.60                & 7.15                 & 30.89                & 20.40                & 66.21                & 29.47                & 50.43                & 19.98                & 16.80                & 35.16          \\
A4                  & 56.40                & 23.21                & 78.00                & 24.43                & 7.64                 & 30.96                & 21.60                & 65.18                & 30.14                & 49.09                & 19.98                & 17.50                & 35.34          \\
A5                  & 56.90                & 23.29                & 77.20                & 22.99                & 7.12                 & 31.01                & 20.40                & 65.61                & 30.72                & 49.88                & 19.98                & 17.42                & 35.21          \\ \midrule
\textit{Qwen3-0.6B} & \multicolumn{1}{l}{} & \multicolumn{1}{l}{} & \multicolumn{1}{l}{} & \multicolumn{1}{l}{} & \multicolumn{1}{l}{} & \multicolumn{1}{l}{} & \multicolumn{1}{l}{} & \multicolumn{1}{l}{} & \multicolumn{1}{l}{} & \multicolumn{1}{l}{} & \multicolumn{1}{l}{} & \multicolumn{1}{l}{} &                \\
Muon-KJ             & 62.16                & 26.79                & 81.90                & 23.80                & 7.30                 & 34.75                & 23.20                & 67.90                & 31.67                & 51.93                & 19.00                & 16.92                & 37.28          \\
A1                  & 61.32                & 27.39                & 82.80                & 24.52                & 8.22                 & 34.99                & 24.60                & 66.76                & 32.15                & 53.91                & 20.39                & 16.46                & \textbf{37.79} \\
A2                  & 61.96                & 26.37                & 82.20                & 24.27                & 8.12                 & 35.09                & 22.40                & 67.08                & 30.24                & 52.25                & 19.57                & 18.20                & 37.31          \\
A3                  & 61.15                & 25.60                & 82.20                & 24.64                & 6.89                 & 34.93                & 21.60                & 67.57                & 31.00                & 52.49                & 21.87                & 18.02                & 37.33          \\
A4                  & 60.82                & 25.94                & 80.70                & 24.15                & 6.71                 & 34.90                & 21.60                & 67.85                & 30.14                & 52.17                & 20.88                & 17.32                & 36.93          \\
A5                  & 61.62                & 27.22                & 81.80                & 25.25                & 7.41                 & 34.78                & 22.20                & 67.41                & 32.06                & 51.62                & 18.84                & 17.81                & 37.34          \\ \bottomrule
\end{tabular}
\label{tab:ns_step_full}
\end{table}

\paragraph{Results.}

Full results are provided in Tab.~\ref{tab:ns_step_full}. 
Although this simple allocation scheme can be effective, its performance gains are not stable across variants, suggesting that static hand-crafted NS step allocation is only a coarse heuristic rather than a robust replacement for adaptive scheduling.

\subsection{Post-warmup Stability}
\label{app:post_warmup_stability}

The layer-ordering reversal in Fig.~\ref{fig:geometry_main}(\textbf{d}) occurs at step 120, and most crossings reported in App.~\ref{app:model_geometry} fall between steps 20 and 120.
Because this period overlaps with learning-rate warmup, \textbf{F4} could in principle reflect an initialization transient, with the layer ordering settling once warmup ends.
To test this, we analyze layer-ordering stability separately during and after warmup.

\paragraph{Setup.}
We use full-training geometry traces of Llama3.1-760M and Llama3.1-1.4B, recorded every 10 optimizer steps at checkpoints $t_1<t_2<\dots<t_K$.
For each operator type, let $\epsilon_\ell^{(t)}$ denote the post-NS orthogonality error of layer $\ell$ at step $t$.
A layer pair $(i,j)$ \emph{reverses} between adjacent checkpoints $t_k$ and $t_{k+1}$ if
\begin{equation}
\bigl(\epsilon_i^{(t_k)}-\epsilon_j^{(t_k)}\bigr)\bigl(\epsilon_i^{(t_{k+1})}-\epsilon_j^{(t_{k+1})}\bigr)<0 .
\end{equation}
We split adjacent intervals into a warmup phase and a post-warmup phase at the end of warmup (step $T_{\mathrm{w}}$), excluding the single interval that crosses this boundary.
Within each phase, the reversal rate is the number of reversals divided by the number of opportunities, i.e., the number of layer pairs times the number of adjacent intervals.
We further measure ranking stability with two Spearman correlations of the layer ranking induced by $\epsilon_\ell^{(t)}$:
$\rho_{\mathrm{adj}}$, the median correlation between adjacent post-warmup checkpoints, which captures short-term stability;
and $\rho_{\mathrm{fl}}$, the correlation between the first post-warmup ranking and the final ranking, which captures cumulative drift.

\begin{table}[h]
\centering
\small
\caption{Layer-ordering stability during and after warmup. Reversal rates (\%) are normalized by all layer-pair $\times$ adjacent-checkpoint opportunities within each phase; \textit{Overall} pools all operator types. $\rho_{\mathrm{adj}}$: median Spearman correlation between adjacent post-warmup checkpoints. $\rho_{\mathrm{fl}}$: Spearman correlation between the first post-warmup and final rankings.}
\label{tab:post_warmup}
\begin{tabular}{lccccc}
\toprule
\textbf{Operator} & \textbf{Warmup Rev.} & \textbf{Post Rev.} & \textbf{Last Rev. Step} & $\boldsymbol{\rho_{\mathrm{adj}}}$ & $\boldsymbol{\rho_{\mathrm{fl}}}$ \\
\midrule
\multicolumn{6}{l}{\textit{Llama3.1-760M}} \\
\texttt{attn\_q}   & 18.63 & 6.13  & 3610 & 0.965 & 0.769 \\
\texttt{attn\_k}   & 18.18 & 16.98 & 3620 & 0.846 & 0.420 \\
\texttt{attn\_v}   & 11.41 & 10.74 & 3620 & 0.909 & $-$0.259 \\
\texttt{attn\_o}   & 9.80  & 3.58  & 3590 & 0.986 & 0.678 \\
\texttt{mlp\_gate} & 15.51 & 8.61  & 3620 & 0.944 & 0.413 \\
\texttt{mlp\_up}   & 12.30 & 11.04 & 3620 & 0.923 & 0.399 \\
\texttt{mlp\_down} & 16.67 & 14.10 & 3620 & 0.916 & $-$0.399 \\
\textit{Overall}   & 14.64 & 10.17 & --   & --    & -- \\
\midrule
\multicolumn{6}{l}{\textit{Llama3.1-1.4B}} \\
\texttt{attn\_q}   & 12.50 & 4.27  & 6670 & 0.979 & 0.853 \\
\texttt{attn\_k}   & 12.73 & 15.50 & 6670 & 0.862 & 0.641 \\
\texttt{attn\_v}   & 7.99  & 10.04 & 6670 & 0.929 & $-$0.647 \\
\texttt{attn\_o}   & 6.85  & 2.40  & 6670 & 0.991 & 0.885 \\
\texttt{mlp\_gate} & 10.26 & 6.51  & 6670 & 0.971 & 0.538 \\
\texttt{mlp\_up}   & 9.43  & 7.66  & 6670 & 0.965 & 0.174 \\
\texttt{mlp\_down} & 10.76 & 13.10 & 6670 & 0.903 & 0.409 \\
\textit{Overall}   & 10.07 & 8.50  & --   & --    & -- \\
\bottomrule
\end{tabular}
\end{table}

\paragraph{Results.}
Tab.~\ref{tab:post_warmup} leads to three observations.
\textit{(i) Reversals persist after warmup.}
The overall reversal rate decreases after warmup, from 14.64\% to 10.17\% on Llama3.1-760M and from 10.07\% to 8.50\% on Llama3.1-1.4B, but remains far from zero.
The decrease is most pronounced for \texttt{attn\_q} and \texttt{attn\_o}, whereas several operator types show little change or even an increase, e.g., \texttt{attn\_k} (12.73\% $\to$ 15.50\%) and \texttt{mlp\_down} (10.76\% $\to$ 13.10\%) on Llama3.1-1.4B.
\textit{(ii) Reversals continue until late in training.}
For every operator type, the last observed reversal occurs at steps 3590--3620 on Llama3.1-760M and at step 6670 on Llama3.1-1.4B, near the end of the recorded traces, so the ordering does not settle at any point after warmup.
\textit{(iii) Rankings drift gradually rather than reshuffle abruptly.}
$\rho_{\mathrm{adj}}$ remains high (0.846--0.991), indicating that rankings at consecutive checkpoints are similar.
In contrast, $\rho_{\mathrm{fl}}$ is substantially lower and even negative for \texttt{attn\_v} on both models ($-0.259$ and $-0.647$) and \texttt{mlp\_down} on Llama3.1-760M ($-0.399$), meaning that the final layer ranking can be largely inverted relative to the first post-warmup ranking.
Since checkpoint-level noise would lower both correlations alike, the gap between them indicates that small, consistent rank changes accumulate into substantial long-horizon drift.

Together, these results show that the reversal in \textbf{F4} is not merely an initialization transient.
Although layer ordering becomes more stable after warmup, reversals persist at a non-negligible rate throughout training and accumulate into large changes in ranking.
The layer ordering observed at any single point in training is therefore an unreliable predictor of the ordering later on, which supports \textbf{F4} and our decision not to prescribe NS allocation by layer depth.

\end{document}